\documentclass[11pt]{article}

\usepackage[table,xcdraw]{xcolor}
\usepackage[normalem]{ulem}
\useunder{\uline}{\ul}{}

\usepackage[toc,page,header]{appendix}
\usepackage{minitoc}

\usepackage{tocloft}

\usepackage{hyperref}

\usepackage{amsmath}
\usepackage{amssymb}
\usepackage{mathtools}
\usepackage{amsthm}

\usepackage{subcaption}

\usepackage{cleveref}

\usepackage{authblk}

\usepackage[margin=1in]{geometry}
\usepackage{thmtools,thm-restate}

\theoremstyle{plain}
\newtheorem{theorem}{Theorem}[section]
\newtheorem{observation}{Observation}[section]

\newtheorem{lemma}[theorem]{Lemma}
\newtheorem{corollary}[theorem]{Corollary}
\theoremstyle{definition}
\newtheorem{definition}[theorem]{Definition}

\newtheorem{example}[theorem]{Example}
\newtheorem*{examplecontinued}{Example 4.2 [Continued]}
\theoremstyle{remark}

\usepackage[algo2e, ruled, noend, linesnumbered]{algorithm2e} 

\usepackage{authblk}

\usepackage{amsmath}               
\usepackage{amsthm}                
\usepackage{amssymb}               
\usepackage{paralist}
\usepackage[shortlabels]{enumitem} 
\usepackage{enumitem}
\usepackage{mathrsfs}
\usepackage{mathtools}
\usepackage{appendix}
\usepackage{tikz}
\usetikzlibrary{angles,quotes,positioning}
\tikzset{main node/.style={circle,fill=none,draw,minimum size=1cm,inner sep=0pt},}
\PassOptionsToPackage{hyphens}{url}

\usepackage{natbib}
\usepackage{bm}
\usepackage{bbm}
 
\usepackage{color}

\usepackage{scrextend}
\usepackage{multirow} 
\newcommand{\benchmarkrelaxedstronger}[1]{\beta_{#1}^{\text{tol}}}
\newcommand{\benchmarkrelaxedweaker}[1]{\beta_{#1}^\text{self-tol}}
\newcommand{\benchmark}[1]{\beta_{#1}^{\text{orig}}}
\newcommand{\benchmarkfunction}{\beta}

\newcommand{\ALG}{\texttt{ALG}}

\newcommand{\ExploreThenCommit}{\texttt{ExploreThenCommit}}
\newcommand{\ExploreThenCommitThrowOut}{\texttt{ExploreThenCommitThrowOut}}
\newcommand{\ActiveArmElimination}{\texttt{ActiveArmElimination}}

\newcommand{\PhasedUCB}{\texttt{PhasedUCB}}
\newcommand{\ComputeQuantities}{\texttt{ComputeActiveArms}}
\newcommand{\ExploreThenUCB}{\texttt{ExploreThenUCB}}
\newcommand{\LipschitzUCB}{\texttt{LipschitzUCB}}
\newcommand{\LipschitzUCBGeneralized}{\texttt{LipschitzUCBGen}}

\newcommand{\Setting}{\text{DSG}}
\newcommand{\StrongSetting}{\text{StrongDSG}}
\newcommand{\WeakSetting}{\text{WeakDSG}}

\newcommand{\Instance}{\mathcal{I}}
\newcommand{\A}{\mathcal{A}}
\newcommand{\B}{\mathcal{B}}

\newcommand{\Event}{G}
\newcommand{\RR}[4]{r_{#1, #4}(#2, #3)}
\newcommand{\RRTilde}[4]{\tilde{r}_{#1, #4}(#2, #3)}

\newcommand{\MR}[3]{v_{#1}(#2, #3)}
\newcommand{\hatMR}[3]{\hat{v}_{#1}(#2, #3)}
\newcommand{\UCBMR}[3]{v^{\text{UCB}}_{#1}(#2, #3)}
\newcommand{\hatMRTimeStep}[4]{\hat{v}_{#1, #4}(#2, #3)}  
\newcommand{\UCBMRTimeStep}[4]{v^{\text{UCB}}_{#1, #4}(#2, #3)} 
\newcommand{\MRTilde}[3]{\tilde{v}_{#1}(#2, #3)} 

\newcommand{\hatMROneArg}[2]{\hat{v}_{#1}(#2)}
\newcommand{\UCBOneArg}[2]{v^{\text{UCB}}_{#1}(#2)}
\newcommand{\hatMROneArgTimeStep}[3]{\hat{v}_{#1, #3}(#2)}
\newcommand{\UCBOneArgTimeStep}[3]{v^{\text{UCB}}_{#1, #3}(#2)}

\newcommand{\NumPullsOneArg}[2]{n_{#2}(#1)}
\newcommand{\NumPulls}[3]{n_{#3}(#1, #2)}
\newcommand{\NumPullsTwoTimesOneArg}[3]{n_{#2, #3}(#1)}
\newcommand{\hatMROneArgNoPlayer}[1]{\hat{v}(#1)}

\newcommand{\Hist}{H}

\newcommand{\CentHist}{\Hist^{\text{C}}}

\newcommand{\p}[1]{\left( #1 \right)}
\newcommand{\br}[1]{\left[ #1 \right]}

\newcommand{\argmax}[0]{\text{argmax}}
\newcommand{\argmin}[0]{\text{argmin}}

\newcommand{\cd}[0]{\cdot}
\newcommand{\abs}[1]{\left \vert #1 \right \vert}

\usepackage{framed, listings}
\usepackage{csquotes}

\title{\Large{Impact of Decentralized Learning on Player Utilities in Stackelberg Games}
\footnote{Authors in alphabetical order. Part of this work was conducted while KD and MJ were at Microsoft Research.}}

\author[1]{Kate Donahue}
\author[2]{Nicole Immorlica}
\author[3]{Meena Jagadeesan}
\author[2]{Brendan Lucier}
\author[2]{Aleksandrs Slivkins}
\affil[1]{\textit{Cornell}}
\affil[2]{\textit{Microsoft Research}}
\affil[3]{\textit{University of California, Berkeley}}

\date{}

\begin{document}
\maketitle

\begin{abstract}
When deployed in the world, a learning agent such as a recommender system or a chatbot often repeatedly interacts with another learning agent (such as a user) over time. In many such two-agent systems, each agent learns separately and the rewards of the two agents are not perfectly aligned. To better understand such cases, we examine the learning dynamics of the two-agent system and the implications for each agent's objective. We model these systems as Stackelberg games with decentralized learning and show that standard regret benchmarks (such as Stackelberg equilibrium payoffs) result in worst-case linear regret for at least one player. To better capture these systems, we construct a relaxed regret benchmark that is tolerant to small learning errors by agents. We show that standard learning algorithms fail to provide sublinear regret, and we develop algorithms to achieve near-optimal $\mathcal{O}(T^{2/3})$ regret for both players with respect to these benchmarks. We further design relaxed environments under which faster learning ($\mathcal{O}(\sqrt{T})$) is possible. Altogether, our results take a step towards assessing how two-agent interactions in sequential and decentralized learning environments affect the utility of both agents.  
\end{abstract}

\newpage 

\doparttoc 
\faketableofcontents 

\part{} 
\parttoc 

\newpage

\section{Introduction}

When learning agents such as recommender systems or chatbots are deployed into the world, these learning agents often repeatedly interact with other learning agents (such as humans). For example, a recommender system---through repeated interactions with a user---learns which content to suggest to the user, while the user simultaneously learns their own preferences over content (Example \ref{example:recsys}). As another example, a chatbot such as ChatGPT---during a chat session---can iteratively refine its generated content to the user's stylistic preferences, while the user (or prompt engineering agent) simultaneously learns how to best interact with the chatbot (Example \ref{example:llm}). 

These two-agent systems---among many others---share the following structural features. The environments are \emph{decentralized} (both agents operate autonomously, without central coordination of their actions). Furthermore, the environments are \textit{sequential} (one agent always chooses their action first\footnote{E.g., a recommender system recommends a slate of items and the user picks among them; in a chatbot session, the user picks a prompt that the LLM responds to.}) and \emph{misaligned} (the agents can obtain different utilities for the same pair of actions\footnote{Misalignment could arise from fundamental differences in agent preferences: the engagement metrics of recommender systems rarely align with user welfare (e.g., \citep{MBH21}); or a chatbot might be trained to optimize societal preferences or cultural norms (e.g., avoiding violent language) which conflict with individual user preferences (e.g., \citep{BCSTCBMGAB22}). Misalignment could also arise from misspecification, if the metrics that the AI system optimizes do not perfectly capture user preferences (e.g., \citep{ZH20}).}). Finally, the environments exhibit \textit{learning} (both agents learn from repeated interactions about which actions to take). For such two-agent environments, core questions of interest include: \textit{how quickly does this two-agent system learn, and what are the implications for each agent's objective?}

In the absence of learning, the interaction between 
misaligned agents taking sequential actions is formalized by Stackelberg games. In this setting, the leader chooses an action first and the follower chooses an action to respond. The two agents are allowed to have distinct utility functions over pairs of actions. The standard benchmark is the \textit{Stackelberg equilibrium}, where the leader picks the best action they can, assuming that the follower will pick their own best response. However, this classical solution concept is tailored to the full-information setting where both players know their own utilities and the leader knows how the follower will best respond; in fact, the static Stackelberg game framework breaks down when agents instead must \emph{learn} these utilities from noisy feedback.

Our focus is on a decentralized Stackelberg learning environment. In this setting, the leader and the follower repeatedly interact and each make decisions about which actions to take, where each agent only observes their own realized stochastic rewards. In this environment, it is natural to model each player's learning process as a multi-armed bandit algorithm\footnote{See \citet{S19, L20} for textbook treatments of multi-armed bandits.} which learns over time  which arms (actions) to pull. A unique feature of this sequential two-player learning environment is that agents must learn in two separate ways---first, both agents learn their own (fixed) preferences from stochastic observations, and secondly, the leader needs to learn and adapt to the follower's (evolving) responses to the leader's actions---which complicates the design of learning algorithms. 

In this paper, we initiate the study of how this learning environment impacts \emph{both} the leader and follower's utility, motivated by how both objectives are of societal interest in natural real-world settings (see Example \ref{example:llm} and Example \ref{example:recsys}). Rather than only focusing on the regret of the leader as is typical in learning in Stackelberg games, we thus examine the \textit{maximum regret} of the two agents.
Our main contributions are to design appropriate benchmarks for each agent and to construct algorithms which achieve near-optimal regret against these benchmarks.  Our results apply to the most general setting which allows for \textit{arbitrary relationships between the two player's utilities}.
\begin{itemize}
    \item \textbf{Linear regret for Stackelberg benchmarks:} We first show that the player utilities in the Stackelberg equilibrium are fundamentally unachievable and necessarily lead to linear regret for at least one agent (Theorem \ref{thrm:worstlinear}). 
    \item \textbf{Relaxed benchmarks:} The possibility of linear regret motivates us to design relaxed benchmarks which are more tolerant to the other agent being suboptimal. We thus define \textit{$\gamma$-tolerant benchmarks} (Definition \ref{def:relaxedbenchmarks}), which account for incomplete learning: the benchmark captures an agent's worst-case utility if the other agent is up to $\gamma$-suboptimal.\footnote{Section \ref{sec:benchmarks} describes our benchmark and $\gamma$ in greater detail, and Section \ref{subsec:relatedwork} compares our benchmarks to prior work.} 
    \item \textbf{Regret bounds:} Using the $\gamma$-tolerant benchmarks, we construct algorithms for the leader and follower that achieve $O(T^{2/3})$ regret (Theorem \ref{thm:explorethenucb}). Surprisingly, this dependence on $T$ is unavoidable, and \textit{any} pair of algorithms achieves $\Omega(T^{2/3})$ regret (Theorem \ref{thrm:dlower}). Nonetheless, under relaxed settings---either with a weaker benchmark or when players agree on which pairs of actions are meaningfully different\footnote{We formalize this as a continuity requirement on the utilities (Section \ref{sec:relaxed}). This requirement allows players to be misaligned (e.g. different preferences), but requires them to agree on which outcomes are \textit{substantially different} from each other.}---we show that faster learning (i.e., $O(\sqrt{T})$ regret) is possible (Section \ref{sec:relaxed}). 
\end{itemize}

From an algorithmic perspective, our results provide insight into which bandit algorithms for the leader allow for low regret for both players. Out-of-box stochastic algorithms do not provide this guarantee: for example, both agents choosing \texttt{ExploreThenCommit} can lead to linear regret even for the $\gamma$-tolerant benchmarks (Proposition \ref{prop:decentralizedlower}). The intuition is that since the follower's actions can change between time steps, the leader is not operating in a stochastic environment; as a result, the follower's exploration phase can distort the leader's learning. This motivates us to design algorithms where \textit{the leader waits for the follower to partially converge before starting to learn}: $\ExploreThenCommitThrowOut$ (Algorithm \ref{algo:explorethencommitthrowout}) and $\ExploreThenUCB$ (Algorithm \ref{algo:explorethenucb}). The more sophisticated of these two algorithms, $\ExploreThenUCB$ (Algorithm \ref{algo:explorethenucb}),  guarantees a $T^{2/3}$ regret bound when the follower applies any algorithm with certain properties (i.e., high-probability instantaneous regret bounds) (Theorem \ref{thm:explorethenucb}). We then consider two relaxed environments where the leader no longer needs to worry about being overly distorted by the follower; in these environments, \textit{the leader can start learning before the follower has partially converged}, which enables $O(\sqrt{T})$ regret bounds (Theorems \ref{thrm:dpL} and \ref{thm:UCBweaker}).

More broadly, our work takes a step towards assessing the utility of \textit{both learning agents} in decentralized, misaligned environments. Our model and results capture the general setting where the player utilities are arbitrarily related, where players might not even agree upon which pairs of actions give similar or different rewards. This motivated us to design benchmarks which are tolerant to small errors in the other player. 
We hope that our benchmarks and algorithms serve as a starting point for assessing when two-agent learning systems in misaligned environments can ensure high utility for both agents.

\subsection{Related Work}\label{subsec:relatedwork}

Most closely related is the work on learning in Stackelberg games (SGs) where both players incur stochastic rewards. \citet{bai2021sample, gan2023robust} focus on the centralized setting where the learner controls the actions and observes the rewards of both players; in contrast, we study a decentralized setting where each player controls their own actions and only observes their own rewards. Nonetheless, the benchmarks proposed in these papers are related to the $\gamma$-tolerant benchmarks that we consider, but with some key differences. For the leader's utility, their benchmark is equivalent to our $\gamma$-tolerant benchmark with a fixed value of $\epsilon$ (rather than an $\inf$ over $\epsilon \le \gamma$ with a regularizer). For the follower's utility, their benchmark only ensures $\epsilon$-optimality with respect to the leader's selected action; in contrast, we consider a different style of follower benchmark that is more conceptually similar to the benchmark for the leader. Also, we study regret, whereas they study the speed of convergence.

Several papers study \emph{decentralized} online learning in SGs. 
\citet{camara2020mechanisms, collina2023efficient} 
posit that the follower runs a no-counterfactual-internal-regret algorithm and design no-regret algorithms for the leader. 
However, they assume \emph{strong alignment} between the players' rewards: \citet{camara2020mechanisms} requires that a follower choosing an $\epsilon$-suboptimal action only results in an $O(\epsilon)$ utility loss for the leader.\footnote{See Assumption 2 in \citet{camara2020mechanisms}. Appendix D therein considers some relaxations, but they lead to $\Omega(T)$ worst-case regret.} \citet{collina2023efficient} partially relax this assumption, but still require the existence of \textit{stable} actions for the leader. In contrast, we do not place any alignment conditions: in fact, 
misalignement is the driver of our linear regret result for the original Stackelberg benchmarks (Theorem \ref{thrm:worstlinear}). Other differences are that we focus on stochastic, rather than adversarial, rewards, and our benchmark is independent of the follower's choice of learning algorithm.\footnote{However, our regret bounds assume that the follower's algorithm gracefully improves over time, see Section \ref{subsec:assumptionsplayer2}.} \citet{haghtalab2023calibrated} takes a different perspective and considers the follower running a \emph{calibrated} algorithm. They design a leader algorithm which waits for the follower to partially converge, and show that the Stackelberg value is obtained in the limit as $T \rightarrow \infty$. In contrast, we focus on \textit{instance-independent regret bounds} for a fixed time horizon $T$, which requires us to relax the benchmark. Other differences are we focus on stochastic, rather than deterministic, rewards, we assume the follower observes the leader's action, and we consider the follower's utility in addition to the leader's utility.

The literature on learning in SGs is vast and includes many other variations. Many works (e.g., \citep{LCM09, balcan2015commitment, zhao2023online, lauffer2023no}) consider the leader performing (offline or) online learning and followers myopically best-responding. Other model variants studied include the leader strategizing against a follower who is running a no-regret learning algorithm \citep{BMSW18, deng2019strategizing, GKSTGWW24, brown2024learning, lin2024persuading}, the leader and follower both running gradient-based algorithms \citep{fiez2019convergence, GZG22}, non-myopic followers who best-respond to a discounted utility over future time steps \citep{haghtalab2022learning, hajiaghayi2023regret}, repeated game formulations under complete information \citep{ZT15, CAK23} the leader offering a menu of actions to the follower \citep{HAX23}, the (human) follower having cognitive biases in responding \citep{pita2009effective}, and both players having side information \citep{HWB24}. Other works have studied learning in structured SGs, including delegated choice (e.g., \citep{kleinberg2018delegated, hajiaghayi2023regret}), strategic classification (e.g., \citep{dong18, chen2019, zrnic2021, ABBN21}), performative prediction (e.g., \citep{perdomo2020performative, miller2021outside}), pricing under buyer and seller uncertainty (e.g., \citep{GHKV23}), contract theory (e.g., \citep{ZBYWJJ23}), cake cutting (e.g., \citep{branzei2024dueling}), and aligned utilities (e.g., \citep{KWS22}). 

Our work also connects to a broader literature on interacting learners. This literature examines interactions between \textit{multiple bandit learners}, studying aspects such as the convergence of systems of no-regret learners to coarse correlated and correlated equilibrium (e.g. \citep{DDK11, DFG21, ADFF22}), multiple bandit learners competing for market share (e.g., \citep{AMSW20, JJH23}), and multiple autobidding algorithms competing in an auction (e.g., \citep{BCIJEM07, BG19, LPSZ23}). Most closely related to this paper is \emph{corralling bandit algorithms} (e.g., \citep{ALNS17, PPA0ZLS20}), where a ``master algorithm" dynamically chooses among several ``base algorithms'': our decentralized learning environment in the case of aligned player utilities is essentially an instance of corralling bandits, with the ``base algorithms'' corresponding to different leader actions. The interacting learner literature also examines \emph{human-algorithm collaboration} studying aspects such as misalignment between engagement metrics and user welfare (e.g., \citep{EW16, MBH21, SVNAM21, KMR22}), impact of underspecification on human-AI misalignment (e.g., \citep{ZH20}), and \enquote{assistive} algorithmic tools (e.g. \citep{chan2019assistive}). Most closely related to this paper is work on online learning in subset selection and conformal prediction, where goals often revolve around selecting a subset of items to present to a learning user \citep{straitouri2023designing, corvelo2024human, straitouri2023improving, wang2022improving, donahue2023two, agarwal2024online, yao2022learning}, often with the goal of achieving complementarity \citep{BWZFNKRW21}. The related area of \textit{human-AI interaction} (see \citep{preece1994human, kim2015human, mackenzie2024human, lazar2017research} for textbook treatments) studies similar questions, often from a more behavioral angle. More broadly, the interacting learner literature also studies applied domains including \textit{multi-agent reinforcement learning} (see \citet{ZYB19} for a survey) and \textit{federated learning} (see \citet{YLCT19} for a survey).


\section{Model and assumptions}\label{sec:setup}
In this section we describe our formal model. We first define an instance $\Instance = (\A, \B, v_1, v_2)$ in our setup, which captures the setup of the underlying static Stackelberg game. Let $\mathcal{A}$ be a finite action set for the  leader (Player 1) and let $\mathcal{B}$ be a finite action set for the follower (Player 2). Let $\MR{i}{a}{b} \in [0, 1]$ denote Player $i$'s value (i.e., mean reward) for the leader choosing $a$ and the follower choosing $b$. The Stackelberg equilibrium takes the following form. Let $b^*(a)$ be the best-response with respect to the follower's rewards:\footnote{In case of ties in follower utility, $b^*(a)$ is the best-response with lowest leader utility.}
 $b^*(a) = \argmax_{b \in \mathcal{B}} \MR{2}{a}{b}$. The Stackelberg equilibrium $(a^*, b^*)$ is defined to be the best action the leader can take, assuming that the follower will exactly best-respond: $$a^* = \argmax_{a \in \mathcal{A}} \MR{1}{a}{b^*(a)}\text{ and } b^* = b^*(a^*).$$  Note that for simplicity, we restrict both players to pure strategies.\footnote{Other works (e.g., \citet{bai2021sample}) similarly restricts both players to pure strategies. }

In this paper, we move from the static Stackelberg Equilibrium environment to a repeated dynamic environment which we call a \textit{decentralized Stackelberg game ($\Setting$)}.  A $\Setting$ operates over $T$ time steps where each player selects actions using a multi-armed bandit algorithm. A $\Setting$ is \textit{sequential}: at each time step $t$, the leader chooses an action $a_t \in \mathcal{A}$ and then the follower chooses an action $b_t \in \mathcal{B}$. A $\Setting$ is also \emph{decentralized}: each player $i$ can observe their own stochastic rewards but not the stochastic rewards of the other player. We measure \textit{regret} for each player $i$ by their cumulative reward across all time steps relative to a benchmark.

\subsection{Interaction between players}\label{subsec:interaction}

In a $\Setting$, the interaction between the leader and follower proceeds as follows. The leader chooses an algorithm $\ALG_1$ mapping their history (formalized below) of observed actions and rewards to a distributions over actions $\mathcal{A}$, and the follower similarly chooses an algorithm $\ALG_{2}$ mapping the leader's action and the follower's history to a distribution over actions $\mathcal{B}$. After the players select algorithms, the interaction between the leader and the follower proceeds as follows at each time step $t$: 
\begin{enumerate}
    \item The leader chooses action $a_t \sim \ALG_1(\Hist_{1,t})$ as a function of their history $\Hist_{1,t}$ and reveals $a_t$ to the follower. 
    \item After observing $a_t$, the follower chooses action $b_t \sim \ALG_{2}(a_t, \Hist_{2,t})$ as a function of their history $\Hist_{2,t}$.  
    \item Players 1 and 2 incur stochastic rewards $\RR{1}{a_t}{b_t}{t} \sim \mathcal{N}(\MR{1}{a_t}{b_t}, 1)$ and $\RR{2}{a_t}{b_t}{t}  \sim \mathcal{N}(\MR{2}{a_t}{b_t}, 1)$. The noise distribution is Gaussian with unit variance\footnote{We assume the reward distributions are independent across time steps and players. We make the Gaussian assumption for simplicity, and we expect that our results would likely extend to subgaussian  Bernoulli distributions.}. 
\end{enumerate}

This interaction captures that the players are \textit{dynamic} in their learning: in particular, this framework is sufficiently general to capture a wide range of learning strategies. However, we do not study the \textit{meta-game} where agents strategically pick learning algorithms (e.g., see \cite{kolumbus2022and} for an example of a work studying the meta-game).  We believe that our model captures many real-world environments such as user-chatbot interactions and recommender system-user interactions, where agents learn about their incurred rewards from past interactions even if they do not actively optimize the higher-order manner in which they learn.
See \Cref{subsec:examples} (Example \ref{example:llm} and Example \ref{example:recsys}) for more details of how these real-world examples are captured by our model.  

\paragraph{Information structures.} Having described how the players interact, we next discuss the players' history, which further enforces decentralization. Each player $i$ can only observe their own reward $\RR{i}{a_t}{b_t}{t}$ (and cannot observe the reward of the other player). In a \textit{strongly decentralized Stackelberg game ($\StrongSetting$)}, the follower can observe the leader's action $a_t$, but the leader cannot observe the follower's action $b_t$, whereas in a \textit{weakly decentralized Stackelberg game ($\WeakSetting$)}, the leader can additionally observe the follower's action $b_t$. Notation for each player's histories is presented in \Cref{app:hist}.

Note that a strongly decentralized Stackelberg game ($\StrongSetting$) restricts what information the leader has access to, and is thus \emph{more challenging} than a weakly decentralized Stackelberg game ($\WeakSetting$). One motivation for a strongly decentralized Stackelberg game is interpretability: the leader and follower may be taking actions in spaces that are not mutually understandable (e.g. a chatbot's representation of human preferences may not be interpretable).  Most of our positive results (i.e., algorithm constructions) focus on $\StrongSetting$s, whereas our negative results apply to both $\StrongSetting$s and $\WeakSetting$s.

\subsection{Measuring regret}\label{subsec:regret}

As is typical in multi-armed bandits, we measure performance by the \textit{regret} of each player with respect to a benchmark, where higher benchmarks make learning more challenging (further detail on benchmarks is in Sections \ref{sec:stackelbergimposs} and \ref{sec:benchmarks}). For each player $i \in \left\{1,2\right\}$, given a benchmark $\benchmarkfunction_i$, we define the (expected) regret of Player $i$ on instance $\Instance$ to be: $$R_i(T; \Instance) = \benchmarkfunction_i \cdot T - \left(\sum_{t=1}^T \mathbb{E}[\RR{i}{a_t}{b_t}{t}] \right)$$ where the expectation is over randomness in the algorithm and in the stochastic rewards. Given action sets $\mathcal{A}$ and $\mathcal{B}$, we let $R_i(T)$ denote the worst-case regret across all value functions $v_1$ and $v_2$ on instances of the form $\Instance = (\A, \B, v_1, v_2)$.

Our goal is to obtain sublinear worst-case regret for \emph{both} players: that is, we will assess $\max(R_1(T), R_2(T))$. We note that this challenging goal is a departure from previous work which has typically focused solely on sublinear regret for the leader (see Section \ref{subsec:relatedwork}). Our motivation for selecting this objective is that (a) a human could be either the leader or the follower, and (b) societal welfare may demand that we care about the utility of both the leader and the follower (discussed further in \Cref{subsec:examples}). 

\subsection{Real-world examples}\label{subsec:examples}

We describe two real-world examples which fit into our framework. 

\begin{example}[User-chatbot interaction]
\label{example:llm}
Consider user-chatbot interactions where the user (e.g., a human or a prompt engineer) selects a prompt and the chatbot (e.g., an LLM-based application such as chatGPT) selects a response. We model the user as the leader and the chatbot as the follower: the user picks a (perhaps high-level) prompt or prompt engineering technique $a \in \mathcal{A}$, and the chatbot picks a response or style of response $b \in \mathcal{B}$. Repeated interactions may occur within a single chatbot session, such as with ChatGPT, where sessions can be resumed when the user logs in at a later time. An example of such an interaction is where the user repeatedly asks for help with similar queries (e.g. content generation or help with technical tasks) and learns better prompt engineering techniques \citep{CZLZ23}, while the chatbot learns how to best respond to this user by using the session history as its context \citep{HLD23, PJJS24}. The user and chatbot may have misaligned rewards for each prompt-output pair: this misalignment could arise from fundamental differences in preferences if the chatbot is trained to optimize societal preferences or cultural norms (e.g., avoiding violent language) which conflict with individual user preferences \citep{BCSTCBMGAB22}. Misalignment could also arise from unintentional misspecification if chatbot optimizes a metric which does not fully capture user preferences (e.g., if the user has an imperfect ability to communicate preferences \citep{ZH20}). 
\end{example}

\begin{example}[User-recommender system interaction]
\label{example:recsys}
Consider interactions between a recommender system and a user, where the recommender system gives a slate (or subset) of items $a \in \mathcal{A}$ to the user, and the user picks an action $b \in \mathcal{B}$ from the slate. When the user returns to the same content recommendation system (e.g. a Netflix/Hulu user with a profile) many times, this becomes a repeated game where both the recommendation system and user learn about their preferences \citep{hajiaghayi2023regret}. Again, misalignment could occur from the engagement metric being misaligned with user welfare \citep{MBH21} or for unintentional reasons (e.g., misspecification due to discrete \textit{thumbs up/thumbs down} user feedback, since true preferences are more nuanced). 
\end{example}

Examples \ref{example:llm}-\ref{example:recsys} motivate why our objective is to minimize regret for both the leader and the follower. First, we may inherently care about utility for the human, who could be either the leader (Example \ref{example:llm}) or the follower (Example \ref{example:recsys}). Secondly, we may also care about utility for the algorithmic tools: for example, a recommendation system that fails to make money may go out of business, or in certain cases, the chatbot/recommender system may better capture societal objectives than certain humans. 

\subsection{Assumptions on the follower's algorithm $\ALG_2$}\label{subsec:assumptionsplayer2}
Finally, we present some technical assumptions on the follower's algorithm. When we analyze regret in Sections \ref{sec:benchmarks}-\ref{sec:relaxed}, many of our constructions do not require that the follower run a particular algorithm, but instead allow the follower to run any algorithm that has sufficiently good performance along certain fine-grained performance metrics that capture the extent to which an algorithm's performance gracefully improves over time.  

These fine-grained performance metrics capture the follower's errors while learning. These errors are captured by the difference between $\MR{2}{a_t}{b_t}$ (the follower's realized mean reward) and the $\max_{b \in \mathcal{B}} \MR{2}{a_t}{b}$ (the best mean reward that the follower could achieve for the leader's action $a_t$). Note that this measure of suboptimality $\max_{b \in \mathcal{B}} \MR{2}{a_t}{b} - \MR{2}{a_t}{b_t}$ captures how well the follower is best-responding to the leader's action. This differs from our main notion of regret in Section \ref{subsec:regret}, which captures the follower's level of discontent relative to a fixed benchmark.

For intuition, we first describe these performance metrics---\textit{high-probability instantaneous regret} and \textit{high-probability anytime regret}---for a typical single-bandit learner which acts in isolation. For the single learner setting, instances $\mathcal{I}=(\mathcal{C},v)$ capture a single action set and a single value function. \emph{High-probability instantaneous regret} measures the suboptimality of the arms that the algorithm pick arms at each time step. More formally, an algorithm acting over an action space $\mathcal{C}$  satisfies high-probability instantaneous regret $g$ if for any instance $\mathcal{I} = (\mathcal{C},v)$: 
\[\mathbb{P}\left[\forall t \in [T]\mid v(c_{t}) \ge \max_{c \in \mathcal{C}} v(c) - g(t, T, \mathcal{C})\right] \ge 1 - T^{-3},\]
where the probability captures randomness in the algorithm and in the stochastic rewards. A \emph{high-probability anytime regret} bound guarantees that the regret bound for the algorithm holds with high probability at every time step $t$. More formally, an algorithm acting over an action space $\mathcal{C}$ satisfies high-probability anytime regret bound $h$ if for any instance $\mathcal{I}=(\mathcal{C},v)$, it holds that:
\[\mathbb{P}\left[\forall t \in [T] \mid  \sum_{t' \le t} v(c_{t'}) \ge \sum_{t' \le t} \max_{c \in \mathcal{C}} v(c) - h(t, T, \mathcal{C})\right] \ge 1 - T^{-3},\]
where the randomness is over the algorithm.\footnote{Compared with standard definitions of instaneous and anytime regret, we require a high-probability bound (rather than expectation). For anytime regret, we also require the bound for all $t$ for a given $T$ (rather than for all $T$).}

In a $\Setting$, we will require similar properties to hold for the follower's algorithm $\ALG_2$, but we take account how the algorithm $\ALG_2$ depends on the action $a_t$ which is selected by the leader's algorithm $\ALG_1$. Let $\NumPullsOneArg{a}{t+1}$ be the number of times that arm $a$ has been pulled prior to the $(t+1)$th time step. An algorithm $\ALG_2$ satisfies a high-probability instantaneous regret bound of $g$ if for any $\ALG_1$ chosen by the leader and any $\Instance = (\mathcal{A}, \mathcal{B}, v_1, v_2)$, it holds that: 
\[\mathbb{P}\left[\forall t \in [T], a \in \mathcal{A} \mid \MR{2}{a_t}{b_{t}} \ge \max_{b \in \mathcal{B}} \MR{2}{a_t}{b} - g(\NumPullsOneArg{a}{t+1}, T,  \mathcal{B})\right] \ge 1 - |\mathcal{A}| \cdot T^{-3}.\]
An algorithm $\ALG_2$ satisfies a high-probability anytime regret bound of $h$ if for any $\ALG_1$ chosen by the leader and any instance $\Instance = (\mathcal{A}, \mathcal{B}, v_1, v_2)$, it holds that: 
\[\mathbb{P}\left[\forall t \in [T], a \in \mathcal{A} \mid   \sum_{t' \le t \mid a_{t'} = a} \MR{2}{a}{b_{t'}} \ge \left(\sum_{t' \le t \mid a_{t'} = a} \max_{b \in \mathcal{B}} \MR{2}{a}{b}\right) - h(\NumPullsOneArg{a}{t+1}, T, \mathcal{B})\right] \ge 1 - |\mathcal{A}| \cdot T^{-3}.\]

In \Cref{sec:algorithmsassumptions}, we discuss the relationship between these metrics, the performance of standard algorithms for the follower on these metrics, and algorithms for the leader for more general $g$ and $h$. As an example, if the follower runs a separate instantiation of $\ActiveArmElimination$ (Algorithm \ref{algo:AAE}) on every arm $a \in \mathcal{A}$, this satisfies high-probability instantaneous regret $g(t, T, \mathcal{B}) = O(\sqrt{|\mathcal{B}| \cdot \log T / t})$ and high-probability anytime regret $h(t, T, \mathcal{B}) = O(\sqrt{|\mathcal{B}| \cdot t \cdot \log T})$ (\Cref{prop:AAE}).

\section{Stackelberg value is unachievable}\label{sec:stackelbergimposs}

In this section, we show that the natural benchmark given by the players' utilities in the underlying static Stackelberg game is unachievable. More formally, given an  instance $\mathcal{I} = (\mathcal{A}, \mathcal{B}, v_1, v_2)$, let $(a^*, b^*)$ be the Stackelberg equilibrium. We define the \textit{Stackelberg benchmarks} to be each player's utility at $(a^*, b^*)$, that is: 
$\benchmark{1} = \MR{1}{a^*}{b^*}$ and $ \benchmark{2} = \MR{2}{a^*}{b^*}$ (where the superscript \enquote{orig} denotes that this is the benchmark for original offline Stackelberg games). The following result illustrates that it is not possible to simultaneously achieve sublinear regret with respect to both players' regret.\footnote{There exists a simple algorithm in the centralized environment that achieves sublinear regret for Player $i$: run a sublinear regret multi-armed bandit algorithm on the arms $(a,b)$ using Player $i$'s stochastic rewards (ignoring the rewards of the other player).} 

\begin{restatable}{theorem}{worstlinear}\label{thrm:worstlinear}
Consider $\StrongSetting$s or $\WeakSetting$s. For any algorithms $\ALG_1$ and $\ALG_2$, there exists an instance $\Instance^*$ with $|\mathcal{A}| = |\mathcal{B}| = 2$ where at least one of the players incurs linear regret with respect to the Stackleberg benchmarks $\benchmark{1}$ and $\benchmark{2}$, that is: $\max(R_1(T; \Instance^*), R_2(T; \Instance^*)) = \Omega(T)$.
\end{restatable}
\begin{proof}[Proof sketch of Theorem \ref{thrm:worstlinear}]
It suffices to prove this lower bound in a \textit{centralized} environment where a single learner can choose action pairs $(a,b)$ and observes rewards for both players (Lemma \ref{lemma:centralizedimpliesdecentralized}). 
We show that on the instances $\Instance$ and $\tilde{\Instance}$ in Table \ref{tab:main_1} (with $\delta = O(1/\sqrt{T})$), at least one of the players incurs linear regret on at least one of the instances. The small value of $\delta$ means that the algorithm fails to distinguish between these instances with constant probability. Nonetheless, the benchmarks are very different: on instance $\Instance$,   $(a^*, b^*) = (a_1, b_1)$, $\benchmark{1} = 0.6$ and $\benchmark{2} = \delta > 0$; on instance $\tilde{\Instance}$, $(a^*, b^*) =  (a_2, b_1)$, $\benchmark{1} = 0.5$, and $\benchmark{2} = 0.6$. Intuitively, when the algorithm fails to distinguish between these instances, then it must  choose the same distribution over $\mathcal{A} \times \mathcal{B}$ on both $\Instance$ and $\tilde{\Instance}$. However, any such distribution either incurs constant loss for the leader on $\Instance$ or constant loss for the follower on $\tilde{\Instance}$. We formalize this proof in \Cref{appendix:worstlinear}. 
\end{proof}

The linear regret in Theorem \ref{thrm:worstlinear} is driven by \textit{misalignment} between the leader's utilities and the follower's utilities: small differences in the follower's utilities can lead to arbitrarily large differences in the leader's utilities. As a result, the suboptimal actions that the follower takes while learning are amplified in the leader's regret. This motivates the design of relaxed benchmarks that take into account the suboptimal actions players inevitably take while learning.

\begin{table}[ht!]
\centering
\begin{subtable}[b]{0.45\linewidth}
\centering
\begin{tabular}{|c|c|c|}
\hline
      & $b_1$                  & $b_2$                  \\ \hline
$a_1$ & $(0.6, \delta)$ & $(0.2, \textbf{0})$             \\ \hline
$a_2$ & $(0.5, 0.6)$       & $(0.4, 0.4)$ \\ \hline
\end{tabular}
\caption{Mean rewards $(\MR{1}{a}{b}, \MR{2}{a}{b})$ for $\Instance$}
\end{subtable}
\hfill
\begin{subtable}[b]{0.45\linewidth}
\centering
\begin{tabular}{|c|c|c|}
\hline
      & $b_1$                  & $b_2$                  \\ \hline
$a_1$ & $(0.6, \delta)$ & $(0.2, $\boldmath{$2\delta$}$)$             \\ \hline
$a_2$ & $(0.5, 0.6)$       & $(0.4, 0.4)$ \\ \hline
\end{tabular}
\caption{Mean rewards $(\MRTilde{1}{a}{b}, \MRTilde{2}{a}{b})$  for $\tilde{\Instance}$}
\end{subtable}
\caption{Two instances $\Instance$ (left) and $\tilde{\Instance}$ (right), which differ solely in the follower's reward for $(a_1, b_2)$ (shown in \textbf{bold}). 
For $\delta$ sufficiently small, the instances $\Instance$ and $\tilde{\Instance}$ are hard to distinguish and turn out to imply a $\Omega(T)$ lower bound on regret with respect to the original Stackelberg benchmarks (Theorem \ref{thrm:worstlinear}).
}
\label{tab:main_1}
\end{table}

\section{$\gamma$-tolerant benchmark and regret bounds}\label{sec:benchmarks}
Having shown that the Stackelberg equilibrium is unattainable, we next propose a novel benchmark and give tight sublinear regret bounds with respect to it.

\subsection{$\gamma$-tolerant benchmark}

Our benchmark is related to the Stackelberg Equilibrium, but adapted to account for the fact that both players are learning and cannot be counted on to exactly best respond. This benchmark is a function of the instance $\Instance = (\mathcal{A}, \mathcal{B}, v_1, v_2)$ but \textit{independent} of the learning algorithms for either player. At a high-level, we construct a set of \textit{approximate best responses} for each player and use these sets to construct more realistic benchmarks; within these sets, our benchmark will be \emph{tolerant} to suboptimality with respect to the other player.  

If the leader takes action $a$, then we define the follower's $\epsilon$-best-response set $\mathcal{B}_{\epsilon}(a)$ as:
\[\mathcal{B}_{\epsilon}(a) := \{b \in \mathcal{B} \mid \MR{2}{a}{b} \geq \max_{b' \in \mathcal{B}} \MR{2}{a}{b'} - \epsilon\}.\]
Defining the $\epsilon$-best response set $\mathcal{A}_{\epsilon}$ for the leader is more subtle. Informally, we define this set to include any action $a \in \mathcal{A}$ which has ``any chance'' of doing at least as well as the the leader's best action if the follower is $\epsilon$-best responding. Specifically, this includes actions $a$ where \textit{some} action $b \in \mathcal{B}_{\epsilon}(a)$ achieves utility close to $\max_{a' \in \mathcal{A}} \min_{b' \in \mathcal{B}_{\epsilon}(a)}\MR{1}{a'}{b'}$ for the leader:\footnote{At first glance, it might appear more natural to instead define $\mathcal{A}_{\epsilon}$ to be all actions where the follower's worst-case approximate best-response yields high utility for the leader, that is: $\{a \in \mathcal{A} \mid \textbf{min}_{b \in \mathcal{B}_{\epsilon}(a)}\MR{1}{a}{b} \ge \max_{a' \in \mathcal{A}} \min_{b' \in \mathcal{B}_{\epsilon}(a')} \MR{1}{a'}{b'} -\epsilon\}$.  However, this set does not necessarily contain the leader's best-response set $\left\{a \in \argmax_{a' \in \mathcal{A}} v_1(a', b^*(a'))\right\}$, which makes it a less natural definition of an approximate best-response set.}  
\[\mathcal{A}_{\epsilon} = \{a \in \mathcal{A} \mid \max_{b \in \mathcal{B}_{\epsilon}(a)}\MR{1}{a}{b} \ge \max_{a' \in \mathcal{A}} \min_{b' \in \mathcal{B}_{\epsilon}(a')} \MR{1}{a'}{b'} -\epsilon\}.\]
Observe that the set $\mathcal{B}_{\epsilon}(a)$ always contains the follower's best-response set $\left\{b \in \argmax_{b' \in \mathcal{B}} v_1(a, b')\right\}$, and furthermore approaches this best-response set in the limit as $\epsilon \rightarrow 0$; similarly, the set $\mathcal{A}_{\epsilon}$ always contains the leader's best-response set $\left\{a \in \argmax_{a' \in \mathcal{A}} v_1(a', b^*(a'))\right\}$ where ties are broken in favor of the follower, and furthermore approaches this best-response set in the limit as $\epsilon \rightarrow 0$.

We use these $\epsilon$-best-response sets to create the relaxed benchmarks for the leader and follower. In these benchmarks, we add an \textit{$\epsilon$-relaxed Stackelberg utility} term with a \textit{$\epsilon$-regularizer} term, and then take an infimum over all possible values $\epsilon \le \gamma$. In particular, the $\epsilon$-relaxed Stackelberg utility takes a $\max$ over the player's actions and a $\min$ over the other player's $\epsilon$-best response set; the regularizer adds a $\epsilon$ penalty for errors made by the other player. 
\begin{definition}
\label{def:relaxedbenchmarks}
Given a maximum tolerance $\gamma > 0$, we define the \textit{$\gamma$-tolerant benchmarks} $\benchmarkrelaxedstronger{1}$ and $\benchmarkrelaxedstronger{2}$ to be: 
\[\benchmarkrelaxedstronger{1} = \inf_{\epsilon \le \gamma} \Big(\underbrace{\max_{a \in \mathcal{A}} \min_{b \in \mathcal{B}_{\epsilon}(a)} \MR{1}{a}{b}}_{\text{$\epsilon$-relaxed Stackelberg utility}} + \underbrace{\epsilon}_{\text{$\epsilon$-regularizer}} \Big)\]
\[\benchmarkrelaxedstronger{2} = \inf_{\epsilon \le \gamma} \Big(\underbrace{\min_{a \in \mathcal{A}_{\epsilon}} \max_{b \in \mathcal{B}} \MR{2}{a}{b}}_{\text{$\epsilon$-relaxed Stackelberg utility}} + \underbrace{\epsilon}_{\text{$\epsilon$-regularizer}} \Big).  \]   
\end{definition}
We provide some high-level intuition for why our benchmark might be appropriate for learning environments. For small values of $\epsilon$, the $\epsilon$-best-response sets describe actions that (from the player's perspective) are similar and difficult to distinguish between while learning. The $\epsilon$-relaxed Stackelberg utility takes a worst-case perspective and takes a minimum over the other player's $\epsilon$-best-response set, since the player's choices within this set can be unpredictable while learning.\footnote{For the leader, \citet{bai2021sample} also takes a similar worst-case perspective over the $\epsilon$-best-response set, but does not introduce a regularizer or take a minimum over $\epsilon$.} The $\epsilon$-regularizer captures the player's intolerance of suboptimality of the other player (see \Cref{appendix:regularizers} for a discussion of regularizers and $\gamma$).

We illustrate these benchmarks in the following example. 

\begin{example}
\label{example:gammatolerant}
Consider the example in Table \ref{tab:main_2} (with $ 0.4> \gamma \geq 4 \delta$). In this case, the leader's benchmark is equal to the Stackelberg utility ($\benchmark{1} = \benchmarkrelaxedstronger{1} = 0.5+\delta$), while the \textit{follower's} benchmark is weaker ($\benchmark{2} = 0.4 >  3 \delta + \delta =\benchmarkrelaxedstronger{2}$), where the second $\delta$ comes from the regularizer. The intuition is that the leader's $\epsilon$-best-response set $\mathcal{A}_{\delta} = \{a_1, a_2\}$ contains both actions, even though $a_2$ is not a Stackelberg equilibrium, which noticeably lowers the follower's $\epsilon$-relaxed Stackelberg utility. In \Cref{app:benchmarkdiscuss}, we provide more detailed discussions of examples. 
\end{example}

\begin{table}[ht]
\centering
\begin{tabular}{|c|c|c|}
\hline
      & $b_1$                  & $b_2$                  \\ \hline
$a_1$ & $(0.5+\delta, 0.4)$ & $(0.2, 0)$             \\ \hline
$a_2$ & $(0.5, 3\delta)$       & $(0.4, 2\delta)$ \\ \hline
\end{tabular}
\caption{A single instance, illustrating the $\gamma$-tolerant benchmark (Example \ref{example:gammatolerant}). }
\label{tab:main_2}
\end{table}

\subsection{Linear regret for ExploreThenCommit}\label{subsec:explorethencommit}

We first show that out-of-box stochastic bandit algorithms do not directly provide sublinear regret against the $\gamma$-tolerant benchmark, where the challenge is that the leader's learning gets distorted when both players simultaneously learn. To demonstrate this, we consider \texttt{ExploreThenCommit} (Algorithm \ref{algo:explorethencommit}) and show that if both the leader and follower are running this algorithm, the regret could be linear for \textit{both players}. 

\paragraph{\ExploreThenCommit($E, \mathcal{C})$ (Algorithm \ref{algo:explorethencommit}).} The algorithm $\ALG = \ExploreThenCommit(E, \mathcal{C})$ takes as input $E \in [T]$ and a set of arms $\mathcal{C}$.\footnote{By setting $\mathcal{C} = \mathcal{A}$, this algorithm can be instantiated as $\ALG_1$ for the leader; by setting $\mathcal{C} = \mathcal{B}$, this algorithm can be instantiated as $\ALG_2$ for the follower.} When $\ALG$ is applied to an instance, for the first $|\mathcal{C}| \cdot E$ rounds, the algorithm $\ALG$ pulls each arm in $\mathcal{C}$ a total of $E$ times in a round-robin fashion. For the remaining $T - |\mathcal{C}| \cdot E$ rounds, the algorithm commits to the optimal empirical mean from the first $|\mathcal{C}| \cdot E$ rounds. This is a standard algorithm for multi-armed bandits \citep{S19, L20}.

\begin{algorithm2e}[htbp]
\caption{$\ExploreThenCommit(E, \mathcal{C})$ applied to history $\Hist$ (see e.g., \citep{S19, L20})}
\label{algo:explorethencommit}
\DontPrintSemicolon
\LinesNumbered
Fix an arbitrary ordering $\mathcal{C} = \left\{c^1, \ldots, c^{|\mathcal{C}|} \right\}$.\\
Let $t = |\Hist|$.\\
\tcc{Explore for the first $E \cdot |\mathcal{C}|$ rounds}
\If{$t \le E \cdot |\mathcal{C}|$}{
    Let $i = {t \pmod {|\mathcal{C}|}} + 1$ be the index of the action that should be pulled.\\
    \Return{point mass at $c^i$} \\
}
\tcc{Commit for the remaining rounds}
\If{$t > E \cdot |\mathcal{C}|$}{
    \tcc{Discard history all but the first $E \cdot |\mathcal{C}|$ rounds.}
    $\Hist^* = \left\{ (t', c_{t'}, r) \mid \exists (t', b_{t'}, r)  \in \Hist \text{ s.t. } t' \le E \cdot |\mathcal{C}| \right\}$ \\
    \tcc{Choose highest empirical mean.}
    \For{$c \in \mathcal{C}$}{
        Set $S(c) := \left\{r \mid \exists (t',  c_{t'}, r) \in\Hist^* \text{ s.t. } c = c_{t'} \right\}$ \tcp*{observed rewards}
        $\hatMROneArgNoPlayer{c} \leftarrow (\sum_{r \in S(c)} r)/|S(c)|$ \tcp*{compute empirical mean}
    }
    \Return{point mass at $\argmax_{c \in \mathcal{C}} \hatMROneArgNoPlayer{c}$} 
}
\end{algorithm2e}

When both players run $\ExploreThenCommit$, we show that if the leader's exploration phase ends before the follower's exploration phase, then both players can incur linear regret.  This lower bound holds for any maximum tolerance $\gamma \le 1$. 
\begin{restatable}{proposition}{decentralizedlower}
\label{prop:decentralizedlower}
Consider $\StrongSetting$s where the follower runs a separate instantiation of \\
$\texttt{ExploreThenCommit}(E, \mathcal{B})$ for every $a \in \mathcal{A}$. Moreover, suppose that the leader runs $\texttt{ExploreThenCommit}(E' \cdot |\mathcal{B}|, \mathcal{A})$ for any $E' \le E$ (i.e., the leader's exploration phase ends before the follower's exploration phase). Then, there exists an instance $\Instance^*$ such that both players incur linear regret with respect to the $\gamma$-tolerant benchmarks $\benchmarkrelaxedstronger{1}$ and $\benchmarkrelaxedstronger{2}$: that is, $\min(R_1(T; \Instance^*), R_2(T; \Instance^*)) = \Omega(T)$. 
\end{restatable}

\begin{proof}[Proof sketch of Proposition \ref{prop:decentralizedlower}]
 The intuition is that in the leader's exploration phase, the follower alternates uniformly between all actions $\mathcal{B}$. This distorts the leader's learning during the leader's exploration phase, and as a result, the leader can choose a highly suboptimal arm during the commit phase. This can lead to linear regret for both players. We construct a single instance (Table \ref{tab:main_2}, with $\delta = 0.1$) where both players incur linear regret. The full proof is deferred to \Cref{appendix:proofetcetcfails}.
\end{proof}

\subsection{Warmup Algorithm}

The constant regret in Proposition \ref{prop:decentralizedlower} motivates the design of more sophisticated algorithms where the leader waits for the follower to partially converge before starting to learn. As a warmup, we show that a simple modification of the setup of Proposition \ref{prop:decentralizedlower} guarantees sublinear regret (i.e., $O\left(|\mathcal{A}|^{1/3} |\mathcal{B}|^{1/3} (\log T)^{1/3} T^{2/3} \right)$ regret for both players). In this algorithm, both players run \texttt{ExploreThenCommit}-based algorithms, but the leader waits for the follower to finish exploring before starting to explore. More specifically, the leader runs \ExploreThenCommitThrowOut, which acts similar to \ExploreThenCommit, but with an extra exploration phase at the start, after which all rewards are thrown out. This initial phase is to allow the follower to partially converge. 

\paragraph{\ExploreThenCommitThrowOut($E, E', \mathcal{C})$ (Algorithm \ref{algo:explorethencommitthrowout}).} The algorithm 
$\ALG_1$ equal to \\ $\ExploreThenCommitThrowOut(E, E', \mathcal{C})$ takes as input $E, E' \in [T]$ and a set of arms $\mathcal{C}$. It throws out the first $E' \cdot |\mathcal{C}|$ rounds and then runs $\ExploreThenCommit(E, \mathcal{C})$.

\begin{algorithm2e}[htbp]
\caption{$\ExploreThenCommitThrowOut(E, E', \mathcal{C})$ applied to history $\Hist$}
\label{algo:explorethencommitthrowout}
\DontPrintSemicolon
\LinesNumbered
Fix an arbitrary ordering $\mathcal{C} = \left\{c^1, \ldots, c^{|\mathcal{C}|} \right\}$.
\tcc{Explore for first $E' \cdot |\mathcal{C}|$ rounds} 
\If{$t \le E' \cdot |\mathcal{C}|$}{
    Let  $i = {t \pmod {|\mathcal{C}|}} + 1$ be the index of the action that should be pulled. \\
    \Return{point mass at $c^i$} 
}
\tcc{Run ETC for the remainder of time, throwing out first $E' \cdot |\mathcal{C}|$ rounds} 
\If{$t > E' \cdot |\mathcal{C}|$}{
$\Hist^* = \left\{ (t' - E' \cdot |\mathcal{C}|, c_{t'}, r) \mid \exists (t', c_{t'}, r)  \in \Hist \text{ s.t. } t' > E' \cdot |\mathcal{C}| \right\}$ \tcp*{Throw out first $E' \cdot |\mathcal{C}|$ rounds of history} 
 \Return{$\ExploreThenCommit(E, \mathcal{C})$ applied to $\Hist^*$}
 }
\end{algorithm2e}

We show that if the follower runs $\ExploreThenCommit$ and the leader runs $\ExploreThenCommitThrowOut$, then both players achieve sublinear regret. For this result, we require that $\gamma$ is not \textit{too small}: $\gamma =  \omega\p{T^{-1/3}\p{\abs{\mathcal{A}} \cd  \abs{\mathcal{B}}}^{1/3}}$ (see \Cref{appendix:regularizers} for a discussion).
\begin{restatable}{theorem}{etcetc}
\label{thm:ETCETC} Consider a $\StrongSetting$ where the follower runs a separate instantiation of \\
$\ExploreThenCommit(E_2, \mathcal{B})$ for every $a \in \mathcal{A}$, and where the leader runs $\ExploreThenCommitThrowOut(E_1, E_2 \cdot |\mathcal{B}|, \mathcal{A})$. If $E_2 = \Theta(|\mathcal{A}|^{-2/3} |\mathcal{B}|^{-2/3} \cdot (\log T)^{1/3} T^{2/3})$, and $E_1 = \Theta(|\mathcal{A}|^{-2/3} \cdot (\log T)^{1/3} T^{2/3})$, then, the regret with respect to the $\gamma$-tolerant benchmarks is bounded as: 
\[\max(R_1(T), R_2(T)) = O\left(|\mathcal{A}|^{1/3} |\mathcal{B}|^{1/3} (\log T)^{1/3} T^{2/3} \right). \]
\end{restatable}
\begin{proof}[Proof sketch for Theorem \ref{thm:ETCETC}]
The ``throw out'' phase for $\ALG_1$ enables $\ALG_2$ to learn and commit to near-optimal actions. The meaningful exploration for $\ALG_1$ thus begins \textit{after} the follower has committed to actions. This enables $\ALG_1$ to identify a near-optimal action given the arms that $\ALG_2$ has already committed to after the first phase of exploration. Returning to our $\gamma$-tolerant benchmarks, for each player, we can upper bound regret by setting $\epsilon$ to be the suboptimality of the other player and achieve the desired regret bound. The full proof is deferred to Appendix \ref{appendix:proofetcetc}.

\end{proof}

One drawback of Theorem \ref{thm:ETCETC} is that requiring the follower to run a single algorithm is relatively restrictive. In the next subsection, we allow for a rich class of follower algorithms.

\subsection{Main Algorithm}

Our main result in this section is an adaptive algorithm for the leader (\ExploreThenUCB, Algorithm \ref{algo:explorethenucb}) that achieves the same regret bounds while permitting greater flexibility for the follower (Theorem \ref{thm:explorethenucb}). Specifically, we only require that the follower converges to $\epsilon$-optimal responses quickly, which we formalize through high-probability instantaneous regret (Section \ref{subsec:assumptionsplayer2}). Since the leader's algorithm needs to be robust to a broader range of follower behaviors, we replace the commit phase of $\ExploreThenCommit$ with an adaptive algorithm. This motivates $\ExploreThenUCB$, which explores in the first phase, and then runs a version of UCB on the arms $\mathcal{A}$. The initial exploration phase in $\ExploreThenUCB$, similar to the initial exploration phase in $\ExploreThenCommitThrowOut$, ensures that the leader waits for the follower to partially converge before starting to learn.

\paragraph{$\ExploreThenUCB(E)$ (Algorithm \ref{algo:explorethenucb}).} The algorithm $\ALG_1 = \ExploreThenUCB(E)$ takes as input $E \in [T/\abs{\mathcal{A}}]$. When $\ALG_1$ applied to an instance, for the first $|\mathcal{A}| \cdot E$ rounds, the algorithm $\ALG_1$ pulls each arm in $\mathcal{A}$ a total of $E$ times, and then discards all history from these rounds. For the remaining $T - |\mathcal{A}| \cdot E$ rounds, the algorithm runs UCB, computing the upper confidence bound $\UCBOneArg{1}{a} = \hatMROneArgTimeStep{1}{a}{t}  + \alpha_t(a)$ using confidence bound $\alpha_t(a) = \Theta\left(\sqrt{(\log T)/\NumPullsTwoTimesOneArg{a}{E \cdot |\mathcal{A}|}{t}}\right)$, where $\hatMROneArgTimeStep{1}{a}{t}$ is the empirical mean and $\NumPullsTwoTimesOneArg{a}{E \cdot |\mathcal{A}|}{t}$ is the number of times that action $a$ is chosen in the UCB phase after time step $E \cdot |\mathcal{A}|$ and prior to time step $t$. The algorithm then chooses the arm with maximum upper confidence bound. 

\begin{algorithm2e}[htbp]
\caption{$\ExploreThenUCB(E)$ applied to $\Hist$} 
\label{algo:explorethenucb}
\DontPrintSemicolon
\LinesNumbered
Fix an arbitrary ordering $\mathcal{A} = \left\{a^1, \ldots, a^{|\mathcal{A}|} \right\}$.\\
Let $t = |\Hist|$.\\
\tcc{Explore for the first $E \cdot |\mathcal{A}|$ rounds}
\If{$t \le E \cdot |\mathcal{A}|$}{
    Let $i = \lceil \frac{t}{E} \rceil$ be the index of the action that should be pulled.\\
    \Return{point mass at $a^i$} \\
}
\If{$t > E \cdot |\mathcal{A}|$} {
$\Hist^* = \left\{ (t' - E \cdot |\mathcal{A}|, a_{t'}, r) \mid \exists (t', a_{t'}, r)  \in \Hist \text{ s.t. } t' > E \cdot |\mathcal{A}| \right\}$ \tcp*{Throw out first $E \cdot |\mathcal{A}|$ rounds of history} 
Initialize $\hatMROneArg{1}{a} = 1$ for $a \in \mathcal{A}$. \tcp*{Initialize empirical means.} 
Initialize $\UCBOneArg{1}{a} = 1$ for $a \in \mathcal{A}$. \tcp*{Initialize UCB.}
\For{$a \in \mathcal{A}$}{
Set $S(a) := \left\{r \mid \exists (t', a_{t'}, \RR{1}{a_{t'}}{b_{t'}}{t'}) \in \Hist^* \text{ s.t. } a = a_{t'},  \RR{1}{a_{t'}}{b_{t'}}{t'} = r \right\}$  \tcp*{Observed rewards}
\If{$S(a) \neq \emptyset$} {
$\hatMROneArg{1}{a} \leftarrow (\sum_{r \in S(a)} r)/|S(a)|$ \tcp*{Empirical mean}
$\alpha(a) \leftarrow 10 \cdot \frac{\sqrt{\log T}}{\sqrt{|S(a)|}}$ \tcp*{confidence bound width}
$\UCBOneArg{1}{a} \leftarrow \min(1, \hatMROneArg{1}{a} + \alpha(a))$ 
}
}
Let $a^* = \argmax_{a \in \mathcal{A}} \left(\UCBOneArg{1}{a}\right)$. \tcp*{arm with max upper confidence bound} 
\Return{point mass at $a^*$}
}
\end{algorithm2e}

 Even though the rewards observed by the leader are \textit{not} stochastic (since the follower can pick different arms over time), we show if the leader runs $\ExploreThenUCB$ and 
the follower runs algorithms with sufficiently low high-probability instantaneous regret, then both players achieve $O\left(|\mathcal{A}|^{1/3} |\mathcal{B}|^{1/3} (\log T)^{1/3} T^{2/3} \right)$ regret. The assumptions on the follower's algorithm are satisfied by standard algorithms such as $\ActiveArmElimination$ (Algorithm \ref{algo:AAE}; Proposition \ref{prop:AAE}) and $\ExploreThenCommit$ (Algorithm \ref{algo:explorethencommit}; Proposition \ref{prop:ETC}). 
For this result, we require that the maximum tolerance $\gamma$ is not \textit{too small}: $\gamma =  \omega\p{T^{-1/3}\p{\abs{\mathcal{A}} \cd  \abs{\mathcal{B}}}^{1/3}}$ (see \Cref{appendix:regularizers} for a discussion).
\begin{restatable}{theorem}{explorethenucb}
\label{thm:explorethenucb}
Let $E = \Theta( |\mathcal{A}|^{-2/3} (|\mathcal{B}| \log T)^{1/3} T^{2/3})$. Consider a $\StrongSetting$, where $\ALG_{2}$ is any algorithm with high-probability instantaneous regret $g(t, T, \mathcal{B}) = O\left((|\mathcal{A}| |\mathcal{B}| \log T)^{1/3} T^{-1/3} \right)$ for $t > E$ and $g(t, T, \mathcal{B}) = 1$ for $t \le E$, and where $\ALG_1 = \ExploreThenUCB(E)$. Then, it holds that the regret with respect to the $\gamma$-tolerant benchmarks $\benchmarkrelaxedstronger{1}$ and $\benchmarkrelaxedstronger{2}$ is bounded as: 
\[\max(R_1(T), R_2(T)) = O\left(|\mathcal{A}|^{1/3} |\mathcal{B}|^{1/3} (\log T)^{1/3} T^{2/3} \right).\]
\end{restatable}

\begin{proof}[Proof sketch of Theorem \ref{thm:explorethenucb}]
The intuition is that the exploration phase of \texttt{ExploreThenUCB} ensures that all of the follower's actions have bounded suboptimality, and the UCB phase accounts for the follower changing which action they choose over time. In more detail, high-probability instantaneous regret guarantees that after the explore phase, all actions that the follower's chooses are within the $\epsilon$-best-response set $\mathcal{B}_{\epsilon^*}(a)$ for $\epsilon^* = \Theta((|\mathcal{A}| \cdot |\mathcal{B}| \cdot \log T)^{1/3} T^{-1/3})$. For the UCB phase, the main lemma (Lemma \ref{lemma:UCBbound}) is that if an arm $a \in \mathcal{A}$ is pulled, the empirical mean is at least $\max_{a \in \mathcal{A}} \min_{b \in \mathcal{B}_{\epsilon^*}(a)} \MR{1}{a}{b} - \Theta\left(\sqrt{\log T / \NumPullsTwoTimesOneArg{a}{E \cdot |\mathcal{A}|}{t}} \right)$ (the optimal utility for the leader when the follower worst-case $\epsilon$-best-responds minus the confidence set size). Lemma \ref{lemma:UCBbound} enables us to analyze the leader's cumulative reward from each arm $a \in \mathcal{A}$ and thus bound the leader's regret. For the follower's regret, Lemma \ref{lemma:UCBbound} enables us to bound the number of times that arms outside of $\mathcal{A}_{\epsilon}$ are chosen, which enables us to bound the follower's regret. We defer the full proof to \Cref{appendix:proofexplorethenucb}.  
\end{proof}

\subsection{Regret lower bound}\label{subsec:regretlowerbound}

A natural question is whether the regret bound in Theorem \ref{thm:explorethenucb} can be improved from $\tilde{O}(T^{2/3})$ to $\tilde{O}(\sqrt{T})$, given that such dependence is possible in single-player bandit problems. Interestingly, we show a \textit{lower bound} of $T^{2/3}$ with respect to the $\gamma$-tolerant benchmarks, thus demonstrating that the dependence on $T$ in Theorem \ref{thrm:dlower} is near-optimal. This lower bound holds for any maximum tolerance $\gamma \le 1$. 
\begin{restatable}{theorem}{dlower}\label{thrm:dlower}
Consider $\StrongSetting$s or $\WeakSetting$s with action sets  $\mathcal{A}$ and $\mathcal{B}$ such that $|\mathcal{A}| \ge 2$ and $|\mathcal{B}| \ge 2$. For any algorithms $\ALG_1$ and $\ALG_2$, there exists an instance $\Instance^* = (\mathcal{A}, \mathcal{B}, v_1, v_2)$ such that at least one of the players incurs $\Omega(T^{2/3} \cd \p{\abs{\mathcal{B}}}^{1/3})$ regret with respect to the $\gamma$-tolerant benchmarks $\benchmarkrelaxedstronger{1}$ and $\benchmarkrelaxedstronger{2}$:
\[\max(R_1(T; \Instance^*), R_2(T; \Instance^*)) = \Omega(T^{2/3} \cd \p{\abs{\mathcal{B}}}^{1/3}).\]
\end{restatable}

\begin{proof}[Proof sketch]
In this sketch, we give intuition for a weaker bound of $\Omega(T^{2/3})$, deferring the strengthening to $\Omega(T^{2/3} \cd \p{\abs{\mathcal{B}}}^{1/3})$ to \Cref{appendix:dlower}. Like in the proof of Theorem \ref{thrm:worstlinear}, it suffices to consider a centralized environment (Lemma \ref{lemma:centralizedimpliesdecentralized}).
We show that on the  $\Instance$ and $\tilde{\Instance}$ in Table \ref{tab:main_4} (with $\delta = \Theta(T^{-1/3})$), at least one player incurs $\Omega(T^{2/3})$ regret on at least one of these instances.
 The only way to distinguish the instances is to pull $(a_1, b_2)$ at least $\Omega(T^{2/3})$ times, which gives low utility for both players. Intuitively, when the algorithm fails to distinguish $\Instance$ and $\tilde{\Instance}$, the algorithm must choose the same distribution over $\mathcal{A} \times \mathcal{B}$, but this gives $\Theta(T^{-1/3})$ loss for the leader on $\Instance$ or $\Theta(T^{-1/3})$ loss for the follower on $\tilde{\Instance}$. The full proof, which relies on a KL-divergence argument, is deferred to \Cref{appendix:dlower}.        
\end{proof}

\begin{table}[ht!]
\centering
\begin{subtable}[b]{0.45\linewidth}
\centering
\begin{tabular}{|c|c|c|}
\hline
      & $b_1$                  & $b_2$                  \\ \hline
$a_1$ & $(0.5 + \delta, \delta)$ & $(0, \textbf{0})$             \\ \hline
$a_2$ & $(0.5, 3 \delta)$       & $(0.5, 3 \delta)$ \\ \hline
\end{tabular}
\caption{Mean rewards $(\MR{1}{a}{b}, \MR{2}{a}{b})$ for $\Instance$}
\label{tab:sub1_main4}
\end{subtable}
\hfill
\begin{subtable}[b]{0.45\linewidth}
\centering
\begin{tabular}{|c|c|c|}
\hline
      & $b_1$                  & $b_2$                  \\ \hline
$a_1$ & $(0.5 + \delta, \delta)$ & $(0, $\boldmath{$2\delta$}$)$             \\ \hline
$a_2$ & $(0.5, 3 \delta)$       & $(0.5, 3 \delta)$ \\ \hline
\end{tabular}
\caption{Mean rewards $(\MRTilde{1}{a}{b}, \MRTilde{2}{a}{b})$ for $\tilde{\Instance}$}
\label{tab:sub2_main4}
\end{subtable}
\caption{Two instances $\Instance$ (left) and $\tilde{\Instance}$ (right), which differ solely in the follower's reward for $(a_1, b_2)$ (shown in \textbf{bold}). For $\delta$ sufficiently small, the instances $\Instance$ and $\tilde{\Instance}$ are hard to distinguish and turn out to imply a $\Omega(T^{2/3})$ lower bound on regret with respect to the $\gamma$-tolerant benchmark (Theorem \ref{thrm:dlower}). }
\label{tab:main_4}
\end{table}

At a high-level, the $T^{2/3}$ regret bound in Theorem \ref{thrm:dlower} is driven by the need to obtain precise estimates of \textit{highly suboptimal action pairs} in order to learn to distinguish between two instances. This is fundamentally different from single learner environments, where the learner only needs to obtain precise estimates of \textit{near-optimal} arms. Our regret upper bound (Theorem \ref{thm:explorethenucb}) and lower bound (Theorem \ref{thrm:dlower}) have near-matching dependence on $T$ and $|\mathcal{B}|$, but a gap in dependence on $|\mathcal{A}|$ (the upper bound scales with $|\mathcal{A}|^{1/3}$ while the lower bound is independent of $|\mathcal{A}|^{1/3}$). An interesting direction for future work is to close this gap. 

\section{Relaxed Settings with Faster Learning}\label{sec:relaxed}

The lower bound in the previous section showed that $\Theta(T^{2/3})$ regret is optimal for the benchmarks $\benchmarkrelaxedstronger{1}$ and $\benchmarkrelaxedstronger{2}$ for general instances. Since a $T^{2/3}$ lower bound is atypical for $K$-armed bandits problems, we next consider relaxed environments under which faster learning----i.e., $O(\sqrt{T})$ regret---is possible. In the first environment, we consider well-behaved instances (Section \ref{subsec:lipschitz}) and in the second environment, we weaken the benchmarks (Section \ref{subsec:weakerbenchmark}). In both environments, we show that the learner does not need to worry about their learning being overly distorted by the follower; thus, the leader can start learning immediately, even before the follower's actions have partially converged, which leads to improved regret bounds. The algorithms that we design for the leader are variants of UCB.

\subsection{Continuity condition on utilities}\label{subsec:lipschitz}

We first show that improved regret bounds are possible with a continuity condition on the player utilities. For intuition, the example in Table \ref{tab:main_1} gave a ``hard'' example resulting in linear regret in Theorem \ref{thrm:worstlinear} and the related example in Table \ref{tab:main_4} resulted in $\Omega(T^{2/3})$ regret in Theorem \ref{thrm:dlower}. These examples relied on two outcomes with nearly identical utilities for the follower having significantly different utilities for the leader, which could be viewed as a violation of continuity. This suggests that if arms that are \textit{sufficiently different} for the leader were also \textit{sufficiently different} for the follower, then it might be possible to beat the regret lower bound from Theorem \ref{thrm:worstlinear} and Theorem \ref{thrm:dlower}. 

We formalize continuity as follows: given an instance $\Instance = (\mathcal{A}, \mathcal{B}, v_1, v_2)$, we define the Lipschitz constant $L^*$ by\footnote{In the case of ties in rewards, if the numerator and denominator are both 0, we define $\frac{|\MR{i}{a}{b} - \MR{i}{a'}{b'}|}{|\MR{j}{a}{b} - \MR{j}{a'}{b'}|}$ to be 1 (because both players agree the elements are equivalent). If the denominator is 0 and numerator is nonzero, we define this fraction to be $\infty$ (because the items are indistinguishable to one player, while they give different rewards to the other).}: 
\[ L^* = \sup_{i \neq j \in \left\{1, 2\right\}} \sup_{(a, b) \neq (a', b')} \frac{|\MR{i}{a}{b} - \MR{i}{a'}{b'}|}{|\MR{j}{a}{b} - \MR{j}{a'}{b'}|}. \]
For example, when the two players have the same utilities (i.e., $v_1 = v_2$), then $L^* = 1$. More generally, our continuity condition captures the extent to which players agree on which outcomes are different from each other (a more detailed discussion is given in \Cref{app:continous}). Returning to the examples in Tables \ref{tab:main_1}, \ref{tab:main_4}, the ``hard'' instances yielding  linear regret for the original Stackelberg benchmarks (Theorem \ref{thrm:worstlinear}; Table \ref{tab:main_1}) have $L^* = \Theta(T^{-1/2})$ 
and the corresponding ``hard'' instances for $T^{2/3}$ regret for the $\gamma$-tolerant benchmarks (Theorem \ref{thrm:dlower}; Table \ref{tab:main_4}) require that $L^* = \Theta(T^{-1/3})$; in contrast, we focus on utility functions where $L^*$ is a constant.

When $L^*$ is bounded, we show that it is possible for both players to achieve $\tilde{O}(\sqrt{T})$ regret even with respect to the \textit{original Stackelberg benchmarks}. The follower can run any algorithm $\ALG_2$ with sufficiently low high-probability anytime regret (e.g., $\ActiveArmElimination$ as in \Cref{prop:AAE} or UCB as in \Cref{prop:UCB}). 
We construct another UCB-based algorithm $\LipschitzUCB$ (Algorithm \ref{algo:lipschitzucb}) for the leader, which expands the confidence sets based on the Lipschitz constant $L^*$.  

\paragraph{$\LipschitzUCB(L, C)$ (Algorithm \ref{algo:lipschitzucb}).} The algorithm $\ALG_1 = \LipschitzUCB(L, C)$ takes as inputs parameters $L$ and $C$. (The parameter $L$ is intended to be an upper bound on the Lipschitz constant $L^*$, and the parameter $C'$ is intended to be such that $\ALG_2$ satisfies anytime regret bound $h(t, T, \mathcal{B}) = \sqrt{C t \log T}$, where $C = C' \cdot \sqrt{|\mathcal{B}|}$ for a constant $C'$.) For each arm $a\in \mathcal{A}$, the algorithm computes UCB estimates $\UCBOneArg{1}{a}$ of the quantity $\max_{b \in \mathcal{B}} \MR{1}{a}{b}$ using the high-probability anytime regret bounds of $\ALG_2$ as well as the upper bound on the Lipschitz constant. The algorithm then chooses the arm $a_t = \argmax_{a \in \mathcal{A}} \max_{b \in \mathcal{B}'(a)} \UCBOneArg{1}{a}$. 

\begin{algorithm2e}[htbp]
\caption{$\LipschitzUCB(L, C)$ applied to $\Hist$} 
\label{algo:lipschitzucb}
\DontPrintSemicolon
\LinesNumbered
Initialize $\hatMROneArg{1}{a} = 1$ for $a \in \mathcal{A}$. \tcp*{Initialize empirical means for $\max_{b \in \mathcal{B}} \MR{1}{a}{b}$.} 
Initialize $\UCBOneArg{1}{a} = 1$ for $a \in \mathcal{A}$. \tcp*{Initialize UCB for $\max_{b \in \mathcal{B}} \MR{1}{a}{b}$}
\For{$a \in \mathcal{A}$}{
Set $S(a) := \left\{r \mid \exists (t', a_{t'}, r) \in \Hist \text{ s.t. } a = a_{t'} \right\}$ \tcp*{Observed rewards}
\If{$S(a) \neq \emptyset$} {
$\hatMROneArg{1}{a} \leftarrow (\sum_{r \in S(a)} r)/|S(a)|$ \tcp*{Empirical mean}
$\alpha(a) \leftarrow \frac{10 \sqrt{\mathcal{B} \log T}}{\sqrt{|S(a)|}} + C \cdot L \cdot \frac{\sqrt{\log T}}{\sqrt{|S(a)|}}$ \tcp*{confidence bound width}
$\UCBOneArg{1}{a} \leftarrow \min(1, \hatMROneArg{1}{a} + \alpha(a))$ 
}
}
Let $a^* = \argmax_{a \in \mathcal{A}} \left(\UCBOneArg{1}{a} \right)$. \tcp*{arm with max upper confidence bound} 
\Return{point mass at $a^*$}.
\end{algorithm2e}

We obtain the following regret bound with respect to the Stackelberg benchmark, our strongest benchmark. 
\begin{restatable}{theorem}{dpL}\label{thrm:dpL}
Consider a $\StrongSetting$ where $\Instance = (\mathcal{A}, \mathcal{B}, v_1, v_2)$ has Lipschitz constant $L^*$. Let $\ALG_2$ be any algorithm satisfying high-probability anytime regret $h(t, T, \mathcal{B}) = C' \sqrt{|\mathcal{B}| t \log T}$ where $C'$ is a constant, and let $\ALG_1 = \LipschitzUCB(L, C' \sqrt{|\mathcal{B}|})$ for any $L \ge L^*$. Then both players achieve the following regret bounds with respect to the original Stackelberg benchmarks $\benchmark{1}$ and $\benchmark{2}$: that is, $R_1(T; \Instance) = O\left(L \sqrt{T |\mathcal{A}| |\mathcal{B}| \log T}\right)$ and $R_2(T; \Instance) = O\left(L^2 \sqrt{T  |\mathcal{A}| \cdot |\mathcal{B}| \log T} \right)$. 
\end{restatable}
\begin{proof}[Proof sketch for Theorem \ref{thrm:dpL}]
The intuition is the continuity conditions imply that small errors by the follower translate into bounded suboptimality for the leader (and vice versa); moreover, the high-probability anytime regret requirements bound the follower's errors. Together, these properties guarantee that the leader's empirical mean $\hatMROneArg{1}{a}$ for each arm $a \in \mathcal{A}$ is close to the mean reward $\MR{1}{a}{b^*(a)}$ that they would receive if the follower best-responded: in more detail, the main lemma (Lemma \ref{lemma:UCBconfidenceboundscorrect}) is that the empirical mean $\hatMROneArg{1}{a}$ is at least $\MR{1}{a}{b^*(a)} - \Theta(L \sqrt{\log T}/ \sqrt{\NumPullsOneArg{a}{t}})$, where $\NumPullsOneArg{a}{t}$ is the number of times that arm $a$ has been pulled prior to time step $t$. Using 
Lemma \ref{lemma:UCBconfidenceboundscorrect} to bound the suboptimality of the leader's choice of actions $a_t \in \mathcal{A}$ and using the anytime regret requirements to bound the follower's suboptimality, we can bound both the leader's regret and the follower's regret. We defer the full proof to \Cref{appendix:lipschitzproof}.
\end{proof}

Finally, we compare our continuity condition and results with those in other works. Our continuity condition bears resemblance to the restrictions on utilities in \citet{camara2020mechanisms, collina2023efficient}: in fact, our conditions are conceptually stronger since we require Lipschitz continuity across \textit{all} pairs of actions rather only for near-optimal actions. However, Theorem \ref{thrm:dpL} is not directly comparable with the results in \citet{camara2020mechanisms, collina2023efficient} since we consider a stronger benchmark (the original Stackelberg benchmark) and also restrict to stochastic rewards.  An interesting direction for future work would be to relax the Lipschitz continuity assumptions in our work, perhaps borrowing intuition from the stable action requirement of \citet{collina2023efficient}.

\subsection{Weaker benchmark}\label{subsec:weakerbenchmark}
Finally, we will consider the case where utilities are allowed to be arbitrary ($L^*$ can be unbounded), but where we compete with weakened benchmarks, which we call \textit{self-$\gamma$-tolerant}. These benchmarks capture the case where the player is not only tolerant of suboptimality the other player, but also tolerant of their own suboptimality. We thus take a min over the $\epsilon$-best-response sets of \textit{both} players. 
\begin{definition}
\label{def:weakerrelaxedbenchmarks}
Given a maximum tolerance $\gamma > 0$, we define the \textit{self-$\gamma$-tolerant benchmarks}, $\benchmarkrelaxedweaker{1}$ and $\benchmarkrelaxedweaker{2}$, to be: 
\[\benchmarkrelaxedweaker{1} = \inf_{\epsilon \le \gamma} \Big(\min_{a \in \mathcal{A}_{\epsilon}}\min_{b \in \mathcal{B}_{\epsilon}(a)} \MR{1}{a}{b} + \epsilon \Big)\]
\[\benchmarkrelaxedweaker{2} = \inf_{\epsilon \le \gamma} \Big(\min_{a \in \mathcal{A}_{\epsilon}}\min_{b \in \mathcal{B}_{\epsilon}(a)} \MR{2}{a}{b} + \epsilon \Big).  \]   
\end{definition}

The tolerance of a player to their own suboptimality is the key difference from the $\gamma$-tolerant benchmarks from Section \ref{sec:benchmarks}. For the follower, the benchmarks behave similarly: for a given value of $\epsilon$, moving from $\max_{b \in \mathcal{B}}\MR{2}{a}{b}$ to $\min_{b \in \mathcal{B}_{\epsilon}(a)}\MR{2}{a}{b}$ 
differs by only an additive value of $\epsilon$. However for the leader, there is a conceptual difference: the value $\textbf{min}_{a \in \mathcal{A}_{\epsilon}} \min_{b \in \mathcal{B}_{\epsilon}(a)} \MR{1}{a}{b} + \epsilon$ is \textit{not} necessarily within $\epsilon$ of $\textbf{max}_{a \in \mathcal{A}} \min_{b \in \mathcal{B}_{\epsilon}(a)} \MR{1}{a}{b} + \epsilon $. This is because $\mathcal{A}_{\epsilon}$ includes any action $a$ that achieves high reward for \textit{some} (near-optimal) actions by the follower, even if the worst-case (near-optimal) action by the follower yields arbitrarily low reward for the leader. As an illustration, the ``hard'' instances specified in Table \ref{tab:main_4} with $\delta = \Theta(T^{-1/3})$ led to the $T^{2/3}$ regret bound. The self-tolerant benchmark $\benchmarkrelaxedweaker{1}$ reduces to $0.2 + \delta$ (rather than 0.5), so choosing $(a_1, b_2)$ no longer results in constant loss for the leader.

\begin{examplecontinued}
\label{example:selftolerantdetailed}
Let's again consider $\Instance$ in Table \ref{tab:main_2}, which we also used to illustrate the $\gamma$-tolerant benchmark. The minimum is again attained at $\epsilon = \delta$, but the benchmark values change to $\benchmarkrelaxedweaker{1} = 0.4+\delta$ and $\benchmarkrelaxedweaker{2} = 2 \cd \delta + \delta$. The intuition is that the self-$\gamma$-tolerance benchmark only requires each agent to compete with the \emph{worst} element within the product set $\mathcal{A}_{\delta} \times \mathcal{B}_{\delta}(a)$. Note that the resulting benchmark differs from the $\gamma$-tolerant benchmark for the follower only by $\delta$, but differs by $0.4$ (a constant) for the leader. We provide a detailed derivation of this example  along with diagrams illustrating richer examples
in \Cref{app:benchmarkdiscuss}. 
\end{examplecontinued}

For the self-tolerant benchmarks, we show it is possible to achieve $\tilde{O}(\sqrt{T |\mathcal{A}| |\mathcal{B}|})$ regret for both players (Theorem \ref{thm:UCBweaker}), which outperforms the $T^{2/3}$ lower bound for the stronger benchmark shown in Theorem \ref{thrm:dlower}. To demonstrate this is feasible,  we focus on $\WeakSetting$s, and we construct a specific pair of algorithms that achieve a $O(\sqrt{T})$ regret upper bound. For the follower, we take the algorithm $\ALG_2$ to be \texttt{ActiveArmElimination} (Algorithm \ref{algo:AAE}), which cycles through phases of exploration, after which all \textit{sufficiently suboptimal arms} are eliminated. For the leader, we construct a UCB-based algorithm $\PhasedUCB$ (Algorithm \ref{algo:phaseducb})  which constructs confidence bounds for every pair of actions  $(a,b)$.

\paragraph{$\PhasedUCB(M_1, \ldots, M_P)$ (Algorithm \ref{algo:phaseducb}).} The algorithm $\ALG_1 = \PhasedUCB(M_1, \ldots, M_P)$ takes as input the parameters $M_1, \ldots, M_P \ge 0$. (The parameter $M_i$ is intended to capture the number of times that an arm is pulled in phase $i$ by the instantiation of $\ActiveArmElimination$ specified by $\ALG_2$.) The algorithm $\ALG_1$ computes UCB estimates $\UCBMR{1}{a}{b}$ for $\MR{1}{a}{b}$, computes the set of active arms $\mathcal{B}'(a)$ in the previous phase of $\ALG_2$'s instantiation of $\ActiveArmElimination$ for each arm $a$ ($\ComputeQuantities$, Algorithm \ref{algo:computequantities}), and chooses the arm with maximum UCB: $a_t = \argmax_{a \in \mathcal{A}} \max_{b \in \mathcal{B}'(a)} \UCBMR{1}{a}{b}$. $\ComputeQuantities$ computes the active arms $\mathcal{B}'(a)$ by iterating through $\Hist$ and keeping track of whenever a new phase is entered using the parameters $M_1, \ldots, M_P$.

\begin{algorithm2e}[htbp]
\caption{$\PhasedUCB(M_1, \ldots, M_P)$ applied to $\Hist$}
\label{algo:phaseducb}
\DontPrintSemicolon
\LinesNumbered
Let $\hatMR{1}{a}{b} = 0$ for $a \in \mathcal{A}$ and $b \in \mathcal{B}$.\tcp*{initialize empirical mean of $\MR{1}{a}{b}$}
Let $\UCBMR{1}{a}{b}  = 1$ for $a \in \mathcal{A}$ and $b \in \mathcal{B}$. \tcp*{initialize UCB for $\MR{1}{a}{b}$} 
Let $\mathcal{B}'(a) = \ComputeQuantities(M_1, \ldots, M_P, \Hist)$. \tcp*{active arms in previous phase for $\ALG_2$}
\For{$a \in \mathcal{A}$}{
    \For{$b \in \mathcal{B}$}{
        Set $S(a,b) := \left\{r \mid \exists (t', a_{t'}, b_{t'}, r) \in \Hist \text{ s.t. } a = a_{t'}, b = b_{t'} \right\}$ \tcp*{observed rewards}
        \If{$S(a,b) \neq \emptyset$}{
            $\hatMR{1}{a}{b} \leftarrow (\sum_{r \in S(a,b)} r)/|S(a,b)|$ \tcp*{compute empirical mean}
            $\alpha(a,b) := 10  \cdot \sqrt{\frac{\log T}{|S(a,b)|}}$ \tcp*{confidence bound width}
            $\UCBMR{1}{a}{b} \leftarrow \min\left(1,  \hatMR{1}{a}{b} + \alpha(a,b) \right)$ \tcp*{compute UCB}  
        }
    }
}
Let $a^* = \argmax_{a \in \mathcal{A}} \max_{b \in \mathcal{B}'(a)} \left(\UCBMR{1}{a}{b} \right)$. \tcp*{arm with max upper confidence bound for any valid $b$} 
\Return{point mass at $a^i$} 
\end{algorithm2e}

\begin{algorithm2e}[htbp]
\caption{$\ComputeQuantities(M_1, \ldots, M_P, \Hist)$} 
\label{algo:computequantities}
\DontPrintSemicolon
\LinesNumbered
Initialize $s'(a) = 0$ for $a \in \mathcal{A}$. \tcp*{Index of the last completed phase for $\ALG_2$ on arm $a$.}
Initialize $t'(a) = 1$ for $a \in \mathcal{A}$. \tcp*{Time step marking beginning of phase $s'+1$ for $\ALG_2$ on arm $a$.}
Initialize $\mathcal{B}'(a) = \mathcal{B}$. \tcp*{Active arms in phase $s'$ for $\ALG_2$ on arm $a$.}
Initialize $newphase_a = \text{False}$ for $a \in \mathcal{A}$. \tcp*{Boolean for first time step in phase for $\ALG_2$ on  $a$.}
Let $t = |\Hist|$. \\
\For{$t'' = 1$ \KwTo $t$}{
    \For{$a \in \mathcal{A}$}{
        \For{$b \in \mathcal{B}$}{
            Let $n(a, b) := \left|\left\{(t'', a_{t''}, b_{t''}, \RR{1}{a_{t''}}{b_{t''}}{t''}) \in \Hist \mid a_{t''} = a, b_{t''} = b, t'' \ge t'_a \right\}\right|$. \\
            \If {$n(a,b) > M_{s'_a + 1}$}{
                $newphase_a = \text{True}$.
            }
        }
        \If{$newphase_a = \text{True}$}{
            Update $\mathcal{B}'(a) \leftarrow \left\{b \in \mathcal{B} \mid \exists (t'', a_{t''}, b_{t''}, \RR{1}{a_{t''}}{b_{t''}}{t''}) \in \Hist \text{ s.t. } t'_a \le t'' < t,  a_{t''} = a, b_{t''} = b \right\}$ \\
            Update $s'(a) \leftarrow s'(a) + 1$. \\
            Update $t'(a) \leftarrow t$. \\
            $newphase_a = \text{False}$.
        }
    }
}
\Return{$\left\{\mathcal{B}'(a)\right\}_{a \in \mathcal{A}}$}.
\end{algorithm2e}

We show that both players achieve $O(\sqrt{T})$ regret. For this result, we require that the $\gamma$ is not \textit{too small}: $\gamma =  \Omega(T^{-1/4} (|\mathcal{A}| \cd  |\mathcal{B}| \cdot \log T)^{1/2}))$ (see \Cref{appendix:regularizers} for a discussion).
\begin{restatable}{theorem}{UCBweaker}\label{thm:UCBweaker}
Consider a $\WeakSetting$, where for each $a \in \mathcal{A}$, the algorithm $\ALG_2$ runs a separate instantiation of $\ActiveArmElimination$ with parameters $M_1, \ldots, M_P$ (where $M_i = \Theta(\log T \cdot 2^{2i})$ denotes the number of times that each arm is pulled in phase $i$). Let $\ALG_1 = \PhasedUCB(M_1, \ldots, M_P)$. Then it holds that the regret with respect to the self-$\gamma$-tolerant benchmarks $\benchmarkrelaxedweaker{1}$ and $\benchmarkrelaxedweaker{2}$ is bounded as:
\[\max(R_1(T), R_2(T)) = O\left(\sqrt{|\mathcal{A}| \cdot |\mathcal{B}| \cdot T \cdot \log T}\right).\] 
\end{restatable}
\begin{proof}[Proof sketch for Theorem \ref{thm:UCBweaker}]
The intuition is that the benchmark allows the leader to choose any $a \in \mathcal{A}_{\epsilon}$. The definition of $\mathcal{A}_{\epsilon}$ means that the leader can take an optimistic perspective on the follower's choice of action $\mathcal{B}_{\epsilon}(a)$ (and not have to prepare for the worst-case action in $\mathcal{B}_{\epsilon}(a)$). This optimistic perspective surfaces in $\PhasedUCB$ in terms of how the leader evaluates an action $a$ based on the maximum UCB $\max_{b \in \mathcal{B}'(a)} \UCBMR{1}{a}{b}$ across  all active arms $b \in \mathcal{B}'(a)$. To analyze this pair of algorithms, we show a bound $\epsilon_t$ for each time step $t$ such that $a_t$ is an $\epsilon_t$-best-response for the leader and $b_t$ is an $\epsilon_t$-best-response for the follower: the main lemma (Lemma \ref{lemma:UCBweakermain}) shows that we can set $\epsilon_t$ to be $\Theta(\sqrt{|\mathcal{B} \cdot \log T / \NumPullsOneArg{a_t}{t}})$ where $\NumPullsOneArg{a_t}{t}$ is the number of times that arm $a_t$ has been pulled prior to time step $t$.  The full proof is deferred to \Cref{appendix:proofUCBweaker}.
\end{proof}

The regret bound in Theorem \ref{thm:UCBweaker} is nearly optimal, as we show in the following $\Omega(\sqrt{T |\mathcal{A}| \cd |\mathcal{B}|})$ lower bound for self-$\gamma$-tolerant benchmarks, which holds for any maximum tolerance $\gamma \le 1$. 
\begin{restatable}{proposition}{lowerboundsqrtt}\label{prop:lowerboundsqrtt}
Consider $\StrongSetting$s or $\WeakSetting$s with actions sets $\mathcal{A}$ and $\mathcal{B}$ such that $|\mathcal{A}| \ge 2$ and $|\mathcal{B}| \ge 2$. For any algorithms $\ALG_1$ and $\ALG_2$, there exists an instance $\Instance^* = (\mathcal{A}, \mathcal{B}, v_1, v_2)$ such that at least one of the players incurs $\Omega(\sqrt{T \cd (\abs{\mathcal{A}} - 1) \cd \abs{\mathcal{B}}})$ regret with respect to the self-$\gamma$-tolerant benchmarks $\benchmarkrelaxedweaker{1}$ and $\benchmarkrelaxedweaker{2}$, that is: $\max(R_1(T; \Instance^*), R_2(T; \Instance^*)) = \Omega(\sqrt{T \cd (\abs{\mathcal{A}} - 1) \cd \abs{\mathcal{B}}})$. 
\end{restatable}
\noindent Taken together, Theorem \ref{thm:UCBweaker} and Proposition \ref{prop:lowerboundsqrtt} demonstrate the self-$\gamma$-tolerant benchmarks lead to $\tilde{\Theta}(\sqrt{T |\mathcal{A}| \mathcal{B}|})$ regret bounds for each player.

We note that Theorem \ref{thm:UCBweaker} requires  the follower to run a specific algorithm: this contrasts with our results for the $\gamma$-tolerant benchmark (Theorem \ref{thm:explorethenucb}) and the Lipschitz benchmark (Theorem \ref{thrm:dpL}) which allowed for greater flexibility in the follower algorithm. An interesting direction for future work would be to design a leader algorithm for the self-$\gamma$-tolerant benchmark that permits a richer family of follower behaviors.

\section{Discussion of Benchmark Parameters}\label{appendix:regularizers}

Our relaxed benchmarks---the $\gamma$-tolerant benchmarks (Definition \ref{def:relaxedbenchmarks}) and the self-$\gamma$-tolerant benchmarks (Definition \ref{def:weakerrelaxedbenchmarks})---depend on two parameters: (1) the maximum tolerance $\gamma$ and (2) the $\epsilon$-regularizer. In this section, we discuss the role of each parameter and describe extensions of our results to alternate settings of these parameters. 

\subsection{Maximum Tolerance $\gamma$}\label{subsec:maxtol}

The value $\gamma$ intuitively captures the players' maximum tolerance for suboptimality. Taking $\gamma$ to be small makes our benchmarks more challenging, because it reduces the space of permissible suboptimality levels $\epsilon$ over which the infimum is taken. In contrast, taking $\gamma$ to be large can make our benchmarks \textit{too easy}: for example, consider Table \ref{tab:highgamma}, which shows a case where setting $\gamma = 0.05$ reduces the benchmark for the follower, but the instance has rewards that are sufficiently far apart that for large $T$ the Stackelberg equilibrium should intuitively be learnable. 

\begin{table}[]
\centering 
\begin{tabular}{|c|c|c|}
\hline
      & $b_1$       & $b_2$       \\ \hline
$a_1$ & (0.6, 0.05) & (0.2, 0.1)  \\ \hline
$a_2$ & (0.5, 0.2)  & (0.4, 0.15) \\ \hline
\end{tabular}
\caption{Taking $\gamma$ to be too small makes the benchmark too easy: for $\gamma=0$, we have $\benchmarkrelaxedstronger{1} = 0.5, \benchmarkrelaxedstronger{2} = 0.2$, but for $\gamma = 0.05$ we have $\benchmarkrelaxedstronger{1}= 0.5$ and $\benchmarkrelaxedstronger{2} = 0.15$ (see Section \ref{subsec:maxtol})}
\label{tab:highgamma}
\end{table}

We briefly discuss how our results extend to different maximum tolerances $\gamma$. First, we prove our lower bounds (Theorem \ref{thrm:dlower}, Proposition \ref{prop:lowerboundsqrtt}) for the ``hardest case'' of $\gamma = 1$, which means that these lower bounds hold for \emph{all} maximum tolerances $\gamma$. 

On the other hand, our upper bounds require sufficiently large $\gamma$. For some intuition, all of our analyses require that $\gamma = \omega(1/\sqrt{T})$, since followers with high-probability instantaneous regret rates of $\Theta(\sqrt{\abs{\mathcal{B}} \cd \log(T)/t})$ require $\Omega(T)$ rounds to find a $O(1/\sqrt{T})$-optimal solution. As to what specific values of $\gamma$ that each result requires, Theorems \ref{thm:ETCETC} and \ref{thm:explorethenucb} hold for any $\gamma = \omega\p{T^{-1/3}\abs{\mathcal{A}}^{1/3} \abs{\mathcal{B}}^{1/3} \cd (\log (T)^{1/3})}$, while Theorem \ref{thm:UCBweaker} assumes that $\gamma = \Omega\left( T^{-1/4} \sqrt{|\mathcal{A}|  |\mathcal{B}| \cdot \log T} \right)$.

\subsection{$\epsilon$-Regularizer}

Since the $\epsilon$-regularizer adds an implicit penalty for increasing $\epsilon$ in the benchmark, a natural question is how our benchmark would change if we changed the regularizer from $\epsilon$ to other functional forms $f(\epsilon)$. To provide some preliminary intuition for this, we consider $f(\epsilon) = c \cd \epsilon^d$ regularizer, which leads to the following generalized $\gamma$-tolerant benchmarks.
\begin{definition}[Generalization of Definition \ref{def:relaxedbenchmarks}]
\label{def:relaxedbenchmarksgeneralized}
Given a maximum tolerance $\gamma > 0$ and parameters $c > 0$, and $d > 0$, we define the \textit{generalized $(c, d, \gamma)$-tolerant benchmarks} $\benchmarkrelaxedstronger{1}$ and $\benchmarkrelaxedstronger{2}$ to be: 
\[\benchmarkrelaxedstronger{1} = \inf_{\epsilon \le \gamma} \Big(\underbrace{\max_{a \in \mathcal{A}} \min_{b \in \mathcal{B}_{\epsilon}(a)} \MR{1}{a}{b}}_{\text{$\epsilon$-relaxed Stackelberg utility}} + \underbrace{c \cd \epsilon^d}_{\text{$\epsilon$-regularizer}} \Big)\]
\[\benchmarkrelaxedstronger{2} = \inf_{\epsilon \le \gamma} \Big(\underbrace{\min_{a \in \mathcal{A}_{\epsilon}} \max_{b \in \mathcal{B}} \MR{2}{a}{b}}_{\text{$\epsilon$-relaxed Stackelberg utility}} + \underbrace{c \cd \epsilon^d}_{\text{$\epsilon$-regularizer}} \Big).  \]   
\end{definition}
At a conceptual level, different settings of $c$ and $d$ capture different levels of tolerance that a player has for sub-optimality in the other player. Higher values of $c$ and smaller values of $d$ capture greater intolerance, and thus lead to harsher penalties. The resulting changes in the benchmarks capture that if a player is less tolerant, we might expect them to experience a higher regret for a given suboptimality level of the other player. 

We show how our two main upper bounds in Section \ref{sec:relaxed} generalize to these new benchmarks, focusing on the case of $c \ge 1$ and $d \le 1$ (where the benchmark becomes harder). We first show the following generalization of Theorem \ref{thm:ETCETC}  by adjusting the explore phase length to depend on $c$ and $d$. 
\begin{restatable}{theorem}{etcetcregularizer}
\label{thm:ETCETCregularizer} 
Suppose that $c \ge 1$ and $d \le 1$, and let $\eta := 2/(2+d)$. Consider a $\StrongSetting$, where the follower runs a separate instantiation of $\ExploreThenCommit(E_2, \mathcal{B})$ for every $a \in \mathcal{A}$, and the leader runs $\ExploreThenCommitThrowOut(E_1, E_2 \cdot |\mathcal{B}|, \mathcal{A})$. If $E_2 = \Theta(|\mathcal{A}|^{-\eta} |\mathcal{B}|^{-\eta} \cdot (\log T)^{1 - \eta} (c\cd T)^{\eta})$, and $E_1 = \Theta(|\mathcal{A}|^{-\eta} \cdot (\log T)^{1 - \eta} (c\cd T)^{\eta})$, then the leader and follower regret with respect to the generalized $(c, d, \gamma)$-tolerant benchmarks are both at most:
\[\max(R_1(T), R_2(T)) = O\big( (|\mathcal{A}|\cd |\mathcal{B}| \cd (\log T))^{1-\eta} \cd (c \cdot T)^{\eta} \big). \]
\end{restatable}
We next show the following generalizations of Theorem \ref{thm:explorethenucb}  by again adjusting the explore phase length to depend on $c$ and $d$. Like in Theorem \ref{thm:explorethenucb}, the assumptions on the follower's algorithm in this result are satisfied by standard algorithms such as $\ActiveArmElimination$ (Algorithm \ref{algo:AAE}; Proposition \ref{prop:AAE}) and $\ExploreThenCommit$ (Algorithm \ref{algo:explorethencommit}; Proposition \ref{prop:ETC}). 
\begin{restatable}{theorem}{explorethenucbregularizer}
\label{thm:explorethenucbregularizer}
Suppose that $c \ge 1$ and $d \le 1$, and let $\eta := 2/(2+d)$. Let $E = \Theta( |\mathcal{A}|^{-\eta} (|\mathcal{B}| \log T)^{1 - \eta} (c \cd T)^{\eta})$. Consider a $\StrongSetting$ where $\ALG_{2}$ is any algorithm with high-probability instantaneous regret \\
$g(t, T, \mathcal{B}) = O\left( (|\mathcal{A}| \cd |\mathcal{B}| \cd \log T)^{\eta/2} \cdot (c \cdot T)^{-\eta/2} \right)$ for $t > E$ and $g(t, T, \mathcal{B}) = 1$ for $t \le E$, and where $\ALG_1 = \ExploreThenUCB(E)$. Then, then the leader and follower regret with respect to the generalized $(c, d, \gamma)$-tolerant benchmarks are both bounded as: 
\[\max(R_1(T), R_2(T)) = O((|\mathcal{A}| \cdot |\mathcal{B}| \cdot (\log T))^{1 - \eta} \cdot (c \cdot T)^{\eta}).\]
\end{restatable}

The proofs of Theorem \ref{thm:ETCETCregularizer} and Theorem \ref{thm:explorethenucbregularizer} follows from the same arguments as the proof of Theorem \ref{thm:ETCETC}  and Theorem \ref{thm:explorethenucb}, respectively, but with the values of $E_1, E_2$ modified (full proofs are deferred to Appendix \ref{appendix:proofsregularizers}). Note that as $d$ decreases, the regret bound worsens: this aligns with the intuition that smaller values of $d$ capture greater intolerance. Similarly, the regret increases with $c$. 

We defer a more extensive treatment of these generalized benchmarks to future work. Moreover, another interesting for future work would be to extend our model and results to more general functions $f(\epsilon)$ and also allow the two players to have different regularizers.

\section{Discussion of Assumptions on the Follower's Algorithm}\label{sec:algorithmsassumptions}

Our algorithms for the leader placed assumptions on the fine-grained performance of the follower's algorithm. More specifically, the regret bound for $\ExploreThenUCB$ required an \textit{high-probability instantaneous regret bound} $g$ for the follower (Theorem \ref{thm:explorethenucb}), and the regret bound for $\LipschitzUCB$ required an \textit{high-probability anytime regret bound} $h$ for the follower (Theorem \ref{thrm:dpL}). 

In this section, we examine these two conditions in more detail. First, we relate these two conditions and show that many standard algorithms satisfy the conditions on $g$ and $h$ in Theorem \ref{thm:explorethenucb} and Theorem \ref{thrm:dpL} (Section \ref{subsec:algosassumptions}). Then, we extend our analysis of $\ExploreThenUCB$ and $\LipschitzUCB$ to more general conditions on $g$ and $h$, respectively (Section \ref{subsec:generalizationfollower}).

\subsection{Algorithms Satisfying These Fine-Grained Regret Guarantees}\label{subsec:algosassumptions}

As a warmup, we first observe that high probability instantaneous regret bounds immediately translate to high-probability anytime regret bounds. 
\begin{observation}
\label{observation:conversion}
Suppose that $\ALG_2$ satisfies a high-probability instantaneous regret bound of $g$. Then it holds that $\ALG_2$ satisfies an anytime regret bound of $h$ defined as $h(t,T) := \sum_{t'=1}^t g(t', T)$. 
\end{observation}
\noindent As a consequence, if $\ALG_2$ satisfies the high-probability instantaneous regret bound in Theorem \ref{thm:explorethenucb} (i.e., $g(t, T, \mathcal{B}) = O\left( (|\mathcal{A}| |\mathcal{B}| \log T)^{1/3} T^{-1/3} \right)$ for $t > E := \Theta( |\mathcal{A}|^{-2/3} (|\mathcal{B}| \log T)^{1/3} T^{2/3})$ and $g(t, T, \mathcal{B}) = 1$ for $t \le E$), then $\ALG_2$ also satisfies an anytime regret bound of $h$ defined for $t > E$ as: 
\[h(t,T) := \sum_{t'=1}^t g(t', T) =  E  + \sum_{t' = E+1}^t O\left( (|\mathcal{A}| |\mathcal{B}| \log T)^{1/3} T^{-1/3} \right) =  O\left( (|\mathcal{A}| |\mathcal{B}| \log T)^{1/3} T^{2/3} \right). \]
However, this naive high-probability anytime regret bound is not strong enough for Theorem \ref{thrm:dpL}. We can nonetheless achieve the desired regret bound with additional assumptions on $\ALG_2$ as we describe below.

\begin{algorithm2e}[htbp]
\caption{$\ActiveArmElimination(M_1, \ldots, M_P)$ applied to $(a, \Hist)$ (adapted from \citep{even2002pac, L20})} 
\label{algo:AAE}
\DontPrintSemicolon
\LinesNumbered
Initialize $s' = 0$, $t' = 1$, $\mathcal{B}'= \mathcal{B}$ \tcp*{Index of the last completed phase, time step marking beginning of phase $s'+1$, active arms in phase $s'$.}
Initialize $newphase = \text{False}$. \tcp*{Boolean for first time step in phase.}
Let $t = |\Hist|$. \\
\For{$t'' = 1$ \KwTo $t$}{
      \For{$b \in \mathcal{B}'$}{
            Let $n(a,b) := \left|\left\{(t'', a_{t''}, b_{t''}, r) \in \Hist \mid a_{t''} = a, b_{t''} = b, t'' \ge t' \right\}\right|$. \\
            
        }
        \If {$n(a,b) = M_{s' + 1} \forall b \in \mathcal{B}'$}{
                $newphase = \text{True}$.
        }
    \If{$newphase = \text{True}$}{
        \For{$b \in \mathcal{B}'$}{
             Set $S(a, b) := \left\{r \mid \exists (t'',  a_{t''}, b_{t''}, r) \in \Hist \text{ s.t. } a_{t''} = a, b = b_{t''}, t'' \geq t'  \right\}$ \tcp*{observed rewards}
        $\hatMR{2}{a}{b} \leftarrow (\sum_{r \in S(a, b)} r)/|S(a, b)|$ \tcp*{compute empirical mean}
             }
        Update $\mathcal{B}' \leftarrow \{b \mid \hatMR{2}{a}{b} + \frac{ 20 \cdot \sqrt{\log T}}{\sqrt{M_{s'}}}
        \geq \max_{b \in \mathcal{B}'}\hatMR{2}{a}{b}\}$.\\
            Update $s' \leftarrow s' + 1$. \\
            Update $t' \leftarrow t$. \\
            $newphase = \text{False}$.
        }
}
    $i$  = $((t-t') \mod (\abs{\mathcal{B}'})) + 1$.\tcp*{Calculate next arm to be pulled} 
\Return{point mass at $b_i$}.
\end{algorithm2e}

By leveraging the structural properties of specific algorithms, we show that many standard algorithms achieve high-probability instantaneous regret $g$ and/or high-probability anytime regret $h$, where $g$ and $h$ are specified according to the functional forms in Theorem \ref{thm:explorethenucb} and Theorem \ref{thrm:dpL}. Proofs of these results are deferred to \Cref{appendix:proofsalgosassumptions}.

First, we show that $\ActiveArmElimination$ \citep{even2002pac} (Algorithm \ref{algo:AAE}, \\
see \citet{L20} for a textbook treatment) satisfies both the high-probability instantaneous regret bound required for Theorem \ref{thm:explorethenucb} and the high-probability anytime regret bound required for Theorem \ref{thrm:dpL}. 
\begin{restatable}{proposition}{AAE}\label{prop:AAE}
Suppose that for every $a \in \mathcal{A}$, the follower runs a separate instantiation of $\texttt{ActiveArmElimination}(M_1, \ldots, M_P)$ (Algorithm \ref{algo:AAE}) with $M_i = \Theta(\log T \cdot 2^{2i})$. Then the follower satisfies high-probability instantaneous regret $g(t, T, \mathcal{B}) = O(\sqrt{\abs{\mathcal{B}} \cdot \log(T) /t }$, which implies $g(t, T, \mathcal{B}) = O\left( (|\mathcal{A}| |\mathcal{B}| \log T)^{1/3} T^{-1/3} \right)$ for $t \ge \Theta( |\mathcal{A}|^{-2/3} (|\mathcal{B}| \log T)^{1/3} T^{2/3})$. Moreover, the follower satisfies high-probability anytime regret $h(t, T, \mathcal{B}) = O(\sqrt{\abs{\mathcal{B}} \cd \log(T) \cd t})$. 
\end{restatable}

Next, we show that $\ExploreThenCommit$  (Algorithm \ref{algo:explorethencommit}, see \citet{S19, L20} for a textbook treatment) satisfies the high-probability instantaneous regret bound required for Theorem \ref{thm:explorethenucb}. 
\begin{restatable}{proposition}{ETC}
\label{prop:ETC}
Suppose that the follower runs a separate instantiation of $\ExploreThenCommit(E, \mathcal{B})$ (Algorithm \ref{algo:explorethencommit}) for every $a \in \mathcal{A}$. Then, the follower satisfies high-probability instantaneous regret $g(t, T, \mathcal{B}) = \mathcal{O}(\sqrt{\log T/E})$ for all time steps $t\geq E \cdot |\mathcal{B}|$. If $E = \Theta( (|\mathcal{A} \cdot |\mathcal{B}|)^{-2/3} (\log T)^{1/3} T^{2/3})$, then $g(t, T, \mathcal{B}) = O\left( (|\mathcal{A}| |\mathcal{B}| \log T)^{1/3} T^{-1/3} \right)$ for $t \ge \Theta( |\mathcal{A}|^{-2/3} (|\mathcal{B}| \log T)^{1/3} T^{2/3})$.
\end{restatable}
\noindent Note that $\ExploreThenCommit$ does \textit{not} satisfy the high-probability anytime regret bound required for Theorem \ref{thrm:dpL} due to the uniform exploration phase at the beginning of the algorithm.

Finally, we show that UCB \citep{ACF02} (see \citet{S19, L20} for a textbook treatment) satisfies the high-probability anytime regret bound required in Theorem \ref{thrm:dpL}. 
\begin{restatable}{proposition}{UCB}
\label{prop:UCB}
Suppose that the follower runs a separate instantiation of UCB for every $a \in \mathcal{A}$. Then, the follower satisfies high-probability anytime regret bound $h(t, T, \mathcal{B}) = O(\sqrt{\abs{\mathcal{B}} \cd t \cd \log(T)})$. 
\end{restatable}
\noindent We do not expect that UCB satisfies the high-probability instantaneous regret bound required for Theorem \ref{thm:explorethenucb}, using the intuition that UCB does not provide final-iterate convergence guarantees.

\subsection{Generalized Analysis of $\ExploreThenUCB$ and $\LipschitzUCB$}\label{subsec:generalizationfollower}

While the specific instantations of $g$ and $h$ in Theorem \ref{thm:explorethenucb} and Theorem \ref{thrm:dpL} are tailored to standard algorithms (Section \ref{subsec:algosassumptions}), we generalize our analysis of $\ExploreThenUCB$ and $\LipschitzUCB$ to a richer class of functions $g$ and $h$, respectively.

We generalize Theorem \ref{thm:explorethenucb} to functions $g(t, T) = O(E^{-c_1}|\mathcal{B}|^{c_2} (\log T)^{c_3})$ for $t> E$, where $c_1 \in (0,1)$ and $c_2, c_3 > 0$ are arbitrary parameters and where $E = \Theta(|\mathcal{A}|^{-1/(1+c_1)} |\mathcal{B}|^{c_2/(1+c_1)} \log(T)^{c_2/(1+c_1)} \cdot T^{1/(1+c_1)})$.
\begin{restatable}{theorem}{explorethenucbgeneralized}
\label{thm:explorethenucbgeneralizedg}
Let $c_1 \in (0,1)$ and $c_2, c_3 > 0$. Let $E = \Theta(|\mathcal{A}|^{-1/(1+c_1)} |\mathcal{B}|^{c_2/(1+c_1)} (\log T)^{c_3/(1+c_1)} \cdot T^{1/(1+c_1)})$.  Consider a $\StrongSetting$ where
 $\ALG_{2}$ is any algorithm with high-probability instantaneous regret $g(t, T, \mathcal{B}) = O\left(E^{-c_1}|\mathcal{B}|^{c_2} (\log T)^{c_3} \right)$ for $t > E$ and $g(t, T, \mathcal{B}) = 1$ for $t \le E$, and where $\ALG_1 = \ExploreThenUCB(E)$. Then, it holds that the regret with respect to the $\gamma$-tolerant benchmarks $\benchmarkrelaxedstronger{1}$ and $\benchmarkrelaxedstronger{2}$ is bounded as: 
\[\max(R_1(T), R_2(T)) = O\left(T^{1/(1+c_1)} \cdot |\mathcal{A}|^{c_1/(1+c_1)} \cdot |\mathcal{B}|^{c_2/(1+c_1)} \cdot (\log T)^{c_3/(1+c_1)} \right) + \Theta\left(\sqrt{T |\mathcal{A}| \log T}\right).\]
\end{restatable}
\noindent Note that the special case of $c_1 = c_2 = c_3 = 1/2$ recovers the functional form of $g$ in Theorem \ref{thm:explorethenucb}. The proof, which follows similarly to the proof of Theorem \ref{thm:explorethenucb}, is deferred to Appendix \ref{appendix:proofsalgosgeneralized}. 

We similarly generalize Theorem \ref{thrm:dpL} to functions $h(t, T) = C' \cdot t^{c_1} \cdot |\mathcal{B}|^{c_2} \cdot (\log(T))^{c_3}$ for $t> E$, where $c_1, c_2, c_3 \in (0,1)$ are arbitrary parameters. This result requires the leader to instead run $\LipschitzUCBGeneralized$ (Algorithm \ref{algo:lipschitzucbgeneralized}), a generalized version of $\LipschitzUCB$ which adjusts the confidence set size based on the parameters $c_1$, $c_2$, and $c_3$. 
\begin{restatable}{theorem}{dpLgeneralized}\label{thrm:dpLgeneralized}
Let $c_1 \in (0,1)$, $c_2, c_3 > 0$, and $C' > 0$ be arbitrary constants.  Consider a $\StrongSetting$ where
$\Instance = (\mathcal{A}, \mathcal{B}, v_1, v_2)$ has Lipschitz constant $L^*$. Let $\ALG_2$ be any algorithm satisfying high-probability anytime regret $h(t, T, \mathcal{B}) = C' \cdot t^{c_1} \cdot |\mathcal{B}|^{c_2} \cdot (\log(T))^{c_3}$. Let $\ALG_1 = \LipschitzUCBGeneralized(L, C' B^{c_2}, c_1, c_3)$ for any $L \ge L^*$. Then both players achieve the following regret bounds with respect to the original Stackelberg benchmarks $\benchmark{1}$ and $\benchmark{2}$: that is, $R_1(T; \Instance) = O\left(\sqrt{T |\mathcal{A}| |\mathcal{B}| \log T} + L |\mathcal{A}|^{1-c_1} |\mathcal{B}|^{c_2} (\log T)^{c_3} T^{c_1}\right)$ and $R_2(T; \Instance) = O\left(L \sqrt{T |\mathcal{A}| |\mathcal{B}| \log T} + L^2 |\mathcal{A}|^{1-c_1} |\mathcal{B}|^{c_2} T^{c_1} (\log T)^{c_3} \right)$. 
\end{restatable}
\noindent Again, note that the special case of $c_1 = c_2 = c_3 = 1/2$ recovers the functional form of $g$ in Theorem \ref{thrm:dpL}. The proof, which follows similarly to the proof of Theorem \ref{thrm:dpL}, is deferred to Appendix \ref{appendix:proofsalgosgeneralized}.

\begin{algorithm2e}[htbp]
\caption{$\LipschitzUCBGeneralized(L, C, c_1, c_3)$ applied to $\Hist$} 
\label{algo:lipschitzucbgeneralized}
\DontPrintSemicolon
\LinesNumbered
Initialize $\hatMROneArg{1}{a} = 1$ for $a \in \mathcal{A}$. \tcp*{Initialize empirical means for $\max_{b \in \mathcal{B}} \MR{1}{a}{b}$.} 
Initialize $\UCBOneArg{1}{a} = 1$ for $a \in \mathcal{A}$. \tcp*{Initialize UCB for $\max_{b \in \mathcal{B}} \MR{1}{a}{b}$}
\For{$a \in \mathcal{A}$}{
Set $S(a) := \left\{r \mid \exists (t', a_{t'}, r) \in \Hist \text{ s.t. } a = a_{t'} \right\}$ \tcp*{Observed rewards}
\If{$S(a) \neq \emptyset$} {
$\hatMROneArg{1}{a} \leftarrow (\sum_{r \in S(a)} r)/|S(a)|$ \tcp*{Empirical mean}
$\alpha(a) \leftarrow \frac{10 \sqrt{\mathcal{B} \log T}}{\sqrt{|S(a)|}} + C \cdot L \cdot (\log T)^{c_3} T^{c_1-1}$ \tcp*{confidence bound width}
$\UCBOneArg{1}{a} \leftarrow \min(1, \hatMROneArg{1}{a} + \alpha(a))$ 
}
}
Let $a^* = \argmax_{a \in \mathcal{A}} \left(\UCBOneArg{1}{a} \right)$. \tcp*{arm with max upper confidence bound} 
\Return{point mass at $a^*$}.
\end{algorithm2e}

\section{Discussion}
In this paper, we studied two-agent environments where interactions are \textit{sequential}, utilities are \textit{misaligned}, and each agent \textit{learns} their utilities over time.
We modeled these environments as decentralized Stackelberg games where both agents are bandit learners who only observe their own utilities, and we investigated the implications for each agent's cumulative utility over time. Motivated by the offline Stackelberg equilibrium benchmarks being infeasible (Theorem \ref{thrm:worstlinear}), we designed $\gamma$-tolerant benchmarks which allow for approximate best responses by the other agent.

We proved that both players can achieve $\tilde{\Theta}(T^{2/3})$ regret with respect to the $\gamma$-tolerant benchmarks. To achieve this regret bound, we designed an algorithm (i.e., $\ExploreThenUCB$; Algorithm \ref{algo:explorethenucb}) where the leader waits for the follower to partially converge before starting to learn; this algorithm achieves $\tilde{\Theta}(T^{2/3})$ regret for both players  under a rich class of follower learning algorithms (Theorem \ref{thm:explorethenucb}). We further show that $\tilde{\Theta}(T^{2/3})$ regret is unavoidable for any pair of algorithms (Theorem \ref{thrm:dlower}).
Furthermore, we showed that $O(\sqrt{T})$ regret is possible in two relaxed environments: i.e., under a relaxed benchmark that is (self-)tolerant of a player's own mistakes (Theorem \ref{thm:UCBweaker}) or when players agree on which pair of actions are different (Theorem \ref{thrm:dpL})

Our results have broader implications for \textit{designing} two-agent environments to achieve favorable utility for both agents. For example, given that our results illustrate that certain properties for the follower (such as high-probability instantaneous regret or high-probability anytime regret bounds) and certain properties for the leader (such as waiting for the follower to partially converge) are conducive to low regret, it may be helpful for a designer to engineer or encourage agents to follow these algorithmic principles. As another example, our continuity results in Section \ref{subsec:lipschitz} illustrate the importance of reducing \textit{near-ties} in utilities between different items, which could be achieved by allowing agents to express preferences between items in a nuanced fashion. 

More broadly, our benchmarks and regret analysis suggest several interesting avenues for future work. For example, while Theorem \ref{thm:explorethenucb} offered flexibility in the follower's choice of algorithm, we required that the leader follow a particular algorithm: it would be interesting to explore richer classes of leader algorithms which maintain low regret. Additionally, while our framework captures a range of real-world applications including chatbots (Example \ref{example:llm}  in \Cref{subsec:examples}) and recommender systems (Example \ref{example:recsys}  in \Cref{subsec:examples}), an interesting future direction would be to focus on a particular application and incorporate application-specific nuances (e.g., bidder learning rates in advertising auctions \citep{nekipelov2015econometrics, 10.1145/3442381.3449968, 10.1145/3038912.3052621}). Finally, while we study the role of continuity requirements that reflect alignment (Section \ref{subsec:lipschitz}), it would be interesting to consider other structured bandit environments such as linear utility functions and generalize our benchmarks and results accordingly. 

 \section*{Acknowledgements}

We thank Keegan Harris, Jason Hartline, Nika Haghtalab, Nick Wu, and Kunhe Yang for valuable comments and feedback. KD was partially supported by a Vannevar Bush Faculty Fellowship, a Simons Collaboration grant, and a grant from the MacArthur Foundation. MJ was partially supported by an Open Philanthropy AI Fellowship. 
\newpage 

\bibliographystyle{plainnat}

\bibliography{ref.bib}

\begin{thebibliography}{76}
\providecommand{\natexlab}[1]{#1}
\providecommand{\url}[1]{\texttt{#1}}
\expandafter\ifx\csname urlstyle\endcsname\relax
  \providecommand{\doi}[1]{doi: #1}\else
  \providecommand{\doi}{doi: \begingroup \urlstyle{rm}\Url}\fi

\bibitem[Agarwal et~al.(2017)Agarwal, Luo, Neyshabur, and Schapire]{ALNS17}
Alekh Agarwal, Haipeng Luo, Behnam Neyshabur, and Robert~E. Schapire.
\newblock Corralling a band of bandit algorithms.
\newblock In \emph{Proceedings of the 30th Conference on Learning Theory,
  {COLT} 2017, Amsterdam, The Netherlands, 7-10 July 2017}, volume~65 of
  \emph{Proceedings of Machine Learning Research}, pages 12--38. {PMLR}, 2017.

\bibitem[Agarwal and Brown(2023)]{agarwal2024online}
Arpit Agarwal and William Brown.
\newblock Online recommendations for agents with discounted adaptive
  preferences.
\newblock \emph{CoRR}, abs/2302.06014, 2023.

\bibitem[Ahmadi et~al.(2021)Ahmadi, Beyhaghi, Blum, and Naggita]{ABBN21}
Saba Ahmadi, Hedyeh Beyhaghi, Avrim Blum, and Keziah Naggita.
\newblock The strategic perceptron.
\newblock In \emph{{EC} '21: The 22nd {ACM} Conference on Economics and
  Computation, Budapest, Hungary, July 18-23, 2021}, pages 6--25. {ACM}, 2021.

\bibitem[Anagnostides et~al.(2022)Anagnostides, Daskalakis, Farina, Fishelson,
  Golowich, and Sandholm]{ADFF22}
Ioannis Anagnostides, Constantinos Daskalakis, Gabriele Farina, Maxwell
  Fishelson, Noah Golowich, and Tuomas Sandholm.
\newblock Near-optimal no-regret learning for correlated equilibria in
  multi-player general-sum games.
\newblock In \emph{{STOC} '22: 54th Annual {ACM} {SIGACT} Symposium on Theory
  of Computing, Rome, Italy, June 20 - 24, 2022}, pages 736--749. {ACM}, 2022.

\bibitem[Aridor et~al.(2020)Aridor, Mansour, Slivkins, and Wu]{AMSW20}
Guy Aridor, Yishay Mansour, Aleksandrs Slivkins, and Zhiwei~Steven Wu.
\newblock Competing bandits: The perils of exploration under competition.
\newblock \emph{CoRR}, abs/2007.10144, 2020.

\bibitem[Auer et~al.(2002)Auer, Cesa{-}Bianchi, and Fischer]{ACF02}
Peter Auer, Nicol{\`{o}} Cesa{-}Bianchi, and Paul Fischer.
\newblock Finite-time analysis of the multiarmed bandit problem.
\newblock \emph{Mach. Learn.}, 47\penalty0 (2-3):\penalty0 235--256, 2002.

\bibitem[Bai et~al.(2021)Bai, Jin, Wang, and Xiong]{bai2021sample}
Yu~Bai, Chi Jin, Huan Wang, and Caiming Xiong.
\newblock Sample-efficient learning of stackelberg equilibria in general-sum
  games.
\newblock \emph{Advances in Neural Information Processing Systems},
  34:\penalty0 25799--25811, 2021.

\bibitem[Bakker et~al.(2022)Bakker, Chadwick, Sheahan, Tessler,
  Campbell{-}Gillingham, Balaguer, McAleese, Glaese, Aslanides, Botvinick, and
  Summerfield]{BCSTCBMGAB22}
Michiel~A. Bakker, Martin~J. Chadwick, Hannah Sheahan, Michael~Henry Tessler,
  Lucy Campbell{-}Gillingham, Jan Balaguer, Nat McAleese, Amelia Glaese, John
  Aslanides, Matt~M. Botvinick, and Christopher Summerfield.
\newblock Fine-tuning language models to find agreement among humans with
  diverse preferences.
\newblock In \emph{Advances in Neural Information Processing Systems 35: Annual
  Conference on Neural Information Processing Systems 2022, NeurIPS 2022, New
  Orleans, LA, USA, November 28 - December 9, 2022}, 2022.

\bibitem[Balcan et~al.(2015)Balcan, Blum, Haghtalab, and
  Procaccia]{balcan2015commitment}
Maria-Florina Balcan, Avrim Blum, Nika Haghtalab, and Ariel~D Procaccia.
\newblock Commitment without regrets: Online learning in stackelberg security
  games.
\newblock In \emph{Proceedings of the sixteenth ACM conference on economics and
  computation}, pages 61--78, 2015.

\bibitem[Balseiro and Gur(2019)]{BG19}
Santiago~R. Balseiro and Yonatan Gur.
\newblock Learning in repeated auctions with budgets: Regret minimization and
  equilibrium.
\newblock \emph{Manag. Sci.}, 65\penalty0 (9):\penalty0 3952--3968, 2019.

\bibitem[Bansal et~al.(2021)Bansal, Wu, Zhou, Fok, Nushi, Kamar, Ribeiro, and
  Weld]{BWZFNKRW21}
Gagan Bansal, Tongshuang Wu, Joyce Zhou, Raymond Fok, Besmira Nushi, Ece Kamar,
  Marco~T{\'{u}}lio Ribeiro, and Daniel~S. Weld.
\newblock Does the whole exceed its parts? the effect of {AI} explanations on
  complementary team performance.
\newblock In \emph{{CHI} '21: {CHI} Conference on Human Factors in Computing
  Systems, Virtual Event / Yokohama, Japan, May 8-13, 2021}, pages 81:1--81:16.
  {ACM}, 2021.

\bibitem[Borgs et~al.(2007)Borgs, Chayes, Immorlica, Jain, Etesami, and
  Mahdian]{BCIJEM07}
Christian Borgs, Jennifer~T. Chayes, Nicole Immorlica, Kamal Jain, Omid
  Etesami, and Mohammad Mahdian.
\newblock Dynamics of bid optimization in online advertisement auctions.
\newblock In \emph{Proceedings of the 16th International Conference on World
  Wide Web, {WWW} 2007, Banff, Alberta, Canada, May 8-12, 2007}, pages
  531--540. {ACM}, 2007.

\bibitem[Br{\^a}nzei et~al.(2024)Br{\^a}nzei, Hajiaghayi, Phillips, Shin, and
  Wang]{branzei2024dueling}
Simina Br{\^a}nzei, MohammadTaghi Hajiaghayi, Reed Phillips, Suho Shin, and Kun
  Wang.
\newblock Dueling over dessert, mastering the art of repeated cake cutting.
\newblock \emph{arXiv preprint arXiv:2402.08547}, 2024.

\bibitem[Braverman et~al.(2018)Braverman, Mao, Schneider, and Weinberg]{BMSW18}
Mark Braverman, Jieming Mao, Jon Schneider, and S.~Matthew Weinberg.
\newblock Selling to a no-regret buyer.
\newblock In \emph{Proceedings of the 2018 {ACM} Conference on Economics and
  Computation, Ithaca, NY, USA, June 18-22, 2018}, pages 523--538. {ACM}, 2018.

\bibitem[Brown et~al.(2023)Brown, Schneider, and Vodrahalli]{brown2024learning}
William Brown, Jon Schneider, and Kiran Vodrahalli.
\newblock Is learning in games good for the learners?
\newblock \emph{Advances in Neural Information Processing Systems}, 36, 2023.

\bibitem[Camara et~al.(2020)Camara, Hartline, and
  Johnsen]{camara2020mechanisms}
Modibo~K Camara, Jason~D Hartline, and Aleck Johnsen.
\newblock Mechanisms for a no-regret agent: Beyond the common prior.
\newblock In \emph{2020 ieee 61st annual symposium on foundations of computer
  science (focs)}, pages 259--270. IEEE, 2020.

\bibitem[Chan et~al.(2019)Chan, Hadfield-Menell, Srinivasa, and
  Dragan]{chan2019assistive}
Lawrence Chan, Dylan Hadfield-Menell, Siddhartha Srinivasa, and Anca Dragan.
\newblock The assistive multi-armed bandit.
\newblock In \emph{2019 14th ACM/IEEE International Conference on Human-Robot
  Interaction (HRI)}, pages 354--363. IEEE, 2019.

\bibitem[Chen et~al.(2023)Chen, Zhang, Langren{\'{e}}, and Zhu]{CZLZ23}
Banghao Chen, Zhaofeng Zhang, Nicolas Langren{\'{e}}, and Shengxin Zhu.
\newblock Unleashing the potential of prompt engineering in large language
  models: a comprehensive review.
\newblock \emph{CoRR}, abs/2310.14735, 2023.

\bibitem[Chen et~al.(2020)Chen, Liu, and Podimata]{chen2019}
Yiling Chen, Yang Liu, and Chara Podimata.
\newblock Learning strategy-aware linear classifiers.
\newblock In \emph{Advances in Neural Information Processing Systems 33: Annual
  Conference on Neural Information Processing Systems 2020, NeurIPS 2020,
  December 6-12, 2020, virtual}, 2020.

\bibitem[Collina et~al.(2023{\natexlab{a}})Collina, Arunachaleswaran, and
  Kearns]{CAK23}
Natalie Collina, Eshwar~Ram Arunachaleswaran, and Michael Kearns.
\newblock Efficient stackelberg strategies for finitely repeated games.
\newblock In \emph{Proceedings of the 2023 International Conference on
  Autonomous Agents and Multiagent Systems, {AAMAS} 2023, London, United
  Kingdom, 29 May 2023 - 2 June 2023}, pages 643--651. {ACM},
  2023{\natexlab{a}}.

\bibitem[Collina et~al.(2023{\natexlab{b}})Collina, Roth, and
  Shao]{collina2023efficient}
Natalie Collina, Aaron Roth, and Han Shao.
\newblock Efficient prior-free mechanisms for no-regret agents.
\newblock \emph{arXiv preprint arXiv:2311.07754}, 2023{\natexlab{b}}.

\bibitem[Corvelo~Benz and Rodriguez(2024)]{corvelo2024human}
Nina Corvelo~Benz and Manuel Rodriguez.
\newblock Human-aligned calibration for ai-assisted decision making.
\newblock \emph{Advances in Neural Information Processing Systems}, 36, 2024.

\bibitem[Daskalakis et~al.(2011)Daskalakis, Deckelbaum, and Kim]{DDK11}
Constantinos Daskalakis, Alan Deckelbaum, and Anthony Kim.
\newblock Near-optimal no-regret algorithms for zero-sum games.
\newblock In Dana Randall, editor, \emph{Proceedings of the Twenty-Second
  Annual {ACM-SIAM} Symposium on Discrete Algorithms, {SODA} 2011, San
  Francisco, California, USA, January 23-25, 2011}, pages 235--254. {SIAM},
  2011.

\bibitem[Daskalakis et~al.(2021)Daskalakis, Fishelson, and Golowich]{DFG21}
Constantinos Daskalakis, Maxwell Fishelson, and Noah Golowich.
\newblock Near-optimal no-regret learning in general games.
\newblock In \emph{Advances in Neural Information Processing Systems 34: Annual
  Conference on Neural Information Processing Systems 2021, NeurIPS 2021,
  December 6-14, 2021, virtual}, pages 27604--27616, 2021.

\bibitem[Deng et~al.(2019)Deng, Schneider, and Sivan]{deng2019strategizing}
Yuan Deng, Jon Schneider, and Balasubramanian Sivan.
\newblock Strategizing against no-regret learners.
\newblock \emph{Advances in neural information processing systems}, 32, 2019.

\bibitem[Donahue et~al.(2024)Donahue, Kollias, and Gollapudi]{donahue2023two}
Kate Donahue, Kostas Kollias, and Sreenivas Gollapudi.
\newblock When are two lists better than one?: Benefits and harms in joint
  decision-making.
\newblock \emph{AAAI '24'}, 2024.

\bibitem[Dong et~al.(2018)Dong, Roth, Schutzman, Waggoner, and Wu]{dong18}
Jinshuo Dong, Aaron Roth, Zachary Schutzman, Bo~Waggoner, and Zhiwei~Steven Wu.
\newblock Strategic classification from revealed preferences.
\newblock In \emph{Proceedings of the 2018 {ACM} Conference on Economics and
  Computation, Ithaca, NY, USA, June 18-22, 2018}, pages 55--70. {ACM}, 2018.

\bibitem[Ekstrand and Willemsen(2016)]{EW16}
Michael~D. Ekstrand and Martijn~C. Willemsen.
\newblock Behaviorism is not enough: Better recommendations through listening
  to users.
\newblock In Shilad Sen, Werner Geyer, Jill Freyne, and Pablo Castells,
  editors, \emph{Proceedings of the 10th {ACM} Conference on Recommender
  Systems, Boston, MA, USA, September 15-19, 2016}, pages 221--224. {ACM},
  2016.

\bibitem[Even-Dar et~al.(2002)Even-Dar, Mannor, and Mansour]{even2002pac}
Eyal Even-Dar, Shie Mannor, and Yishay Mansour.
\newblock Pac bounds for multi-armed bandit and markov decision processes.
\newblock In \emph{COLT}, volume~2, pages 255--270. Springer, 2002.

\bibitem[Fiez et~al.(2019)Fiez, Chasnov, and Ratliff]{fiez2019convergence}
Tanner Fiez, Benjamin Chasnov, and Lillian~J Ratliff.
\newblock Convergence of learning dynamics in stackelberg games.
\newblock \emph{arXiv preprint arXiv:1906.01217}, 2019.

\bibitem[Gan et~al.(2023)Gan, Han, Wu, and Xu]{gan2023robust}
Jiarui Gan, Minbiao Han, Jibang Wu, and Haifeng Xu.
\newblock Robust stackelberg equilibria.
\newblock \emph{arXiv preprint arXiv:2304.14990}, 2023.

\bibitem[Goktas et~al.(2022)Goktas, Zhao, and Greenwald]{GZG22}
Denizalp Goktas, Jiayi Zhao, and Amy Greenwald.
\newblock Robust no-regret learning in min-max stackelberg games.
\newblock In \emph{21st International Conference on Autonomous Agents and
  Multiagent Systems, {AAMAS} 2022, Auckland, New Zealand, May 9-13, 2022},
  pages 543--552. International Foundation for Autonomous Agents and Multiagent
  Systems {(IFAAMAS)}, 2022.

\bibitem[Guo et~al.(2023)Guo, Haghtalab, Kandasamy, and Vitercik]{GHKV23}
Wenshuo Guo, Nika Haghtalab, Kirthevasan Kandasamy, and Ellen Vitercik.
\newblock Leveraging reviews: Learning to price with buyer and seller
  uncertainty.
\newblock In \emph{Proceedings of the 24th {ACM} Conference on Economics and
  Computation, {EC} 2023, London, United Kingdom, July 9-12, 2023}, 2023.

\bibitem[Guruganesh et~al.(2024)Guruganesh, Kolumbus, Schneider,
  Talgam{-}Cohen, Vlatakis{-}Gkaragkounis, Wang, and Weinberg]{GKSTGWW24}
Guru Guruganesh, Yoav Kolumbus, Jon Schneider, Inbal Talgam{-}Cohen,
  Emmanouil{-}Vasileios Vlatakis{-}Gkaragkounis, Joshua~R. Wang, and S.~Matthew
  Weinberg.
\newblock Contracting with a learning agent.
\newblock \emph{CoRR}, abs/2401.16198, 2024.

\bibitem[Haghtalab et~al.(2022)Haghtalab, Lykouris, Nietert, and
  Wei]{haghtalab2022learning}
Nika Haghtalab, Thodoris Lykouris, Sloan Nietert, and Alexander Wei.
\newblock Learning in stackelberg games with non-myopic agents.
\newblock In \emph{Proceedings of the 23rd ACM Conference on Economics and
  Computation}, pages 917--918, 2022.

\bibitem[Haghtalab et~al.(2023)Haghtalab, Podimata, and
  Yang]{haghtalab2023calibrated}
Nika Haghtalab, Chara Podimata, and Kunhe Yang.
\newblock Calibrated stackelberg games: Learning optimal commitments against
  calibrated agents.
\newblock \emph{arXiv preprint arXiv:2306.02704}, 2023.

\bibitem[Hajiaghayi et~al.(2023)Hajiaghayi, Mahdavi, Rezaei, and
  Shin]{hajiaghayi2023regret}
MohammadTaghi Hajiaghayi, Mohammad Mahdavi, Keivan Rezaei, and Suho Shin.
\newblock Regret analysis of repeated delegated choice.
\newblock \emph{arXiv preprint arXiv:2310.04884}, 2023.

\bibitem[Han et~al.(2023)Han, Albert, and Xu]{HAX23}
Minbiao Han, Michael Albert, and Haifeng Xu.
\newblock Learning in online principal-agent interactions: The power of menus.
\newblock \emph{CoRR}, abs/2312.09869, 2023.

\bibitem[Harris et~al.(2024)Harris, Wu, and Balcan]{HWB24}
Keegan Harris, Zhiwei~Steven Wu, and Maria-Florina Balcan.
\newblock Regret minimization in stackelberg games with side information.
\newblock \emph{CoRR}, abs/2402.08576, 2024.

\bibitem[Hong et~al.(2023)Hong, Levine, and Dragan]{HLD23}
Joey Hong, Sergey Levine, and Anca~D. Dragan.
\newblock Zero-shot goal-directed dialogue via {RL} on imagined conversations.
\newblock \emph{CoRR}, abs/2311.05584, 2023.

\bibitem[Jagadeesan et~al.(2023)Jagadeesan, Jordan, and Haghtalab]{JJH23}
Meena Jagadeesan, Michael~I. Jordan, and Nika Haghtalab.
\newblock Competition, alignment, and equilibria in digital marketplaces.
\newblock In \emph{Thirty-Seventh {AAAI} Conference on Artificial Intelligence,
  {AAAI} 2023}, pages 5689--5696, 2023.

\bibitem[Kao et~al.(2022)Kao, Wei, and Subramanian]{KWS22}
Hsu Kao, Chen{-}Yu Wei, and Vijay~G. Subramanian.
\newblock Decentralized cooperative reinforcement learning with hierarchical
  information structure.
\newblock In \emph{International Conference on Algorithmic Learning Theory, 29
  March - 1 April 2022, Paris, France}, volume 167 of \emph{Proceedings of
  Machine Learning Research}, pages 573--605. {PMLR}, 2022.

\bibitem[Kim(2015)]{kim2015human}
Gerard~Jounghyun Kim.
\newblock \emph{Human-computer interaction: fundamentals and practice}.
\newblock CRC press, 2015.

\bibitem[Kleinberg and Kleinberg(2018)]{kleinberg2018delegated}
Jon Kleinberg and Robert Kleinberg.
\newblock Delegated search approximates efficient search.
\newblock In \emph{Proceedings of the 2018 ACM Conference on Economics and
  Computation}, pages 287--302, 2018.

\bibitem[Kleinberg et~al.(2022)Kleinberg, Mullainathan, and Raghavan]{KMR22}
Jon~M. Kleinberg, Sendhil Mullainathan, and Manish Raghavan.
\newblock The challenge of understanding what users want: Inconsistent
  preferences and engagement optimization.
\newblock \emph{CoRR}, abs/2202.11776, 2022.

\bibitem[Kolumbus and Nisan(2022)]{kolumbus2022and}
Yoav Kolumbus and Noam Nisan.
\newblock How and why to manipulate your own agent: On the incentives of users
  of learning agents.
\newblock \emph{Advances in Neural Information Processing Systems},
  35:\penalty0 28080--28094, 2022.

\bibitem[Lattimore and Szepesvári(2020)]{L20}
Tor Lattimore and Csaba Szepesvári.
\newblock \emph{Bandit Algorithms}.
\newblock University of Cambridge ESOL Examinations, 2020.
\newblock ISBN 9781108571401.

\bibitem[Lauffer et~al.(2023)Lauffer, Ghasemi, Hashemi, Savas, and
  Topcu]{lauffer2023no}
Niklas Lauffer, Mahsa Ghasemi, Abolfazl Hashemi, Yagiz Savas, and Ufuk Topcu.
\newblock No-regret learning in dynamic stackelberg games.
\newblock \emph{IEEE Transactions on Automatic Control}, 2023.

\bibitem[Lazar et~al.(2017)Lazar, Feng, and Hochheiser]{lazar2017research}
Jonathan Lazar, Jinjuan~Heidi Feng, and Harry Hochheiser.
\newblock \emph{Research methods in human-computer interaction}.
\newblock Morgan Kaufmann, 2017.

\bibitem[Letchford et~al.(2009)Letchford, Conitzer, and Munagala]{LCM09}
Joshua Letchford, Vincent Conitzer, and Kamesh Munagala.
\newblock Learning and approximating the optimal strategy to commit to.
\newblock In \emph{Algorithmic Game Theory, Second International Symposium,
  {SAGT} 2009, Paphos, Cyprus, October 18-20, 2009. Proceedings}, volume 5814
  of \emph{Lecture Notes in Computer Science}, pages 250--262. Springer, 2009.

\bibitem[Lin and Chen(2024)]{lin2024persuading}
Tao Lin and Yiling Chen.
\newblock Persuading a learning agent.
\newblock \emph{arXiv preprint arXiv:2402.09721}, 2024.

\bibitem[Lucier et~al.(2023)Lucier, Pattathil, Slivkins, and Zhang]{LPSZ23}
Brendan Lucier, Sarath Pattathil, Aleksandrs Slivkins, and Mengxiao Zhang.
\newblock Autobidders with budget and {ROI} constraints: Efficiency, regret,
  and pacing dynamics.
\newblock \emph{CoRR}, abs/2301.13306, 2023.

\bibitem[MacKenzie(2024)]{mackenzie2024human}
I.~Scott MacKenzie.
\newblock Human-computer interaction: An empirical research perspective.
\newblock 2024.

\bibitem[Miller et~al.(2021)Miller, Perdomo, and Zrnic]{miller2021outside}
John~P Miller, Juan~C Perdomo, and Tijana Zrnic.
\newblock Outside the echo chamber: Optimizing the performative risk.
\newblock In \emph{International Conference on Machine Learning}, pages
  7710--7720. PMLR, 2021.

\bibitem[Milli et~al.(2021)Milli, Belli, and Hardt]{MBH21}
Smitha Milli, Luca Belli, and Moritz Hardt.
\newblock From optimizing engagement to measuring value.
\newblock In \emph{FAccT '21: 2021 {ACM} Conference on Fairness,
  Accountability, and Transparency, Virtual Event / Toronto, Canada, March
  3-10, 2021}, pages 714--722. {ACM}, 2021.

\bibitem[Nekipelov et~al.(2015)Nekipelov, Syrgkanis, and
  Tardos]{nekipelov2015econometrics}
Denis Nekipelov, Vasilis Syrgkanis, and Eva Tardos.
\newblock Econometrics for learning agents.
\newblock In \emph{Proceedings of the sixteenth acm conference on economics and
  computation}, pages 1--18, 2015.

\bibitem[Nisan and Noti(2017)]{10.1145/3038912.3052621}
Noam Nisan and Gali Noti.
\newblock An experimental evaluation of regret-based econometrics.
\newblock In \emph{Proceedings of the 26th International Conference on World
  Wide Web}, WWW '17, page 73–81. International World Wide Web Conferences
  Steering Committee, 2017.

\bibitem[Noti and Syrgkanis(2021)]{10.1145/3442381.3449968}
Gali Noti and Vasilis Syrgkanis.
\newblock Bid prediction in repeated auctions with learning.
\newblock In \emph{Proceedings of the Web Conference 2021}, WWW '21, page
  3953–3964, New York, NY, USA, 2021. Association for Computing Machinery.

\bibitem[Pacchiano et~al.(2020)Pacchiano, Phan, Abbasi{-}Yadkori, Rao, Zimmert,
  Lattimore, and Szepesv{\'{a}}ri]{PPA0ZLS20}
Aldo Pacchiano, My~Phan, Yasin Abbasi{-}Yadkori, Anup Rao, Julian Zimmert, Tor
  Lattimore, and Csaba Szepesv{\'{a}}ri.
\newblock Model selection in contextual stochastic bandit problems.
\newblock In \emph{Advances in Neural Information Processing Systems 33: Annual
  Conference on Neural Information Processing Systems 2020, NeurIPS 2020,
  December 6-12, 2020, virtual}, 2020.

\bibitem[Pan et~al.(2024)Pan, Jones, Jagadeesan, and Steinhardt]{PJJS24}
Alexander Pan, Erik Jones, Meena Jagadeesan, and Jacob Steinhardt.
\newblock Feedback loops with language models drive in-context reward hacking.
\newblock \emph{CoRR}, abs/2402.06627, 2024.

\bibitem[Perdomo et~al.(2020)Perdomo, Zrnic, Mendler-D{\"u}nner, and
  Hardt]{perdomo2020performative}
Juan Perdomo, Tijana Zrnic, Celestine Mendler-D{\"u}nner, and Moritz Hardt.
\newblock Performative prediction.
\newblock In \emph{International Conference on Machine Learning}, pages
  7599--7609. PMLR, 2020.

\bibitem[Pita et~al.(2009)Pita, Jain, Ord{\'o}{\~n}ez, Tambe, Kraus, and
  Magori-Cohen]{pita2009effective}
James Pita, Manish Jain, Fernando Ord{\'o}{\~n}ez, Milind Tambe, Sarit Kraus,
  and Reuma Magori-Cohen.
\newblock Effective solutions for real-world stackelberg games: When agents
  must deal with human uncertainties.
\newblock In \emph{Proceedings of The 8th International Conference on
  Autonomous Agents and Multiagent Systems-Volume 1}, pages 369--376, 2009.

\bibitem[Preece et~al.(1994)Preece, Rogers, Sharp, Benyon, Holland, and
  Carey]{preece1994human}
Jenny Preece, Yvonne Rogers, Helen Sharp, David Benyon, Simon Holland, and Tom
  Carey.
\newblock \emph{Human-computer interaction}.
\newblock Addison-Wesley Longman Ltd., 1994.

\bibitem[Slivkins(2019)]{S19}
Aleksandrs Slivkins.
\newblock Introduction to multi-armed bandits.
\newblock \emph{Found. Trends Mach. Learn.}, 12\penalty0 (1-2):\penalty0
  1--286, 2019.

\bibitem[Straitouri and Rodriguez(2023)]{straitouri2023designing}
Eleni Straitouri and Manuel~Gomez Rodriguez.
\newblock Designing decision support systems using counterfactual prediction
  sets.
\newblock \emph{arXiv preprint arXiv:2306.03928}, 2023.

\bibitem[Straitouri et~al.(2023)Straitouri, Wang, Okati, and
  Rodriguez]{straitouri2023improving}
Eleni Straitouri, Lequn Wang, Nastaran Okati, and Manuel~Gomez Rodriguez.
\newblock Improving expert predictions with conformal prediction.
\newblock In \emph{International Conference on Machine Learning}, pages
  32633--32653. PMLR, 2023.

\bibitem[Stray et~al.(2021)Stray, Vendrov, Nixon, Adler, and
  Hadfield{-}Menell]{SVNAM21}
Jonathan Stray, Ivan Vendrov, Jeremy Nixon, Steven Adler, and Dylan
  Hadfield{-}Menell.
\newblock What are you optimizing for? aligning recommender systems with human
  values.
\newblock \emph{CoRR}, abs/2107.10939, 2021.

\bibitem[Wang et~al.(2022)Wang, Joachims, and Rodriguez]{wang2022improving}
Lequn Wang, Thorsten Joachims, and Manuel~Gomez Rodriguez.
\newblock Improving screening processes via calibrated subset selection.
\newblock In \emph{International Conference on Machine Learning}, pages
  22702--22726. PMLR, 2022.

\bibitem[Yang et~al.(2019)Yang, Liu, Chen, and Tong]{YLCT19}
Qiang Yang, Yang Liu, Tianjian Chen, and Yongxin Tong.
\newblock Federated machine learning: Concept and applications.
\newblock \emph{{ACM} Trans. Intell. Syst. Technol.}, 10\penalty0 (2):\penalty0
  12:1--12:19, 2019.

\bibitem[Yao et~al.(2022)Yao, Li, Nekipelov, Wang, and Xu]{yao2022learning}
Fan Yao, Chuanhao Li, Denis Nekipelov, Hongning Wang, and Haifeng Xu.
\newblock Learning the optimal recommendation from explorative users.
\newblock In \emph{Proceedings of the AAAI Conference on Artificial
  Intelligence}, volume~36, pages 9457--9465, 2022.

\bibitem[Zhang et~al.(2019)Zhang, Yang, and Basar]{ZYB19}
Kaiqing Zhang, Zhuoran Yang, and Tamer Basar.
\newblock Multi-agent reinforcement learning: {A} selective overview of
  theories and algorithms.
\newblock \emph{CoRR}, abs/1911.10635, 2019.

\bibitem[Zhao et~al.(2023)Zhao, Zhu, Jiao, and Jordan]{zhao2023online}
Geng Zhao, Banghua Zhu, Jiantao Jiao, and Michael Jordan.
\newblock Online learning in stackelberg games with an omniscient follower.
\newblock In \emph{International Conference on Machine Learning}, pages
  42304--42316. PMLR, 2023.

\bibitem[Zhu et~al.(2023)Zhu, Bates, Yang, Wang, Jiao, and Jordan]{ZBYWJJ23}
Banghua Zhu, Stephen Bates, Zhuoran Yang, Yixin Wang, Jiantao Jiao, and
  Michael~I. Jordan.
\newblock The sample complexity of online contract design.
\newblock In \emph{Proceedings of the 24th {ACM} Conference on Economics and
  Computation, {EC} 2023, London, United Kingdom, July 9-12, 2023}, page 1188.
  {ACM}, 2023.

\bibitem[Zhuang and Hadfield{-}Menell(2020)]{ZH20}
Simon Zhuang and Dylan Hadfield{-}Menell.
\newblock Consequences of misaligned {AI}.
\newblock In \emph{Advances in Neural Information Processing Systems 33: Annual
  Conference on Neural Information Processing Systems 2020, NeurIPS 2020,
  December 6-12, 2020, virtual}, 2020.

\bibitem[Zrnic et~al.(2021)Zrnic, Mazumdar, Sastry, and Jordan]{zrnic2021}
Tijana Zrnic, Eric Mazumdar, S.~Shankar Sastry, and Michael~I. Jordan.
\newblock Who leads and who follows in strategic classification?
\newblock In \emph{Advances in Neural Information Processing Systems 34: Annual
  Conference on Neural Information Processing Systems 2021, NeurIPS 2021,
  December 6-14, 2021, virtual}, pages 15257--15269, 2021.

\bibitem[Zuo and Tang(2015)]{ZT15}
Song Zuo and Pingzhong Tang.
\newblock Optimal machine strategies to commit to in two-person repeated games.
\newblock In \emph{Proceedings of the Twenty-Ninth {AAAI} Conference on
  Artificial Intelligence, January 25-30, 2015, Austin, Texas, {USA}}, pages
  1071--1078. {AAAI} Press, 2015.

\end{thebibliography}

\newpage
\appendix

\addcontentsline{toc}{section}{Appendix} 
\part{Appendix} 
\parttoc 

\section{Worked-out examples, auxiliary notation, and auxiliary lemmas}\label{app:benchmarkdiscuss}

\begin{table}[]
\centering 
\begin{tabular}{|c|c|c|c|}
\hline
                              & $b_1$                                         & $b_2$                                       & $b_3$                                        \\ \hline
\cellcolor[HTML]{FFCCC9}$a_1$ & \cellcolor[HTML]{DAE8FC}$(1, 0.5+2\delta)$    & \cellcolor[HTML]{DAE8FC}$(0.7, 0.5+\delta)$ & $(1.1, 0)$                                   \\ \hline
\rowcolor[HTML]{DAE8FC} 
\cellcolor[HTML]{FFCCC9}$a_2$ & {\color[HTML]{960095} $(0.8, 3.5\cd \delta)$} & $(1.2, 3\cd \delta)$                        & {\color[HTML]{036400} $(0.9, 4 \cd \delta)$} \\ \hline
$a_3$                          & \cellcolor[HTML]{DAE8FC}$(0.5, 0.5)$          & $(0.7, 0)$                                  & $(2, 0.1)$                                   \\ \hline
\end{tabular}
\caption{Calculating the $\delta$-tolerant benchmark: Note that $(a_1, b_1)$ is the Stackelberg equilibrium,  which by Theorem \ref{thrm:worstlinear} cannot in general be learned with sublinear regret. For each row, cells shaded in {\color{blue}blue} if they are within the $\delta$ best response for the follower ($\mathcal{B}_{\delta}(a_i)$). Entry $(a_2, b_1)$ (with {\color[HTML]{960095} purple} text) gives the leader's $\delta$-relaxed Stackelberg utility - the leader's best action, assuming the follower picks the worst item within the $\delta$-response ball. Rows $a_1, a_2$ (shaded in {\color{red}red}) are in $\mathcal{A}_{\delta}$, the set of actions where the leader has a chance of doing at least as well as the $\delta$-relaxed Stackelberg utility ($(a_2, b_1$)). Finally, $(a_2, b_3)$ (in {\color[HTML]{036400} green}) gives the follower's best response, assuming the leader picks the worst action for it within $\mathcal{A}_{\delta}$. }
\label{tab:workedex}
\end{table}

\begin{table}[]
\centering 
\begin{tabular}{|c|c|c|c|}
\hline
                              & $b_1$                                         & $b_2$                                       & $b_3$                                        \\ \hline
\cellcolor[HTML]{FFCCC9}$a_1$ & \cellcolor[HTML]{DAE8FC}$(1, 0.5+2\delta)$    & \cellcolor[HTML]{DAE8FC} {\color[HTML]{960095}$(0.7, 0.5+\delta)$} & $(1.1, 0)$                                   \\ \hline
\rowcolor[HTML]{DAE8FC} 
\cellcolor[HTML]{FFCCC9}$a_2$ & $(0.8, 3.5\cd \delta)$ & {\color[HTML]{036400}$(1.2, 3\cd \delta)$}                        &  $(0.9, 4 \cd \delta)$ \\ \hline
$a_3$                          & \cellcolor[HTML]{DAE8FC}$(0.5, 0.5)$          & $(0.7, 0)$                                  & $(2, 0.1)$                                   \\ \hline
\end{tabular}
\caption{Calculating the \textbf{self-}$\delta$-tolerant benchmark: Note that $\mathcal{B}_{\delta}, \mathcal{A}_{\delta}$ are defined the same as in the $\gamma$-tolerant benchmark in Table \ref{tab:workedex}, so the only difference is the location of the $\delta$-relaxed Stackelberg utility values for the leader and the follower, which are calculated by finding the \emph{worst} expected reward for each within the $\mathcal{B}_{\delta}, \mathcal{A}_{\delta}$ sets. Here, they occur for the leader in $(a_1, b_2)$ (in {\color[HTML]{960095}  purple}) and for the follower in $(a_2, b_2)$ (in {\color[HTML]{036400} green}). }
\label{tab:workedex_relax}
\end{table}

\subsection{Worked out version of Example \ref{example:gammatolerant} for $\gamma$-tolerant benchmark }\label{appendix:workedoutgammatolerantexample}

We work out the $\gamma$-tolerant benchmark for Example \ref{example:gammatolerant} in more detail. Consider instance $\Instance$ (leftmost table) in Table \ref{tab:main_2} (with $ 0.4> \gamma \geq 4 \delta$), which we will use to illustrate our benchmark. We show that  $\benchmarkrelaxedstronger{1} = 0.5+\delta$ and $\benchmarkrelaxedstronger{2} = 4 \delta$.
To calculate the benchmarks, we compute the sum of the $\epsilon$-relaxed Stackelberg value and $\epsilon$-regularizer for different values of $\epsilon$ and then take a minimum. We will show that the minimum turns out to be achieved at $\epsilon = \delta$. 

First, for $\epsilon=0$ this benchmark is equal to the Stackelberg equilibrium, which gives values $0.5+\delta, 0.4$ for the leader and follower respectively. For $\epsilon \in (0, \delta)$, the $\epsilon$-relaxed Stackelberg value stays the same while the regularizer increases. For $\epsilon = \delta$, the behavior of the $\epsilon$-Stackelberg utility becomes more complicated. 
\begin{itemize}
    \item \textbf{Follower $\epsilon$-best-response set:} In this instance, $\mathcal{B}_{\delta}(a_1) = \{a_1\}$: for arm $a_1$, because $0.4>\delta$, only $\{b_1\}$ is in the best-response set. However, $\mathcal{B}_{\delta}(a_2) = \{b_1, b_2\}$: both arms for the follower are within $\delta$ of optimal.
    
    \item \textbf{Leader $\epsilon$-relaxed Stackelberg utility:} This term captures the best utility that the leader can achieve if the follower worst-case $\epsilon$-best-responds according to $\argmin_{b \in \mathcal{B}_{\delta}}(a)$. Since $\mathcal{B}_{\delta}(a_1) = \{b_1\}$, we see that $\min_{b \in \mathcal{B}_{\delta}}(a_1) = 0.5+\delta$. However, for $a_2$, $\min_{b \in \mathcal{B}_{\delta}}(a) = \MR{1}{a_2}{b_2} = 0.4$. The leader's best action is to pick $a_1$, so the $\delta$-relaxed Stackelberg utility is equal to $0.5+\delta$. 
    \item \textbf{Leader $\epsilon$-best-response sets: } We construct the $\mathcal{A}_{\delta}$ set by considering all actions $a$ where the \emph{best-case} outcome within the $\mathcal{B}_{\delta}(a)$ gives reward at least within $\delta$ of our benchmark value of $0.5+\delta$. We can see $\mathcal{A}_{\delta} = \{a_1, a_2\}$ because they both contain an item within $\delta$ of the benchmark value ($(a_1, b_1)$ or $(a_2, b_1)$ respectively). 
    \item \textbf{Follower's $\epsilon$-relaxed Stackelberg utility:} This term considers the worst-case action within $\mathcal{A}_{\delta}$ for the follower. If the leader picks $a_1$, the only response is $b_1$ which gives value $0.4$, while if the leader picks $a_2$, the best response is $b_2$ which gives value $3\cd \delta$. The minimum of these, plus a regularizer term, gives a benchmark of $4 \cd \delta $. 
\end{itemize}
The above analysis shows that for $\epsilon = \delta$, the $\epsilon$-relaxed Stackelberg utility plus the $\epsilon$-regularizer  are equal to $(0.5+2 \delta, 4 \delta)$ for the leader and follower, respectively. For $\epsilon \in (\delta, \gamma)$, the best response sets will not change, but the penalty for $\epsilon$ will increase, so these will not affect the infimum. Taking the minimum over the calculated benchmarks for $\epsilon \in \{0,\gamma\}$ gives $0.5+\delta, 4 \delta$ for the leader and follower respectively.

\subsection{Worked out version of Example \ref{example:gammatolerant} for self-$\gamma$-tolerant benchmark }\label{appendix:workedoutselfgammatolerantexample}

We work out the self-$\gamma$-tolerant benchmark for Example \ref{example:gammatolerant} in more detail. Again, consider $\Instance$ in Table \ref{tab:main_2}, which we also used to illustrate the $\gamma$-tolerant benchmark in Example \ref{example:gammatolerant}. Recall that for $\epsilon = 0$, we recover the Stackelberg equilibrium benchmark of $(0.5+\delta, 0.1)$ for the leader and follower, respectively. For $\epsilon \in (0, \delta)$ the $\mathcal{B}_{\epsilon}(a), \mathcal{A}_{\epsilon}$ sets don't change, but the penalty increases, so this is irrelevant for the infimum. Recall that from that analysis, we found that $\mathcal{B}_{\delta}(a_1) =\{b_1\}, \mathcal{B}_{\delta}(a_2) = \{a_1, a_2\}$, and $\mathcal{A}_{\delta} = \{a_1, a_2\}$. The self$-\gamma$-tolerance benchmark only requires each agent to compete with the \emph{worst} element within the product set $\mathcal{A}_{\delta} \times  \mathcal{B}_{\delta}(a)$ (if we consider the instance where $\epsilon = \delta$). 

For the given instance, this gives the benchmarks for the leader and follower of $0.4+\delta$ and $2 \cd \delta + \delta$, where we have added a $\delta$ regularizer penalty to both. Finally, we note that for $\epsilon \in (\delta, 0.1)$, again the $\mathcal{B}_{\epsilon}(a), \mathcal{A}_{\epsilon}$ sets do not change but the penalty increases, so these are again irrelevant for the infimum. Taking the minimum of the benchmarks over $\epsilon \in \{0, \delta\}$ gives $0.4+\delta, 3\delta$ for the leader and follower respectively. Note that this differs from the $\gamma$-tolerant benchmark for the follower only by $\delta$, but differs by $0.1$ (a constant) for the leader. 

\subsection{Additional worked out example for the benchmark}

Tables \ref{tab:workedex} and \ref{tab:workedex_relax} contain worked examples of how the benchmarks are calculated for more complex examples.

\subsection{Additional Notation and Auxiliary Lemmas}\label{app:hist}

We introduce the following notation and auxiliary lemmas which will be convenient in our proofs.

\paragraph{Notation for Player Histories.}
First, we introduce the following notation for the player histories that will be convenient to use in algorithmic specifications and proofs. 

In a \textit{weakly decentralized Stackelberg game ($\WeakSetting$)}, let the leader's history up to time step $t$ be the set of arms that were pulled, as well as the reward for the leader at each time step: 
\[\Hist_{1, t} := \left\{(t', a_{t'}, b_{t'}, \RR{1}{a_{t'}}{b_{t'}}{t'})  \mid 1 \le t' < t \right\}.\]
In a \textit{strongly decentralized Stackelberg game ($\StrongSetting$)}, the leader cannot even observe the action chosen 
by the follower, but the follower's information remains unchanged. That is $\Hist_{1, t} := \left\{(t', a_{t'}, \RR{1}{a_{t'}}{b_{t'}}{t'}) \mid 1 \le t' < t \right\}$.

Let the follower's history be 
\[\Hist_{2, t} := \left\{(t', a_{t'}, b_{t'}, \RR{2}{a_{t'}}{b_{t'}}{t'}) \mid 1 \le t' < t \right\}.\] 
When the follower runs a separate algorithm on each choice of $a \in \mathcal{A}$ and does not share information across arms (e.g., in Proposition \ref{prop:decentralizedlower}, Theorem \ref{thm:ETCETC}, \Cref{prop:UCB}, and \Cref{prop:AAE}), then the follower's history for the arm $a \in \mathcal{A}$ is given by: 
\[\Hist_{2, t, a} := \left\{(\NumPullsOneArg{a}{t'+1}, b_{t'}, \RR{2}{a_{t'}}{b_{t'}}{t'}) \mid 1 \le t' < t, a_{t'} = a\right\},\]
where $\NumPullsOneArg{a}{t'+1}$ is the number of times that arm $a$ is pulled prior to the $(t'+1)$th time step. 

\paragraph{Auxiliary lemma for regret analysis.} Next, we introduce the following auxiliary lemma which will be useful in the regret analysis.
\begin{lemma}
\label{lemma:boundsqrtexpression}
Let $\mathcal{C}$ be a finite set of arms and let $T \ge 1$ be a  time horizon. Let $(c_1, \ldots, c_T) \in \mathcal{C}^T$ denote any history of arm pulls. Let $\NumPullsOneArg{c}{t} = \sum_{t=1}^{t-1} \mathbbm{1}[c_{t'} = c]$ denote the number of times that $c$ is pulled prior to time step $t$. Then it holds that:
\[\sum_{c \in \mathcal{C}} \frac{1}{\sqrt{\NumPullsOneArg{c}{t}}} \le O\left(\sqrt{T \cdot |\mathcal{C}|} \right)  \]
\end{lemma}
\begin{proof}
We observe that 
\[\sum_{c \in \mathcal{C}} \frac{1}{\sqrt{\NumPullsOneArg{c}{t}}} = \sum_{c \in \mathcal{C}} \sum_{n=1}^{\NumPullsOneArg{c}{t}} \frac{1}{\sqrt{n}} \le_{(A)} \sum_{c \in \mathcal{C}} O \left(\sqrt{\NumPullsOneArg{c}{t} + 1} \right) \le_{(B)} 
O\left(\sqrt{T \cdot |\mathcal{C}|} \right) \le  ,\]
where (A) follows from an integral bound and (B) follows from Jensen's inequality.    
\end{proof}

\section{Proofs of regret lower bounds}

Our regret bounds analyze a centralized setting (\Cref{app:centralizedgame}) and build on standard tools \citep{L20} for regret lower bounds (\Cref{appendix:tools}). We prove Proposition \ref{prop:lowerboundsqrtt} in  \Cref{appendix:lowersqrtt},  Theorem \ref{thrm:worstlinear} in \Cref{appendix:worstlinear}, and Theorem \ref{thrm:dlower} in \Cref{appendix:dlower}.

\subsection{Centralized environment}\label{app:centralizedgame}

When analyzing regret lower bounds, it is also convenient to consider a \textit{centralized environment} where a single player controls the actions of both players and observes all past actions. While the centralized environment is not our primary focus, it can (informally speaking) be viewed as a limiting case of the decentralized setting with extremely sophisticated players who could communicate their strategies to each other. We define the history for the centralized environment to be: 
\[\CentHist_{t} = \left\{(t', a_{t'}, b_{t'}, \RR{1}{a_{t'}}{b_{t'}}{t'}, \RR{2}{a_{t'}}{b_{t'}}{t'}) \mid 1 \le t' \le t  \right\}.\] The centralized player chooses an algorithm $\ALG$ mapping a history to a joint distribution over pairs of actions. 

We show that centralized algorithms are strictly more general than decentralized environments, in that any rewards realized in a decentralized environment can also be realized in a centralized environment. 
\begin{lemma}
\label{lemma:centralizedimpliesdecentralized}
Consider a $\StrongSetting$ or $\WeakSetting$. Fix an instance $\Instance = (\mathcal{A}, \mathcal{B}, v_1, v_2)$ and time horizon $T$. 
For any pair of algorithms $\ALG_1$ and $\ALG_2$, there exists a centralized algorithm $\ALG$ such that the leader rewards $(r_{1,1}(a_1, b_1), \ldots, r_{1,T}(a_T, b_T))$
are identically distributed for $\ALG$ and $(\ALG_1, \ALG_2)$ and the follower rewards $(r_{2,1}(a_1, b_1), \ldots, r_{2,T}(a_T, b_T))$ are also  identically distributed for $\ALG$ and $(\ALG_1, \ALG_2)$. 
\end{lemma}

Lemma \ref{lemma:centralizedimpliesdecentralized} follows immediately from designing $\ALG$ to ``simulate'' histories for the leader and the follower (by projecting away the information unavailable to each player) and then to choose arms by applying $\ALG_1$ and $\ALG_2$ on these histories. 

\subsection{Useful lemmas}\label{appendix:tools}

Our regret bounds leverage the following standard tools \citep{L20} which we restate for completeness. Like in \citet{L20}, we will use the Bretagnolle–Huber inequality. 
\begin{theorem}[paraphrased from Theorem 14.2 in \citep{L20}]
\label{thm:bhinequality}
Let $P$ and $Q$ be probability measures on the same measurable space $(\Omega, \mathcal{F})$, and let $E \in \mathcal{F}$ be an arbitrary event. Then it holds that:
\[ P(\Event) + Q(\Event^c) \ge \frac{1}{2} e^{-KL(P,Q)} \]
where $\Event^c = \Omega \setminus \Event$ is the complement of $\Event$ and $KL(P,Q)$ is the KL divergence between $P$ and $Q$. 
\end{theorem}

We similarly work with the canonical bandit model (Section 4.6 in \citet{L20}) but with some modifications because there are two observed rewards (for the leader and the follower) in our setup. We call the analogous setup  in our setting the \textit{canonical centralized bandit model}. Note that the sample space of the probability space is now $(((\mathcal{A} \times \mathcal{B}) \times \mathbb{R} \times \mathbb{R})^T$ (instead of $([k]\times \mathbb{R})^T$, like in the typical canonical bandit model).  

We show an analogous \textit{divergence decomposition} (Lemma 15.1 in \citet{L20}) applies to our setting. For this result, fix $\mathcal{A}$ and $\mathcal{B}$, and let $v$ and $\tilde{v}$ be two different specifications of utilities. For $i \in \left\{1, 2\right\}$, let $r_i(a,b)$ denote the reward distribution $N(\MR{i}{a}{b}, 1)$ and let $\tilde{r}_i(a,b)$ denote the reward distribution $N(\MRTilde{i}{a}{b}, 1)$. 
\begin{theorem}[adapted from Lemma 15.1 in \citep{L20}]
\label{thm:divergencedecomposition}
Fix an algorithm $\ALG$ for the centralized environment.  Let $P$ (resp. $\tilde{P}$) denote the probability measure corresponding to the canonical centralized bandit model for $\ALG$ applied to $(\mathcal{A}, \mathcal{B}, v)$ (resp. $(\mathcal{A}, \mathcal{B}, \tilde{v})$). Let $\NumPulls{a}{b}{T} = \sum_{t=1}^T \mathbbm{1}[a_t = a, b_t = b]$ denote the number of times that arm $(a,b)$ is pulled. Then it holds that:
\[D(P, \tilde{P}) = \sum_{(a,b) \in \mathcal{A} \times \mathcal{B}} \mathbb{E}_{P}[\NumPulls{a}{b}{T}] \cdot (D(r_1(a,b), \tilde{r}_1(a,b)) + D(R_2(a,b), \tilde{r}_2(a,b)). \]
where $D(\cdot, \cdot)$ denotes the KL divergence, where $r_i(a,b)$  denotes the reward distribution $N(\MR{i}{a}{b}, 1)$ and $\tilde{r}_i(a,b)$ denotes the reward distribution $N(\MRTilde{i}{a}{b}, 1)$ for $i = 1,2$. 
\end{theorem}
\begin{proof}
This follows from the exact same argument as the proof in \citet{L20}, where $X_t$ is interpreted as the \textit{pair} of rewards $(\RR{1}{a_t}{b_t}{t}, \RR{2}{a_t}{b_t}{t})$ (or $(\RRTilde{1}{a_t}{b_t}{t}, \RRTilde{2}{a_t}{b_t}{t})$ observed at time step $t$. Let $r(a,b)$ be the product distribution $r_1(a,b) \times r_2(a,b)$, and let $\tilde{r}(a,b)$ be the product distribution $\tilde{r}_1(a,b)\times \tilde{r}_2(a,b)$.
This yields:
\[D(P, \tilde{P}) = \sum_{(a,b) \in \mathcal{A} \times \mathcal{B}} \mathbb{E}_P[\NumPulls{a}{b}{T}] \cdot D(r(a,b), \tilde{r}(a,b)). \]
The result follows from applying the ``chain rule'' which implies that the KL divergence of a product distribution is the sum of KL divergences of the individual distributions: 
\[D(r(a,b), \tilde{r}(a,b)) = D(r_1(a,b), \tilde{r}_1(a,b)) + D(r_2(a,b), \tilde{r}_2(a,b).\]
\end{proof}

Recall that we assume Gaussian noise, which further simplifies Theorem \ref{thm:divergencedecomposition}. By applying standard KL divergence bounds for univariate Gaussians, we obtain the following corollary of Theorem \ref{thm:divergencedecomposition}. 
\begin{corollary}
\label{cor:divergencedecompositiongaussian}
Fix an algorithm $\ALG$ for the centralized environment.  Let $P$ (resp. $\tilde{P}$) denote the probability measure corresponding to the canonical centralized bandit model for $\ALG$ applied to $(\mathcal{A}, \mathcal{B}, v)$ (resp. $(\mathcal{A}, \mathcal{B}, \tilde{v})$). Let $\NumPulls{a}{b}{T} = \sum_{t=1}^T \mathbbm{1}[a_t = a, b_t = b]$ denote the number of times that arm $(a,b)$ is pulled. Then it holds that:
\[D(P, \tilde{P}) = \sum_{(a,b) \in \mathcal{A} \times \mathcal{B}} \mathbb{E}_{P}[\NumPulls{a}{b}{T}] \cdot \frac{(\MR{1}{a}{b} - \MRTilde{1}{a}{b})^2 + (\MR{2}{a}{b} - \MRTilde{2}{a}{b})^2 }{2}. \]
\end{corollary}

\subsection{Proof for Proposition \ref{prop:lowerboundsqrtt}}\label{appendix:lowersqrtt}

We prove Proposition \ref{prop:lowerboundsqrtt}, restated below.
\lowerboundsqrtt*

\begin{proof}[Proof of Proposition \ref{prop:lowerboundsqrtt}]

\begin{table}[]
\centering 
\begin{tabular}{|c|c|c|c|c|}
\hline
         & $b_1$              & $\ldots$           & $b'$               & $\ldots$           \\ \hline
$a_1$    & $(\delta, \delta)$ & $(\delta, \delta)$ & $(\delta, \delta)$ & $(\delta, \delta)$ \\ \hline
$\vdots$ & $(0, 0)$           & $(0, 0)$           & $(0, 0)$           & $(0, 0)$           \\ \hline
$a'$     & $(0, 0)$           & $(0, 0)$           & *                  & $(0, 0)$           \\ \hline
$\vdots$ & $(0, 0)$           & $(0, 0)$           & $(0, 0)$           & $(0, 0)$           \\ \hline
\end{tabular}
\caption{Hard instance for Proposition \ref{prop:lowerboundsqrtt}, where $*$ is equal to $(0, 0)$ for instance $\mathcal{I}_{a_1, b_1}$, and $(2\delta, 2\delta)$ otherwise. }
\label{tab:sqrtboundex}
\end{table}

Fix $\mathcal{A}$ and $\mathcal{B}$ such that $|\mathcal{A}| \ge 2$ and $|\mathcal{B}| \ge 1$. 

 It suffices to prove this lower bound in a \textit{centralized} environment where a single learner can choose action pairs $(a,b)$ and observes rewards for both players (Lemma \ref{lemma:centralizedimpliesdecentralized}).
We define a family of instances in the centralized game and evaluate the self-tolerant benchmarks on this family of instances.  Arbitrarily pick some $a_1\in \mathcal{A}$ to be the ``base'' action. Let $\mathcal{F}_{\delta, \mathcal{A}, \mathcal{B}}$ be the family of $(|\mathcal{A}| - 1) \cdot |\mathcal{B}| + 1$ instances of the form $(\mathcal{A}, \mathcal{B}, v_1, v_2)$ for varying settings of $v_1$ and $v_2$, where we index the instances by $(a',b') \in ((\mathcal{A} \setminus \left\{a_1\right\}) \times \mathcal{B}) \cup \left\{(a_1, b_1) \right\}$. 
The utility functions for the instance $\mathcal{I}_{(a', b')}$ are equal to the terms below (illustrated in Table \ref{tab:sqrtboundex}): 
\[
\MR{1}{a}{b} = \MR{2}{a}{b} = 
\begin{cases}
  \delta & \text{ if } a = a_1  \\
  0 & \text{ if } (a', b') \neq (a,b), a \neq a_1  \\
  2 \delta  & \text{ if } (a', b') = (a, b), a \neq a_1
\end{cases}
\]

We claim that the $\benchmarkrelaxedweaker{1} = \benchmarkrelaxedweaker{2} = \delta$ for the instance $\mathcal{I}_{(a_1, b_1)}$ and $\benchmarkrelaxedweaker{1} = \benchmarkrelaxedweaker{2} = 2 \delta$ for the instances $\mathcal{I}_{(a', b')}$ where $(a', b') \neq (a_1, b_1)$. To see this, observe that on the instance $\mathcal{I}_{(a_1, b_1)}$, it holds that $\mathcal{B}_{\epsilon}(a_1) = \mathcal{B}$ and $\mathcal{A}_{\epsilon} = \left\{a_1 \right\}$ if $\epsilon < \delta$. Thus, it holds that $\min_{a \in \mathcal{A}_{\epsilon}} \min_{b \in \mathcal{B}_{\epsilon}(a)} \MR{1}{a}{b} + \epsilon \ge \delta$ for all $\epsilon$, so the benchmark is equal to 
\[ \benchmarkrelaxedweaker{1} = \benchmarkrelaxedweaker{2} = \delta,\]
as desired. 
On instances $\mathcal{I}_{(a', b')}$ where $(a', b') \neq (a_1, b_1)$, it holds that $\mathcal{B}_{\epsilon}(a') = \left\{ b' \right\}$ if $\epsilon < 2 \delta$ and $\mathcal{A}_{\epsilon} = \left\{a' \right\}$ if $\epsilon < \delta$. If $\epsilon < \delta$ or $\epsilon \ge 2 \delta$, then $\min_{a \in \mathcal{A}_{\epsilon}} \min_{b \in \mathcal{B}_{\epsilon}(a)} \MR{1}{a}{b} + \epsilon \ge 2 \delta$. If $\delta \le \epsilon < 2\delta$, then $\mathcal{A}_{\epsilon} = \left\{a', a_1 \right\}$ and it also holds that $\min_{a \in \mathcal{A}_{\epsilon}} \min_{b \in \mathcal{B}_{\epsilon}(a)} \MR{1}{a}{b} + \epsilon \ge 2 \delta$. 
This means that the self-tolerant benchmarks are equal to: 
\[ \benchmarkrelaxedweaker{1} = \benchmarkrelaxedweaker{2} = 2 \delta,\]
as desired. 

Because the utilities in $\mathcal{F}_{\delta, \mathcal{A}, \mathcal{B}}$ and the benchmarks are the same for the leader and follower, we see that the regret is also the same for both players. Thus, for the remainder of the analysis, we do not need to distinguish between the regret of the leader and the regret of the follower. Let $R(T; \mathcal{I})$ denote the regret incurred on instance $\Instance$. Since the benchmarks are equal to the maximum reward across all pairs of arms, the expected regret is always nonnegative. 

Fix any $\ALG$ for the centralized environment. For each $(a,b) \in ((\mathcal{A} \setminus \left\{a_1\right\}) \times \mathcal{B}) \cup \left\{(a_1, b_1) \right\}$, let $P_{a,b}$ denote the probability measure over canonical centralized bandit model when $\ALG$ is applied to the instance $\mathcal{I}_{a,b}$ (see \Cref{appendix:tools}). 
Let $\NumPulls{a}{b}{T} = \sum_{t=1}^T \mathbbm{1}[a_t = a, b_t = b]$ be the random variable denoting the number of times that $(a,b)$ is pulled. We define:  
\[(a_m, b_m) := \argmin_{(a,b) \in \mathcal{A} \times \mathcal{B} \mid a \neq a_1 } \mathbb{E}_{P_{(a_1, b_1)}}[ \NumPulls{a}{b}{T}]\]
to be the arm pulled the minimum number of times in expectation over $P_{(a_1, b_1)}$ (i.e., the expectation when $\ALG$ is applied to the instance $\mathcal{I}_{a_1,b_1}$). This means that 
\[\mathbb{E}_{P_{(a_1, b_1)}}[\NumPulls{a_m}{b_m}{T}] \le \frac{T}{(|\mathcal{A}| - 1) \cdot |\mathcal{B}|}.\] We will construct $\delta$ such that the regret is high on at least one of the instances $\mathcal{I}_{(a_1, b_1)}$ and $\mathcal{I}_{(a_m, b_m)}$. 

Now, let $\Event$ denote the event that $\sum_{b \in \mathcal{B}} \NumPulls{a_1}{b}{T} \le T/2$ (i.e., the arm $a_1$ is pulled less than $T/2$ times). It is easy to see that the regret satisfies:
\[ R(T; \mathcal{I}_{a_1, b_1}) \ge \frac{\delta \cdot T}{2} \cdot P_{a_1,b_1}[\Event]\]
\[ R(T; \mathcal{I}_{a_m, b_m}) \ge \frac{\delta \cdot T}{2} \cdot P_{a_m,b_m}[\Event^c]\]
where $\Event^c$ is the complement of $\Event$. We apply Theorem \ref{thm:bhinequality} to see that:
\begin{align*}
    R(T; \mathcal{I}_{a_1, b_1}) + R(T; \mathcal{I}_{a_m, b_m}) &= \frac{\delta \cdot T}{2} \left(P_{a_1,b_1}[\Event] + P_{a_m,b_m}[\Event^c]\right) \\
    &\ge_{(1)} \frac{\delta \cdot T}{2} \cdot \frac{1}{2} \text{exp}\left(-KL(P_{a_1,b_1},P_{a_m,b_m}) \right) \\
   &\ge_{(2)} \frac{\delta \cdot T}{2} \cdot \frac{1}{2} \text{exp}\left(- \mathbb{E}_{P_{a_1, b_1}}[n_T(a_m, b_m)] \cdot (2 \delta)^2 \right) \\
   &\ge_{(3)} \frac{\delta \cdot T}{4} \cdot \text{exp}\left(- \frac{4 \cdot \delta^2 \cdot T}{(|\mathcal{A} - 1) |\mathcal{B}|} \right). 
\end{align*}
where (1) applies Theorem \ref{thm:bhinequality} and (2) applies Corollary \ref{cor:divergencedecompositiongaussian}, and (3) applies the fact that $\NumPulls{a_m}{b_m}{T} \le \frac{T}{(|\mathcal{A}| - 1) \cdot |\mathcal{B}|}$. If we set $\delta = \Theta(\sqrt{\frac{|\mathcal{A} - 1) |\mathcal{B}|}{T}})$, then we obtain a bound of $\Theta(\sqrt{T \cdot (|\mathcal{A}| - 1) \cdot |\mathcal{B}|})$. Since expected regret is nonnegative for these instances (see discussion above), this implies that either $R(T; \mathcal{I}_{a_1, b_1})  = \Omega(\sqrt{T \cdot (|\mathcal{A}| - 1) \cdot |\mathcal{B}|})$ or $R(T; \mathcal{I}_{a_m, b_m})  = \Omega(\sqrt{T \cdot (|\mathcal{A}| - 1) \cdot |\mathcal{B}|})$ as desired.

\end{proof}

\subsection{Proof of Theorem \ref{thrm:worstlinear}}\label{appendix:worstlinear}

\worstlinear*
\begin{proof}
 It suffices to prove this lower bound in a \textit{centralized} environment where a single learner can choose action pairs $(a,b)$ and observes rewards for both players (Lemma \ref{lemma:centralizedimpliesdecentralized}). We construct a pair of instances $\Instance$ and $\tilde{\Instance}$ such that at least one of the players incurs linear regret on at least one of the instances. In particular, we take $\Instance$ and $\tilde{\Instance}$ to be the instances depicted in Table \ref{tab:main_1} with $\delta = O(1/\sqrt{T})$ (reproduced  here for convenience). 

 \begin{table*}[ht!]
\centering
\begin{subtable}[b]{0.45\linewidth}
\centering
\begin{tabular}{|c|c|c|}
\hline
      & $b_1$                  & $b_2$                  \\ \hline
$a_1$ & $(0.6, \delta)$ & $(0.2, \textbf{0})$             \\ \hline
$a_2$ & $(0.5, 0.6)$       & $(0.4, 0.4)$ \\ \hline
\end{tabular}
\caption{Mean rewards $(\MR{1}{a}{b}, \MR{2}{a}{b})$ for $\Instance$}
\end{subtable}
\hfill
\begin{subtable}[b]{0.45\linewidth}
\centering
\begin{tabular}{|c|c|c|}
\hline
      & $b_1$                  & $b_2$                  \\ \hline
$a_1$ & $(0.6, \delta)$ & $(0.2, $\boldmath{$2\delta$}$)$             \\ \hline
$a_2$ & $(0.5, 0.6)$       & $(0.4, 0.4)$ \\ \hline
\end{tabular}
\caption{Mean rewards $(\MRTilde{1}{a}{b}, \MRTilde{2}{a}{b})$  for $\tilde{\Instance}$}
\end{subtable}
\end{table*}

We first compute the benchmarks on these two instances. On instance $\Instance$, it holds that  $(a^*, b^*) = (a_1, b_1)$, $\benchmark{1} = 0.6$ and $\benchmark{2} = \delta \ge 0$. On the other hand, on instance $\tilde{\Instance}$, it holds that $(a^*, b^*) =  (a_2, b_1)$, $\benchmark{1} = 0.5$, and $\benchmark{2} = 0.6$.
It is easy to see that $R_1(T; \mathcal{I})$ and $R_2(T; \tilde{\mathcal{I}})$ are always \textit{nonnegative}.  

Fix any $\ALG$ for the centralized environment. Let $P$ (resp. $\tilde{P}$) denote the probability measure over canonical centralized bandit model when $\ALG$ is applied to the instance $\Instance$ (resp. $\tilde{\Instance}$) (see \Cref{appendix:tools}).  We will show that the regret is high on at least one of the instances $\Instance$ and $\tilde{\Instance}$. 

Now let $\NumPulls{a}{b}{T} = \sum_{t=1}^T \mathbbm{1}[a_t = a, b_t = b]$ be the random variable denoting the number of times that $(a,b)$ is pulled, and let $\Event$ denote the event that $\NumPulls{a_1}{b_1}{T} \le T/2$ (i.e., the arm $(a_1, b_1)$ is pulled less than $T/2$ times). It is easy to see that the regret satisfies:
\[ R_1(T; \Instance) \ge \frac{0.1 \cdot T}{2} \cdot P[\Event]\]
\[ R_2(T; \tilde{\Instance}) \ge \frac{(0.6 - \delta) \cdot T}{2} \cdot \tilde{P}[\Event^c]\]
where $\Event^c$ is the complement of $\Event$. We apply Theorem \ref{thm:bhinequality} to see that:
\begin{align*}
    R_1(T; \Instance) + R_2(T; \tilde{\Instance}) &= 
    \frac{0.1 \cdot T}{2} \cdot P[\Event] + \frac{(0.6 - \delta) \cdot T}{2} \cdot \tilde{P}[\Event^c] \\
    &\ge \frac{0.1 \cdot T}{2} \cdot \left( P[\Event]  + \tilde{P}[\Event^c] \right)\\
    &\ge_{(1)}  \frac{0.1 \cdot T}{2} \cdot \frac{1}{2} \text{exp}\left(-KL(P,\tilde{P}) \right) \\
   &\ge_{(2)} \frac{0.1 \cdot T}{2} \cdot \frac{1}{2} \text{exp}\left(- \mathbb{E}_{P}[\NumPulls{a_1}{b_2}{T}] \cdot \frac{(2 \cd \delta)^2}{2} \right) \\
   &\ge_{(3)} \frac{0.1 \cdot T}{4}\cdot \text{exp}\left(- 2 \cdot \delta^2 \cdot T\right). 
\end{align*}
where (1) applies Theorem \ref{thm:bhinequality} and (2) applies Corollary \ref{cor:divergencedecompositiongaussian}, and (3) uses the fact that $\NumPulls{a_1}{b_2}{T} \le T$. If we take $\delta = O(T^{-1/2})$, then we obtain a bound of $\Omega(T)$. Since these expected regrets are always nonnegative (see discussion above), this implies that either $R_1(T; \Instance)  = \Omega(T)$ or $R_2(T; \tilde{\Instance}) = \Omega(T)$ as desired. 
\end{proof}

\subsection{Proof of Theorem \ref{thrm:dlower}}\label{appendix:dlower}

\dlower*
\begin{proof}

\begin{table}[]
\centering 
\begin{tabular}{|c|c|c|c|c|}
\hline
         & $b_1$                  & $\ldots$              & $b'$                  & $\ldots$              \\ \hline
$a_1$    & $(0.5, 3 \cd \delta)$  & $(0.5, 3 \cd \delta)$ & $(0.5, 3 \cd \delta)$ & $(0.5, 3 \cd \delta)$ \\ \hline
$\vdots$ & $(0.5+\delta, \delta)$ & $(0, 0)$              & *                     & $(0, 0)$              \\ \hline
$\vdots$ & $(0.5+\delta, \delta)$ & $(0, 0)$              & *                     & $(0, 0)$              \\ \hline
$\vdots$ & $(0.5+\delta, \delta)$ & $(0, 0)$              & *                     & $(0, 0)$              \\ \hline
\end{tabular}
\caption{Hard instance for Theorem \ref{thrm:dlower}, where $*$ is equal to $(0, 0)$ for instance $\mathcal{I}_{a_1, b_1}$, and $(0, 2\delta)$ otherwise. Note that this example is structurally similar to the illustrative example in Table \ref{tab:main_4}, but with $\abs{\mathcal{A}}, \abs{\mathcal{B}}\geq 2$. }
\label{tab:dlower}
\end{table}

Fix $\mathcal{A}$ and $\mathcal{B}$ such that $|\mathcal{A}| \ge 2$ and $|\mathcal{B}| \ge 2$. 

 It suffices to prove this lower bound in a \textit{centralized} environment where a single learner can choose action pairs $(a,b)$ and observes rewards for both players (Lemma \ref{lemma:centralizedimpliesdecentralized}).
We define a family of instances in the centralized game and evaluate the self-tolerant benchmarks on this family of instances. Arbitrarily pick some $(a_1, b_1) \in \mathcal{A} \times \mathcal{B}$ to be the ``base'' action. Let $\mathcal{F}_{\delta, \mathcal{A}, \mathcal{B}}$ be the family of $|\mathcal{B}|$ instances of the form $(\mathcal{A}, \mathcal{B}, v_1, v_2)$ for varying settings of $v$, where we index the instances by $\mathcal{B}$. 
The utility functions for the instance $\mathcal{I}_{b'}$ are equal to terms below (illustrated in Table \ref{tab:dlower}): 
\[
\MR{1}{a}{b} = 
\begin{cases}
  0.5 & \text{ if } a = a_1  \\
 0.5 + \delta & \text{ if } a \neq a_1, b = b_1  \\
 0 & \text{ if } a \neq a_1, b \neq b_1.
\end{cases}
\]
\[
\MR{2}{a}{b} = 
\begin{cases}
 3 \delta & \text{ if } a = a_1  \\
 \delta & \text{ if } a \neq a_1, b = b_1  \\
 2 \delta & \text{ if } b = b', a \neq a_1, b \neq b_1 \\
 0 & \text{ if } b \neq b', a \neq a_1, b \neq b_1 
\end{cases}
\]

We claim that the $\benchmarkrelaxedstronger{1} = 0.5 + \delta$  and  $\benchmarkrelaxedstronger{2} = \delta$ for the instance $\mathcal{I}_{(a_1, b_1)}$ and $\benchmarkrelaxedstronger{1} = 0.5$ and $\benchmarkrelaxedstronger{2} = 3 \delta$ for the instances $\mathcal{I}_{(a', b')}$ where $(a', b') \neq (a_1, b_1)$. 
\begin{itemize}
    \item \textit{Instance $\mathcal{I}_{b_1}$:} If $\epsilon < \delta$, it holds that $\mathcal{B}_{\epsilon}(a) = \left\{b_1\right\}$ for $a \neq a_1$ and $\mathcal{A}_{\epsilon} = \mathcal{A} \setminus \left\{a_1 \right\}$. If $\epsilon \ge \delta$, then it holds that $\mathcal{B}_{\epsilon}(a) = \mathcal{B}$ and $\mathcal{A}_{\epsilon} = \mathcal{A}$. Altogether, this means that $\benchmarkrelaxedstronger{1} = 0.5 + \delta$ and $\benchmarkrelaxedstronger{2} = \delta$.
    \item \textit{Instances $\mathcal{I}_{b'}$ where $b' \neq b_1$:} If $\epsilon < \delta$, it holds that $\mathcal{B}_{\epsilon}(a) = \left\{b' \right\}$ for $a \neq a_1$ and $\mathcal{A}_{\epsilon} = \left\{a_1 \right\}$. If $\delta \le \epsilon < 2 \delta$, then it holds that $\mathcal{B}_{\epsilon}(a_1) = \mathcal{B}$ and $\mathcal{B}_{\epsilon}(a) = \{b', b_1\}$ for $a \neq a_1$, and $\mathcal{A}_{\epsilon} = \mathcal{A}$. If $\epsilon \ge 2 \delta$, then it holds that $\mathcal{B}_{\epsilon}(a) = \mathcal{B}$ and $\mathcal{A}_{\epsilon} = \mathcal{A}$. Altogether, this means that $\benchmarkrelaxedstronger{1} = 0.5$ and $\benchmarkrelaxedstronger{2} = 3 \delta$. 
\end{itemize}
It is easy to see that the regret $R_1(T; \mathcal{I}_{b_1})$ and the regret $R_2(T; \mathcal{I}_{b})$ for $b \neq b_1$ are always nonnegative. 

Fix an $\ALG$ be an algorithm for the centralized environment. For each $b \in \mathcal{B}$, let $P_{b}$ denote the probability measure over canonical centralized bandit model when $\ALG$ is applied to the instance $\mathcal{I}_{b}$ (see \Cref{appendix:tools}). 
Let $\NumPulls{a}{b}{T} = \sum_{t=1}^T \mathbbm{1}[a_t = a, b_t = b]$ be the random variable denoting the number of times that $(a,b)$ is pulled. We define:  
\[b_m := \argmin_{b \in \mathcal{B} \mid b \neq b_1} \mathbb{E}_{P_{b_1}} \left[\sum_{a \neq a_1} \NumPulls{a}{b}{T}\right]\]
to be the arm $b$ such that the set of arms $(a', b)$ for $a' \neq a_1$ is pulled the minimum number of times in expectation over $P_{b_1}$ (i.e., the expectation when $\ALG$ is applied to the instance $\mathcal{I}_{b_1}$). This means that 
\[\sum_{b \neq b_1} \sum_{a \neq a_1} \mathbb{E}_{P_{b_1}}[\NumPulls{a}{b}{T}] \ge (|\mathcal{B}| - 1) \sum_{a \neq a_1} \mathbb{E}_{P_{b_1}}[\NumPulls{a}{b_m}{T}].\]
We will construct $\delta$ such that the regret is high on at least one of the instances $\mathcal{I}_{b_1}$ and $\mathcal{I}_{b_m}$.

Now, let $\Event$ denote the event that $\sum_{a \neq a_1} \NumPulls{a}{b_1}{T} \le T/2$ (i.e., arms of the form $(a', b_1)$ for $a' \neq a$ are pulled less than $T/2$ times). It is easy to see that the regret satisfies:
\[ R_1(T; \mathcal{I}_{b_1}) \ge \frac{\delta \cdot T}{2} \cdot P_{b_1}[E]\]
\[ R_2(T; \mathcal{I}_{b_m}) \ge \frac{2 \cdot \delta \cdot T}{2} \cdot P_{b_m}[E^c]\]
\[R_1(T; \mathcal{I}_{b_1}) \ge (0.5 + \delta) \cdot \mathbb{E}\left[\sum_{a \neq a_1, b \neq b_1} \NumPulls{a}{b}{T} \right] \ge 0.5 \cdot \mathbb{E}\left[\sum_{a \neq a_1, b \neq b_1} \NumPulls{a}{b}{T}\right]. \]
where $\Event^c$ is the complement of $\Event$. We apply Theorem \ref{thm:bhinequality} to see that:
\begin{align*}
    &2 \cdot R_1(T; \mathcal{I}_{b_1}) + R_2(T; \mathcal{I}_{b_m}) \\
    &= \frac{\delta \cdot T}{2} \cdot P_{b_1}[E] + \frac{2 \cdot \delta \cdot T}{2} \cdot P_{b_m}[E^c] + 0.5 \cdot \mathbb{E}\left[\sum_{a \neq a_1, b \neq b_1} \NumPulls{a}{b}{T} \right] \\
    &\ge \frac{\delta \cdot T}{2} \cdot \left( P_{b_1}[E] +  P_{b_m}[E^c]\right) + 0.5 \cdot \mathbb{E}\left[\sum_{a \neq a_1, b \neq b_1} \NumPulls{a}{b}{T} \right] \\
    &\ge_{(1)} \frac{\delta \cdot T}{2} \text{exp}\left(-KL(P_{b_1},P_{b_m}) \right) + + 0.5 \cdot \mathbb{E}\left[\sum_{a \neq a_1, b \neq b_1} \NumPulls{a}{b}{T} \right] \\
   &\ge_{(2)} \frac{\delta \cdot T}{2} \cdot \frac{1}{2} \text{exp}\left(- \mathbb{E}_{P_{b_1}}\left[\sum_{a \neq a_1} \NumPulls{a}{b_m}{T}\right] \cdot \frac{(2 \delta)^2}{2} \right) + 0.5 \cdot \mathbb{E}\left[\sum_{a \neq a_1, b \neq b_1} \NumPulls{a}{b}{T} \right] \\ 
   &\ge_{(3)} \frac{\delta \cdot T}{2} \cdot \frac{1}{2} \text{exp}\left(- \mathbb{E}_{P_{b_1}}\left[\sum_{a \neq a_1} \NumPulls{a}{b_m}{T} \right] \cdot \frac{(2 \delta)^2}{2} \right) + 0.5 (|\mathcal{B}| - 1) \cdot \mathbb{E}\left[\sum_{b \neq b_1} \NumPulls{a}{b_m}{T} \right] 
\end{align*}
where (1) applies Theorem \ref{thm:bhinequality} and (2) applies Corollary \ref{cor:divergencedecompositiongaussian} and where (3) uses the fact that $\sum_{b \neq b_1} \sum_{a \neq a_1} \mathbb{E}_P[\NumPulls{a}{b}{T}] \ge (|\mathcal{B}| - 1) \sum_{a \neq a_1} \mathbb{E}_P[\NumPulls{a}{b_m}{T}]$. 

We claim that the expression is $\Omega(T^{2/3} (|\mathcal{B}| - 1)^{1/3})$. 
We split into two cases based on the value of $\mathbb{E}\left[\sum_{a \neq a_1} \NumPulls{a}{b_m}{T}\right]$: $\mathbb{E}\left[\sum_{a \neq a_1} \NumPulls{a}{b_m}{T}\right] \ge \Theta(T^{2/3} (|\mathcal{B}| - 1)^{-2/3})$ and $\mathbb{E}\left[\sum_{a \neq a_1} \NumPulls{a}{b_m}{T}\right] \le \Theta(T^{2/3} (|\mathcal{B}| - 1)^{-2/3})$. 
\begin{enumerate}
    \item \textit{Case 1:  $\mathbb{E}\left[\sum_{a \neq a_1} \NumPulls{a}{b_m}{T}\right] \ge \Theta(T^{2/3} (|\mathcal{B}| - 1)^{-2/3})$.} In this case, we see that $0.5 (|\mathcal{B}| - 1) \cdot \mathbb{E}\left[\sum_{b \neq b_1} \NumPulls{a}{b_m}{T}\right] = \Omega(T^{2/3} (|\mathcal{B}| - 1)^{1/3})$. 
    \item \textit{Case 2:  $\mathbb{E}\left[\sum_{a \neq a_1} \NumPulls{a}{b_m}{T}\right] \le \Theta(T^{2/3} (|\mathcal{B}| - 1)^{-2/3})$.} In this case, we can write:
    \[\frac{\delta \cdot T}{2} \cdot \frac{1}{2} \text{exp}\left(- \mathbb{E}_{P_{b_1}}\left[\sum_{a \neq a_1} \NumPulls{a}{b_m}{T}\right] \cdot \frac{(2 \delta)^2}{2} \right) \ge \frac{\delta \cdot T}{2} \cdot \frac{1}{2} \text{exp}\left(- \Theta\left( T^{2/3} (|\mathcal{B}| - 1)^{-2/3} \cdot \delta^2  \right)\right). \]
    In this case, we set $\delta = \Theta(T^{-1/3} (|\mathcal{B}| - 1)^{1/3})$ and the expression becomes $\Omega(T^{2/3} (|\mathcal{B}| - 1)^{1/3})$.  
\end{enumerate}

This proves that $ 2 \cdot R_1(T; \mathcal{I}_{b_1}) + R_2(T; \mathcal{I}_{b_m}) = \Omega(T^{2/3} (|\mathcal{B}| - 1)^{1/3})$.

Since expected regret is nonnegative for these instances (see discussion above), this implies that either $R_1(T; \mathcal{I}_{b_1})  = \Omega(T^{2/3} (|\mathcal{B}| - 1)^{1/3})$ or $R_2(T; \mathcal{I}_{b_m})  = \Omega(T^{2/3} (|\mathcal{B}| - 1)^{1/3})$ as desired. 
\end{proof}

\section{Proofs for Section \ref{sec:benchmarks}}\label{app:proofbenchmarks}

\subsection{Proof of Proposition \ref{prop:decentralizedlower}}\label{appendix:proofetcetcfails}

We prove Proposition \ref{prop:decentralizedlower}.

\decentralizedlower*

This proof holds for $\gamma < 0.1$ (the construction can be generalized to other constant $\gamma$ by adjusting the values of the mean rewards; we present this construction which builds on Table \ref{tab:main_2}). 

\begin{proof}

\begin{table}[ht]
\centering
\begin{tabular}{|c|c|c|}
\hline
      & $b_1$                  & $b_2$                  \\ \hline
$a_1$ & $(0.6, 0.4)$ & $(0.2, 0)$             \\ \hline
$a_2$ & $(0.5, 0.3)$       & $(0.4, 0.2)$ \\ \hline
\end{tabular}
\caption{A single instance, illustrating the $\gamma$-tolerant benchmark - variant of Table \ref{tab:main_2} with $\delta=0.1$ }
\label{tab:main_3}
\end{table}
We take $\Instance^*$ to be the instance $\Instance$ in Table \ref{tab:main_3} (equivalent to Table \ref{tab:main_2} with $\delta = 0.1$). 

The fact that $E' < E$ means that the leader's exploration phase takes place entirely during the follower's exploration phase. Moreover, since the leader's exploration parameter $E' \cdot |\mathcal{B}|$ is divisible by $|\mathcal{B}|$, for every arm $a \in \mathcal{A}$, the follower pulls every arm $b \in \mathcal{B}$ an equal number of times. 
Given that follower explores evenly between the two arms $b_1$ and $b_2$, the leader's expected average reward $\mathbb{E}[\hatMROneArg{1}{a_1}]$ from $a_1$ during the first $E' \cdot |\mathcal{B}|$ rounds is given by $(0.6 + 0.2) / 2 = 0.4$ and the leader's expected average reward $\mathbb{E}[\hatMROneArg{1}{a_2}]$ average reward from $a_2$ is given by $(0.5 +0.4) / 2 = 0.45$. 

The proofs boils down to analyzing the relationship between the distributions $\hatMROneArg{1}{a_1}$ and $\hatMROneArg{1}{a_2}$. Note that we allow $E, E'$ to be arbitrary, so we cannot use standard concentration bounds. Instead, we leverage the symmetry of the distribution of the  empirical mean $\hatMROneArg{1}{a_1}$ (this follows from the fact that $\hatMROneArg{1}{a_1} - \mathbb{E}[\hatMROneArg{1}{a_1}]$ is distributed as a Gaussian). This means that:
\[\mathbb{P}[\hatMROneArg{1}{a_1} > 0.4] = \mathbb{P}[\hatMROneArg{1}{a_1} <  0.4] = 0.5.\]
(The probability $\mathbb{P}[\hatMROneArg{1}{a_1} = \mathbb{E}[\hatMROneArg{1}{a_1}] = 0.4]$ is equal to $0$.)
Similarly, we see that:
\[\mathbb{P}[\hatMROneArg{1}{a_2} >  0.45] = \mathbb{P}[\hatMROneArg{1}{a_2} < 0.45] = 0.5.\]
Because the stochastic rewards have independent randomness, we know that with probability at least 0.25 we have $\hatMROneArg{1}{a_1} < 0.4$ and $\hatMROneArg{1}{a_2} > 0.45$. When this occurs, the leader commits to pulling arm $a_2$. 

Regardless of the follower's choice of action ($b_1$ or $b_2$) in the commit phase, this means that the follower obtains reward at most $0.3$ and the leader obtains reward at most $0.5$. However, recall that we found that the $\gamma$-tolerant benchmark (for $\gamma=0.1$) are $\benchmarkrelaxedstronger{1} = 0.6$ and $\benchmarkrelaxedstronger{2} = 0.4$. This leads to linear regret (at least $0.25 \cd 0.1 \cd T$) for both players, even with respect to the $\gamma$-tolerant benchmark. 
\end{proof}

\subsection{Proof of Theorem \ref{thm:ETCETC}}\label{appendix:proofetcetc}

\etcetc*

In this theorem, we will assume $\gamma = \omega\p{T^{-1/3}\abs{\mathcal{A}}^{1/3} \abs{\mathcal{B}}^{1/3} \cd (\log (T)^{1/3})}$ (see Section \ref{subsec:maxtol} for a discussion of $\gamma$) . 

\paragraph{Notation.} We will use the following notation in the proof. For $a \in \mathcal{A}$ and $b \in \mathcal{B}$, let $\hatMR{2}{a}{b}$ denote the empirical mean of observations that the follower has seen for arm $a$ during the first $E_2 \cdot |\mathcal{A}| \cdot |\mathcal{B}|$ time steps. For $a \in \mathcal{A}$, let $\hatMROneArg{1}{a}$ denote the empirical mean of observations that the leader has seen for arm $a$ during the first time steps $t \in [E_2 \cdot |\mathcal{B}| \cdot |\mathcal{A}| + 1, E_2 \cdot |\mathcal{B}| \cdot |\mathcal{A}| + E_1 \cdot |\mathcal{A}|]$. We denote by $\tilde{b}(a) = \argmax_{b \in \mathcal{B}} \hatMR{2}{a}{b}$ the arm that follower has committed to for rounds $t > E_2 \cdot |\mathcal{A}| \cdot |\mathcal{B}|$ onwards. We denote by $\tilde{a} = \argmax_{ a \in \mathcal{A}} \hatMROneArg{1}{a}$ the arm that the leader has committed to for rounds $t > E_1 \cdot |\mathcal{A}|$.

\paragraph{Clean event.} We define the clean event $\Event := \Event_L \cap \Event_F$ to be the intersection of a clean event $\Event_L$ for the leader and a clean event $\Event_F$ for the follower. Informally speaking, the clean event for the leader is the event that for all arms, the empirical mean reward $\hatMROneArg{1}{a}$ is close to the true reward $\MR{1}{a}{\tilde{b}(a)}$. The event $\Event_L$ is formalized as follows:
\[\forall{a \in \mathcal{A}}: |\hatMROneArg{1}{a} - \MR{1}{a}{\tilde{b}(a)} | \le \frac{10 \sqrt{\log T}}{\sqrt{E_1}}.\]
Similarly, informally speaking, the clean event for the follower is the event that for all arms, the empirical mean reward $\hatMR{2}{a}{b}$ is close to the true reward $\MR{2}{a}{b}$. The event $\Event_F$ is formalized as follows:
\[\forall{a \in \mathcal{A}, b \in \mathcal{B}}: |\hatMR{2}{a}{b} -\MR{2}{a}{b}| \le \frac{10 \sqrt{\log T}}{\sqrt{E_2}}.\]

We prove that the clean event occurs with high probability. 
\begin{lemma}
\label{lemma:cleanETCETC}
Assume the notation above. Let the follower run a separate instantiation of \\$\ExploreThenCommit(E_2, \mathcal{B})$ for every $a \in \mathcal{A}$, and let the leader run $\ExploreThenCommitThrowOut(E_1, E_2 \cdot |\mathcal{B}|, \mathcal{A})$. Then the clean event occurs with probability $\mathbb{P}[\Event] \ge 1 - (|\mathcal{A}| \cdot |\mathcal{B}| + |\mathcal{A}|) T^{-3}$.
\end{lemma}
\begin{proof}

First, we consider the follower's clean event $\Event_F$. For each $a \in \mathcal{A}, b \in \mathcal{B}$, the follower has seen $E_2$ samples, so by a Chernoff bound, we have that 
$$P\br{|\hatMR{2}{a}{b} - \MR{2}{a}{b}| \ge \frac{10 \sqrt{\log T}}{\sqrt{E_2}}} \leq  T^{-3}.$$
We union bound over $a \in \mathcal{A}$ and $b \in \mathcal{B}$.

Next, we consider the leader's clean event $G_L$. Note that $\hatMROneArg{1}{a}$ estimate is derived from rewards only after the follower has committed to a best response, so it is drawn from a distribution centered at $\MR{1}{a}{\tilde b(a)}$, with $E_1$ samples. Again by applying a Chernoff bound, we see that 
$$P\br{|\hatMROneArg{1}{a} - \MR{1}{a}{\tilde b(a)}| \ge \frac{10 \sqrt{\log T}}{\sqrt{E_1}}} \leq  T^{-3}.$$
We union bound over $a \in \mathcal{A}$.

Finally, we apply another union bound which leads 
$\mathbb{P}[\Event] \ge 1- (\abs{\mathcal{A}} \cd \abs{\mathcal{B}} + \abs{\mathcal{A}}) \cd T^{-3})$. 
\end{proof}

We also prove the following lower bounds on the leader's utility and follower's utility from the actions $\tilde{a}$ and $\tilde{b}(\tilde{a})$ that they commit to. 
\begin{lemma}
\label{lemma:mainlemmaETCETC}
Assume the notation above. Let the follower run a separate instantiation of \\ $\ExploreThenCommit(E_2, \mathcal{B})$ for every $a \in \mathcal{A}$, and let the leader run $\ExploreThenCommitThrowOut(E_1, E_2 \cdot |\mathcal{B}|, \mathcal{A})$. Suppose that the clean event $\Event$ holds. Then, for some $\epsilon^* = \Theta\left(\max\left(\frac{\sqrt{\log T}}{\sqrt{E_1}}, \frac{\sqrt{\log T}}{\sqrt{E_2}} \right)\right)$, it holds that:
\[\MR{1}{\tilde{a}}{\tilde{b}(\tilde{a})} \ge \max_{a \in \mathcal{A}} \min_{b \in \mathcal{B}_{\epsilon^*}(a)} \MR{1}{a}{b} - \epsilon^*\]
and that
\[\MR{2}{\tilde{a}}{\tilde{b}(\tilde{a})}\ge \min_{a \in \mathcal{A}_{\epsilon^*}} \max_{b \in \mathcal{B}} \MR{2}{a}{b} - \epsilon^*.\]    
\end{lemma}

\begin{proof}[Proof of Lemma \ref{lemma:mainlemmaETCETC}]

We assume that the clean event $\Event$ holds. We take $\epsilon^* = \Theta\left((|\mathcal{A}| \cdot |\mathcal{B}| \cdot (\log T))^{1/3} \cdot T^{-1/3}\right)$ with sufficiently high implicit constant. 

First, we show that the follower chooses a near-optimal action for every $a \in \mathcal{A}$: that is, $\MR{2}{a}{\tilde{b}(a)} \ge \max_{b \in \mathcal{B}} \MR{2}{a}{b} - \epsilon^*$.  Since $\Event_F$ holds, for every $a \in \mathcal{A}$ and $b \in \mathcal{B}$, we know that $|\hatMR{2}{a}{b} - \MR{2}{a}{b}| \le \frac{10 \sqrt{\log T}}{\sqrt{E_2}}$. Based on our setting of $E_2$ and because $\tilde{b}(a) = \argmax_{ b \in \mathcal{B}} \hat{v}^2_{a,b}$, it holds that: 
\[ \MR{2}{a}{\tilde{b}(a)} \ge \left(\max_{b \in \mathcal{B}} \MR{2}{a}{b}\right) - \frac{20 \sqrt{\log T}}{\sqrt{E_2}} \ge \left(\max_{b \in \mathcal{B}} \MR{2}{a}{b}\right) - \epsilon^*, \]
as desired. 

Next, we show that the leader chooses a near-optimal action: that is, $\MR{1}{\tilde{a}}{\tilde{b}(\tilde{a})} \ge \max_{a \in \mathcal{A}} \MR{1}{a}{ \tilde{b}(a)} - \epsilon^*$. 
Since $\Event_L$ holds, we know that $|\hatMROneArg{1}{a}  - \MR{1}{a}{\tilde{b}(a)}| \le \frac{10 \sqrt{\log T}}{\sqrt{E_1}}$. Based on our setting of $E_2$ and because $\tilde{a} = \argmax_{a \in \mathcal{A}} \hatMROneArg{1}{a}$, it holds that: 
\[ \MR{1}{\tilde{a}}{\tilde{b}(\tilde{a})} \ge \left(\max_{a \in \mathcal{A}} \MR{1}{a}{\tilde{b}(a)}\right) - \frac{20 \sqrt{\log T}}{\sqrt{E_1}} \ge \left(\max_{a \in \mathcal{A}} \MR{1}{a}{\tilde{b}(a)}\right) - \epsilon^*.\]
as desired.

To bound the leader's utility, observe that $\MR{2}{a}{\tilde{b}(a)} \ge \max_{b \in \mathcal{B}} \MR{2}{a}{b} - \epsilon^*$ implies that $b \in \mathcal{B}_{\epsilon^*}(a)$. This, coupled with the other bound, means that:
\[\MR{1}{\tilde{a}}{\tilde{b}(\tilde{a})} \ge \max_{a \in \mathcal{A}} \MR{1}{a}{\tilde{b}(a)} - \epsilon^* \ge \left(\max_{a \in \mathcal{A}} \min_{b \in \mathcal{B}_{\epsilon^*}(a)} \MR{1}{a}{b}\right) - \epsilon^*.\]

To bound the follower's utility, observe that $\MR{1}{\tilde{a}}{\tilde{b}(\tilde{a})} \ge \max_{a \in \mathcal{A}} \MR{1}{a}{\tilde{b}(a)} - \epsilon^*$ and $\MR{2}{a}{\tilde{b}(a)} \ge \max_{b \in \mathcal{B}} \MR{2}{a}{b} - \epsilon^*$ together imply that
\[\max_{b \in \mathcal{B}_{\epsilon^*}(a)} \MR{1}{\tilde{a}}{b}  \ge \MR{1}{\tilde{a}}{\tilde{b}(\tilde{a})}  \ge \max_{a \in \mathcal{A}} \MR{1}{a}{\tilde{b}(a)} - \epsilon^* \ge \left(\max_{a \in \mathcal{A}} \min_{b \in \mathcal{B}_{\epsilon^*}(a)} \MR{1}{a}{b}\right) - \epsilon^*,\]
which implies that $a \in \mathcal{A}_{\epsilon^*}$. This means that 
\[ \MR{2}{\tilde{a}}{\tilde{b}(\tilde{a})} \ge \left(\max_{b \in \mathcal{B}} \MR{2}{\tilde{a}}{b}\right) - \epsilon^* \ge \min_{a \in \mathcal{A}_{\epsilon^*}} \max_{b \in \mathcal{B}} \MR{2}{a}{b} - \epsilon^*.\] 

\end{proof}

We now prove Theorem \ref{thm:ETCETC}. 
\begin{proof}[Proof of Theorem \ref{thm:ETCETC}]

Assume that the clean event $\Event$ holds. This occurs with probability at least $1 - (|\mathcal{A}| \cdot |\mathcal{B}| + |\mathcal{A}|) T^{-3}$ (Lemma \ref{lemma:cleanETCETC}), so the clean event not occuring counts negligibly towards regret.

First, we consider the first $E_2 \cdot |\mathcal{B}| \cdot |\mathcal{A}| + E_1 \cdot |\mathcal{A}|$ time steps. Each time step results in $O(1)$ regret for both players. Based on the settings of $E_1$ and $E_2$, these phases contribute a regret of:
\[E_2 \cdot |\mathcal{B}| \cdot |\mathcal{A}| + E_1 \cdot |\mathcal{A}| = O \left(|\mathcal{A}|^{1/3} \cdot |\mathcal{B}|^{1/3} \cdot (\log T)^{1/3} \cdot T^{2/3} \right). \]

We focus on $t > E_2 \cdot |\mathcal{B}| \cdot |\mathcal{A}| + E_1 \cdot |\mathcal{A}|$ for the remainder of the analysis. Our main ingredient is Lemma \ref{lemma:mainlemmaETCETC}.
Note that $\epsilon^* = \Theta\left(\max\left(\frac{\sqrt{\log T}}{\sqrt{E_1}}, \frac{\sqrt{\log T}}{\sqrt{E_2}} \right)\right) = \Theta\left((|\mathcal{A}| \cdot |\mathcal{B}| \cdot (\log T))^{1/3} \cdot T^{-1/3}\right) $ based on the settings of $E_1$ and $E_2$. The regret of the leader can be bounded as:
\begin{align*}
  &\benchmarkrelaxedstronger{1} \cdot (T - E_2 \cdot |\mathcal{B}| \cdot |\mathcal{A}| - E_1 \cdot |\mathcal{A}|) - \sum_{t > E_2 \cdot |\mathcal{B}| \cdot |\mathcal{A}| + E_1 \cdot |\mathcal{A}|} \MR{1}{a_t}{b_t}  \\
  &\le (T - E_2 \cdot |\mathcal{B}| \cdot |\mathcal{A}| - E_1 \cdot |\mathcal{A}|) \cdot \left(\max_{a \in \mathcal{A}} \min_{b \in \mathcal{B}_{\epsilon^*}(a)} \MR{1}{a}{b} + \epsilon^*\right) - \sum_{t > E_2 \cdot |\mathcal{B}| \cdot |\mathcal{A}| + E_1 \cdot |\mathcal{A}|} \MR{1}{\tilde{a}}{\tilde{b}(\tilde{a})} \\
  &= (T - E_2 \cdot |\mathcal{B}| \cdot |\mathcal{A}| - E_1 \cdot |\mathcal{A}|) \cdot \epsilon^* + (T - E_2 \cdot |\mathcal{B}| \cdot |\mathcal{A}| - E_1 \cdot |\mathcal{A}|) \left(\max_{a \in \mathcal{A}} \min_{b \in \mathcal{B}_{\epsilon^*}(a)} \MR{1}{a}{b} - \MR{1}{\tilde{a}}{\tilde{b}(\tilde{a})} \right) \\
  &\le_{(A)} 2 \cdot T \cdot \epsilon^* \\
  &\le O\left(T^{2/3} (\log T)^{1/3} |\mathcal{A}|^{1/3} |\mathcal{B}|^{1/3} \right). 
\end{align*}
where (A) follows from Lemma \ref{lemma:mainlemmaETCETC}. The regret of the follower can similarly be bounded as:
\begin{align*}
  &\benchmarkrelaxedstronger{1} \cdot (T - E_2 \cdot |\mathcal{B}| \cdot |\mathcal{A}| - E_1 \cdot |\mathcal{A}|) - \sum_{t > E_2 \cdot |\mathcal{B}| \cdot |\mathcal{A}| + E_1 \cdot |\mathcal{A}|} \MR{2}{a_t}{b_t} \\
  &\le (T - E_2 \cdot |\mathcal{B}| \cdot |\mathcal{A}| - E_1 \cdot |\mathcal{A}|) \cdot \left(\min_{a \in \mathcal{A}_{\epsilon^*}} \max_{b \in \mathcal{B}} \MR{2}{a}{b} + \epsilon^* \right) - \sum_{t > E_2 \cdot |\mathcal{B}| \cdot |\mathcal{A}| + E_1 \cdot |\mathcal{A}|} \MR{2}{\tilde{a}}{\tilde{b}(\tilde{a})} \\
  &= (T - E_2 \cdot |\mathcal{B}| \cdot |\mathcal{A}| - E_1 \cdot |\mathcal{A}|) \cdot \epsilon^* + (T - E_2 \cdot |\mathcal{B}| \cdot |\mathcal{A}| - E_1 \cdot |\mathcal{A}|) \left(\min_{a \in \mathcal{A}_{\epsilon^*}} \max_{b \in \mathcal{B}} \MR{2}{a}{b} - \MR{2}{\tilde{a}}{\tilde{b}(\tilde{a})} \right) \\
  &\le_{(B)} 2 \cdot T \cdot \epsilon^* \\
  &\le O\left(T^{2/3} (\log T)^{1/3} |\mathcal{A}|^{1/3} |\mathcal{B}|^{1/3} \right). 
\end{align*}
where (B) follows from Lemma \ref{lemma:mainlemmaETCETC}.
This proves the desired result.

\end{proof}

\subsection{Proof of Theorem \ref{thm:explorethenucb}}\label{appendix:proofexplorethenucb}

\explorethenucb*

We assume $\gamma = \omega\left(\mathcal{A}|^{1/3} |\mathcal{B}|^{1/3} (\log T)^{1/3} T^{-1/3} \right)$.

\paragraph{Notation.} We will use the following notation in the proof. 
Let $\epsilon^* = \max_{t > E} g(t, T, \mathcal{B})$. Let $\tilde{a} = \argmax_{a \in \mathcal{A}} \min_{\mathcal{B}_{\epsilon^*}(a)} \MR{1}{a}{b}$ be the optimal action for the leader if the follower can worst-case $\epsilon^*$-best-respond to any action.  Let $\hatMROneArgTimeStep{1}{a}{t}$ be the empirical mean specified in \texttt{ExplorethenUCB} at the beginning of time step $t$: this is the empirical mean of all observations that the leader has seen for arm $a$ prior to time step $t$ during the UCB phase (i.e., after time step $E \cdot |\mathcal{A}| + 1$ and prior to time step $t$). Moreover, for each arm $a \in \mathcal{A}$, let $S(a) = \left\{ t > E \cdot |\mathcal{A}| \mid a_{t} = a \right\}$ be the set of time steps where arm $a$ is pulled during the UCB phase, and let $\NumPullsTwoTimesOneArg{a}{E \cdot |\mathcal{A}|}{t} = |\left\{ E \cdot |\mathcal{A}| < t' < t \mid a_{t'} = a \right\}|$ be the number of times that $a$ is pulled during the UCB phase prior to time step $t'$.   

\paragraph{Clean event.} We define the clean event $\Event := \Event_L \cap \Event_F$ to be the intersection of a clean event $\Event_L$ for the leader and a clean event $\Event_F$ for the follower. Informally speaking, the clean event for the leader is the event that for all arms $a\in \mathcal{A}$ and for all time steps $t$, the empirical mean $\hatMROneArgTimeStep{1}{a}{t}$ is close to the true average of the mean rewards across actions taken by $b$ when the leader has chosen action $a$. The event $\Event_L$ is formalized as follows: 
\[\forall a \in \mathcal{A}, t \le T:  \left|\frac{1}{\NumPullsTwoTimesOneArg{a}{E \cdot |\mathcal{A}|}{t}} \sum_{E \cdot |\mathcal{A}| <t' < t \mid a_{t'} = a} \MR{1}{a_{t'}}{b_{t'}} - \hatMROneArgTimeStep{1}{a}{t} \right| \le \frac{10 \sqrt{\log T}}{\sqrt{\NumPullsTwoTimesOneArg{a}{E \cdot |\mathcal{A}|}{t}}}. \]
The clean event $\Event_F$ for the follower is the event that the follower picks an item within the $\epsilon^*$ best response set: $\forall t > E \cdot |\mathcal{A}|: 
b_t \in \mathcal{B}_{\epsilon^*}(a_t)$. 

We first prove that the clean event $\Event$ occurs with high probability.
\begin{lemma}
\label{lemma:cleaneventexplorethenucb}
Assume the notation above.  Let $\ALG_{2}$ be any algorithm with high-probability instantaneous regret $g$ where $g(t, T, \mathcal{B}) = O(E^{-1/2} |\mathcal{B}|^{1/2} (\log T)^{1/2})$ for $t > E$ and $g(t, T, \mathcal{B}) = 1$ for $t \le E$, and let $\ALG_1 = \ExploreThenUCB(E)$.  Then, the event $\Event$ occurs with high probability: $\mathbb{P}[\Event] \ge 1 - T^{-3} (|\mathcal{A}| + 1)$. 
\end{lemma}
\begin{proof}
We first show that $\mathbb{P}[\Event_F] \ge 1 - |\mathcal{A}| \cdot T^{-3}$. A sufficient condition for this event to hold is that: 
\[\forall t > E \cdot |\mathcal{A}|: \MR{2}{a_t}{b_t} \ge \max_{b \in \mathcal{B}} \MR{2}{a_t}{b} - \max_{t > E} g(t, T, \mathcal{B}).\]
Since the exploration phases pulls every arm $a \in \mathcal{A}$ a total of $E$ times, the high-probability instantaneous regret assumption guarantees that this event holds with probability at least $1 - |\mathcal{A}| \cdot T^{-3}$, as desired.

We next show that $\mathbb{P}[\Event_L] \ge 1 - T^{-3}$. This follows from a a Chernoff bound (and using the analogue of one of the canonical bandit models in \cite{L20}) combined with a union bound. 

The lemma follows from another union bound over $\Event_L$ and $\Event_F$. 
\end{proof}

Our main lemma provides, an upper bound on $\frac{1}{\NumPullsTwoTimesOneArg{a'}{E \cdot |\mathcal{A}|}{T} - 1} \sum_{ t \in S(a') \setminus \left\{\max(S(a'))\right\}} \MR{1}{a_t}{b_t}$, which is the average of the mean rewards obtained on  $a'$ across all time steps $t$ where $a'$ is pulled (except for the last round), for each arm $a' \in \mathcal{A}$. In particular, we upper bound this quantity by the worst-case optimal reward under $\epsilon$-best-responses by the follower ($\max_{a \in \mathcal{A}} \min_{b \in \mathcal{B}_{\epsilon^*}(a)} \MR{1}{a}{b}$) minus the twice the size of the confidence set of $a$. 
\begin{lemma}
\label{lemma:UCBbound}
Assume the notation above.  Suppose that the clean event $\Event$ holds. Then it holds that:
\[\frac{1}{\NumPullsTwoTimesOneArg{a'}{E \cdot |\mathcal{A}|}{T+1} - 1} \sum_{ t \in S(a') \setminus \left\{\max(S(a'))\right\}} \MR{1}{a_t}{b_t} \ge \max_{a \in \mathcal{A}} \min_{b \in \mathcal{B}_{\epsilon^*}(a)} \MR{1}{a}{b} - \frac{20 \sqrt{\log T}}{\sqrt{\NumPullsTwoTimesOneArg{a'}{E \cdot |\mathcal{A}|}{T+1}} - 1}. \]
\end{lemma}
\begin{proof}
We assume that the clean event $\Event = \Event_L \cap \Event_F$ holds. Note that $t^* = \max(S(a'))$ denotes the last time step during which $a'$ is chosen. Recalling that $\tilde{a} = \argmax_{a \in \mathcal{A}} \min_{\mathcal{B}_{\epsilon^*}(a)} \MR{1}{a}{b}$, let $S = S(\tilde{a}) \cap [E \cdot |\mathcal{A}| + 1, t^* - 1]$ be the set of time steps during the UCB phase prior to time step $t^*$ where arm $\tilde{a}$ is pulled. We see that at the beginning of time step $t^*$, it holds that:
\begin{align*}
\frac{1}{\NumPullsTwoTimesOneArg{a'}{E \cdot |\mathcal{A}|}{T+1} - 1} \sum_{ t \in S(a') \setminus \left\{t^* \right\}} \MR{1}{a_t}{b_t} 
&\ge_{(1)} \hatMROneArgTimeStep{1}{a'}{t^*} - \frac{10 \sqrt{\log T}}{\sqrt{\NumPullsTwoTimesOneArg{a'}{E \cdot |\mathcal{A}|}{T+1} - 1}} \\ 
&\ge \UCBOneArgTimeStep{1}{a'}{t^*} - \frac{20 \sqrt{\log T}}{\sqrt{\NumPullsTwoTimesOneArg{a'}{E \cdot |\mathcal{A}|}{T+1} - 1}} \\
 &\ge \UCBOneArgTimeStep{1}{\tilde{a}}{t^*} - \frac{20 \sqrt{\log T}}{\sqrt{\NumPullsTwoTimesOneArg{a'}{E \cdot |\mathcal{A}|}{T+1} - 1}} \\
 &= \hatMROneArgTimeStep{1}{\tilde{a}}{t^*} + \frac{10 \sqrt{\log T}}{\sqrt{|S|}} - \frac{20 \sqrt{\log T}}{\sqrt{\NumPullsTwoTimesOneArg{a'}{E \cdot |\mathcal{A}|}{T+1} - 1}} \\
 &\ge_{(2)} \frac{1}{|S|} \sum_{t \in S} \MR{1}{a_t}{b_t}  - \frac{20 \sqrt{\log T}}{\sqrt{\NumPullsTwoTimesOneArg{a'}{E \cdot |\mathcal{A}|}{T+1} - 1}}  \\
 &\ge_{(3)} \max_{a \in \mathcal{A}} \min_{b \in \mathcal{B}_{\epsilon^*}(a)} \MR{1}{a}{b} - \frac{20 \sqrt{\log T}}{\sqrt{\NumPullsTwoTimesOneArg{a'}{E \cdot |\mathcal{A}|}{T+1} - 1}}
\end{align*}
where (1) and (2) use the clean event $\Event_L$. Step (3) uses the clean event $\Event_F$ which guarantees that $b_t \in \mathcal{B}_{\epsilon^*}(a_t)$ for all $t$, which means that for any $t \in S$, it holds that: 
\[\MR{1}{a_t}{b_t} = \MR{1}{\tilde{a}}{b_t} \ge \min_{b \in \mathcal{B}_{\epsilon^*}(a)} \MR{1}{\tilde{a}}{b} =  \max_{a \in \mathcal{A}} \min_{b \in \mathcal{B}_{\epsilon^*}(a)} \MR{1}{a}{b} \]
as desired. 
\end{proof}

Now we are ready to prove Theorem \ref{thm:explorethenucb}. 

\begin{proof}[Proof of Theorem \ref{thm:explorethenucb}]
Assume that the clean event $\Event$ holds. This occurs with probability at least $1 - (1 + |\mathcal{A}|) T^{-3}$ (Lemma \ref{lemma:cleaneventexplorethenucb}), so the clean event not occurring counts negligibly towards regret.

The regret in the explore phase is bounded by $O(1)$ in each round, the total regret from that phase is $O(T^{2/3} |\mathcal{A}|^{1/3}  |\mathcal{B}|^{1/3} (\log T)^{1/3})$ for either player. 

The remainder of the analysis boils down to bounding the regret in the UCB phase. We separately analyze the regret of the leader and the follower. Observe that $\epsilon^* = \max_{t > E} g(t, T, \mathcal{B}) = O\left(\mathcal{A}|^{1/3} |\mathcal{B}|^{1/3} (\log T)^{1/3} T^{-1/3} \right)$ based on the assumption on the follower's algorithm.

\paragraph{Regret for the leader.} We bound the regret as:
\begin{align*}
&\benchmarkrelaxedstronger{1} \cdot (T-E \cd \abs{\mathcal{A}}) - \sum_{t=E \cdot |\mathcal{A}| + 1}^{T} \MR{1}{a_t}{b_t} \\
&\le \sum_{t=E \cdot |\mathcal{A}| + 1}^{T} \left(\epsilon^* + \max_{a \in \mathcal{A}} \min_{b \in \mathcal{B}_{\epsilon^*}(a)} \MR{1}{a}{b} - \MR{1}{a_t}{b_t} \right)  \\
&=  \sum_{a \in \mathcal{A}} \sum_{t \in S(a)}  \left(\epsilon^* + \max_{a \in \mathcal{A}} \min_{b \in \mathcal{B}_{\epsilon^*}(a)} \MR{1}{a}{b} -  \MR{1}{a_t}{b_t}  \right) \\
&\le |\mathcal{A}| + \sum_{a \in \mathcal{A}} \sum_{t \in S(a) \setminus \left\{\max(S(a)) \right\}} \left(\epsilon^* + \max_{a \in \mathcal{A}} \min_{b \in \mathcal{B}_{\epsilon^*}(a)} \MR{1}{a}{b} -  \MR{1}{a_t}{b_t} \right)\\
&\le |\mathcal{A}| + \underbrace{\epsilon^* \cdot T}_{(1)}\\
&+  \underbrace{\sum_{a \in \mathcal{A}} (\NumPullsTwoTimesOneArg{a}{E \cdot |\mathcal{A}|}{T+1} - 1) \cdot \left(\max_{a \in \mathcal{A}} \min_{b \in \mathcal{B}_{\epsilon^*}(a)} \MR{1}{a}{b} -  \frac{1}{\NumPullsTwoTimesOneArg{a}{E \cdot |\mathcal{A}|}{T+1} - 1} \sum_{t \in S(a) \setminus \left\{\max(S(a)) \right\}} \left( \MR{1}{a_t}{b_t} \right) \right)}_{(2)}
\end{align*}
The term $|\mathcal{A}|$ computes negligibly, 
term (1) is equal to $\Theta(\mathcal{A}|^{1/3} |\mathcal{B}|^{1/3} (\log T)^{1/3} T^{2/3})$, and term (2) can be bounded by:
\begin{align*}
  &\sum_{a \in \mathcal{A}} (\NumPullsTwoTimesOneArg{a}{E \cdot |\mathcal{A}|}{T+1} - 1) \cdot \left(\max_{a \in \mathcal{A}} \min_{b \in \mathcal{B}_{\epsilon^*}(a)} \MR{1}{a}{b} -  \frac{1}{\NumPullsTwoTimesOneArg{a}{E \cdot |\mathcal{A}|}{T+1} - 1} \sum_{t \in S(a) \setminus \left\{\max(S(a)) \right\}} \left( \MR{1}{a_t}{b_t} \right) \right) \\
  &\le  \sum_{a \in \mathcal{A}} (\NumPullsTwoTimesOneArg{a}{E \cdot |\mathcal{A}|}{T+1} - 1) \cdot \frac{20 \sqrt{\log T}} {\sqrt{\NumPullsTwoTimesOneArg{a}{E \cdot |\mathcal{A}|}{T+1}} - 1} \\
   &\le O\left(\sqrt{\log T} \cdot \sum_{a \in \mathcal{A}} \sqrt{\NumPullsTwoTimesOneArg{a}{E \cdot |\mathcal{A}|}{T+1} - 1} \right) \\
   &\le O\left(\sqrt{|\mathcal{A}| T \log T} \right),
\end{align*}
where the first inequality uses Lemma \ref{lemma:UCBbound} and the last inequality uses Jensen's inequality. 

\paragraph{Regret for the follower.} Note that $\cup_{a \in \mathcal{A}_{\epsilon^*}} S(a)$ denotes the set of time steps where an action in $\mathcal{A}_{\epsilon^*}$ is chosen. We bound the regret as:
\begin{align*}
   &\benchmarkrelaxedstronger{2} \cdot (T - E \cd \abs{\mathcal{A}}) - \sum_{t=E \cdot |\mathcal{A}| + 1}^{T} \MR{2}{a_t}{b_t}  \\
   &\le \underbrace{\left(\sum_{t=E \cdot |\mathcal{A}| + 1}^T \mathbbm{1}[t \not\in \cup_{a \in \mathcal{A}_{\epsilon^*}} S(a)]\right)}_{(1)} + \underbrace{\sum_{t \in \cup_{a \in \mathcal{A}_{\epsilon^*}} S(a)} \left(\min_{a \in \mathcal{A}_{\epsilon^*}} \max_{b \in \mathcal{B}} \MR{2}{a}{b} - \MR{2}{a_t}{b_t} \right)}_{(2)} + \underbrace{\epsilon^* \cdot |\cup_{a \in \mathcal{A}_{\epsilon^*}} S(a)|}_{(3)} 
\end{align*}
We first bound term (1), which can be rewritten as $\sum_{t=E \cdot |\mathcal{A}| + 1}^T \mathbbm{1}[t \not\in \cup_{a \in \mathcal{A}_{\epsilon^*}} S(a)] = \sum_{a \not\in \mathcal{A}_{\epsilon^*}} \NumPullsTwoTimesOneArg{a}{E \cdot |\mathcal{A}|}{T}$. This counts the number of times that arms outside of $\mathcal{A}_{\epsilon^*}$ are pulled during the UCB phase. The key intuition is when an arm $a_t \not\in \mathcal{A}_{\epsilon^*}$, it holds that: 
\[ \MR{1}{a_t}{b_t} \le \max_{b \in \mathcal{B}_{\epsilon^*}(a')} \MR{1}{a_t}{b} < \max_{a \in \mathcal{A}} \min_{b \in \mathcal{B}_{\epsilon^*}(a)} \MR{1}{a}{b}  - \epsilon^*,\]
where the first inequality uses the fact that $b_t \in \mathcal{B}_{\epsilon^*}(a_t)$ (which follows from the clean event $\Event_F$) and the second inequality uses the fact that $a_t \not\in \mathcal{A}_{\epsilon^*}$. This implies that for any $a' \not\in \mathcal{A}_{\epsilon^*}$, the average reward across all time steps (except for the last time step) where $a'$ is pulled satisfies: 
\[\frac{1}{\NumPullsTwoTimesOneArg{a'}{E \cdot |\mathcal{A}|}{T+1} - 1} \sum_{ t \in S(a') \setminus \left\{\max(S(a'))\right\}} \MR{1}{a_{t'}}{b_{t'}} < \max_{a \in \mathcal{A}} \min_{b \in \mathcal{B}_{\epsilon^*}(a)} \MR{1}{a}{b}  - \epsilon^*.  \]
However, by Lemma \ref{lemma:UCBbound}, we can also lower bound the average reward across all time steps (except for the last time step) where $a'$ is pulled in terms of $\NumPullsTwoTimesOneArg{a'}{E \cdot |\mathcal{A}|}{T+1}$ as follows: 
\[\frac{1}{\NumPullsTwoTimesOneArg{a'}{E \cdot |\mathcal{A}|}{T+1} - 1} \sum_{ t \in S(a') \setminus \left\{\max(S(a'))\right\}} \MR{1}{a_{t'}}{b_{t'}} \ge \max_{a \in \mathcal{A}} \min_{b \in \mathcal{B}_{\epsilon^*}(a)} \MR{1}{a}{b} - \frac{10 \sqrt{\log T}}{\sqrt{\NumPullsTwoTimesOneArg{a'}{E \cdot |\mathcal{A}|}{T+1} - 1}}. \]
Putting these two inequalities together, we see that:
\[\frac{10 \sqrt{\log T}}{\sqrt{\NumPullsTwoTimesOneArg{a'}{E \cdot |\mathcal{A}|}{T+1} - 1}} \ge \epsilon^*, \]
which bounds the number of times that $a'$ is pulled during the UCB phase as follows: 
\[\NumPullsTwoTimesOneArg{a'}{E \cdot |\mathcal{A}|}{T+1} \le \Theta\left(\frac{\log T}{(\epsilon^*)^2} \right) = \Theta\left((\log T)^{1/3} T^{2/3} |\mathcal{A}|^{-2/3} |\mathcal{B}|^{-2/3} \right). \]
This means that:
\[\sum_{t=E \cdot |\mathcal{A}| + 1}^T \mathbbm{1}[t \not\in \cup_{a \in \mathcal{A}_{\epsilon^*}} S(a)] = \sum_{a \not\in \mathcal{A}_{\epsilon}} \NumPullsTwoTimesOneArg{a}{E \cdot |\mathcal{A}|}{T+1} \le \Theta\left((\log T)^{1/3} T^{2/3} |\mathcal{A}|^{1/3} |\mathcal{B}|^{-2/3} \right) \]

Next, we bound term (2):
\begin{align*}
  \min_{a \in \mathcal{A}_{\epsilon^*}} \max_{b \in \mathcal{B}} \MR{2}{a}{b} - \mathbb{E}[\MR{2}{a_t}{b_t}] &\le \sum_{t \in \cup_{a \in \mathcal{A}_{\epsilon^*}} S(a)} \left(\max_{b \in \mathcal{B}} \MR{2}{a_t}{b} - \mathbb{E}[\MR{2}{a_t}{b_t}] \right) \le  |\cup_{a \in \mathcal{A}_{\epsilon^*}} S(a)| \cdot \epsilon^* \\
  &\le T \cdot \epsilon^* \\
  &= \Theta\left((\log T)^{1/3} T^{2/3} |\mathcal{A}|^{1/3} |\mathcal{B}|^{1/3} \right).  
\end{align*}

Finally, we bound term (3) as $\epsilon^* \cdot |S| \le T \cdot \epsilon^* = \Theta\left((\log T)^{1/3} T^{2/3} |\mathcal{A}|^{1/3} |\mathcal{B}|^{1/3} \right)$.

Putting this all together yields the desired bound.
\end{proof}

\section{Proofs for Section \ref{sec:relaxed}}\label{app:relaxed}

\subsection{Alignment and continuity discussion}\label{app:continous}

\begin{table}[]
\centering 
\begin{tabular}{|c|c|c|}
\hline
      & $b_1$        & $b_2$        \\ \hline
$a_1$ & $(1, 0)$     & $(1-x, y)$   \\ \hline
$a_2$ & $(1-2x, 2y)$ & $(1-3x, 3y)$ \\ \hline
\end{tabular}
\caption{Set $x, y \in (0, 1/3)$ to obtain an example where both players have completely inverted ordered preferences over outcomes, but for $x, y > \mathcal{O}(1/T)$ have bounded continuity. }
\label{tab:misalignedcanlearn}
\end{table}

We note that constant $L^*$ still allows for a rich space of disagreement on values. We will formalize our discussion on the distinction between requiring that the leader and the follower have the same relative ordering on every pair of $(a, b)$ outcomes (\textit{ordered alignment}) and that they agree on which pairs of outcomes are \textit{sufficiently different} (\textit{continuity}). In particular, our Lipschitz condition requires continuity, but still allows for arbitrarily misordered alignment. As an example, Table \ref{tab:misalignedcanlearn} gives an example where the leader and the follower have completely inverted preferences over every outcome, but have utility that is $\max\p{\frac{x}{y}, \frac{y}{x}}$ Lipschitz continuous.

\subsection{Proofs and examples for Section \ref{subsec:lipschitz}}\label{appendix:lipschitzproof}

In this section, we prove Theorem \ref{thrm:dpL}, restated below for convenience. 
\dpL*

\paragraph{Notation.} Let $\hatMROneArgTimeStep{1}{a}{t}$ be the empirical mean specified in \texttt{LipschitzUCB} at the beginning of time step $t$, which is the mean of the leader's stochastic rewards $\left\{\RR{1}{a_{t'}}{b_{t'}}{t'} \mid a_{t'} = a, 1 \le t' < t  \right\}$. We also define $\hatMRTimeStep{1}{a}{b}{t}$ to be the mean of the leader's stochastic rewards for the arm $(a,b)$ up through time step $t-1$ (the set given by $\left\{\RR{1}{a_{t'}}{b_{t'}}{t'} \mid a_{t'} = a, b_{t'} = b, 1 \le t' < t  \right\}$). Note that this quantity is not computable by the leader in a $\StrongSetting$, but we nonetheless find it convenient to consider in the analysis. Let $\NumPullsOneArg{a}{t} = \left|1 \le t' < t \mid a_t = a \right|$ be the number of times that $a$ has been chosen prior to time step $t$. Let $\NumPulls{a}{b}{t} = \left|1 \le t' < t \mid a_t = a, b_t = b \right|$ be the number of times that $(a,b)$ has been chosen prior to time step $t$. 
For each arm $a \in \mathcal{A}$, let $b^*(a) = \argmax_{b \in \mathcal{B}} \MR{2}{a}{b}$ be the follower's best response. 

\paragraph{Clean event.} We define the clean event $\Event = \Event_L \cap \Event_F$ to be the intersection of a clean event $\Event_L$ for the leader and a clean event $\Event_F$ for the follower. Informally speaking, the clean event for the leader is the event that for all pairs of arms, the empirical mean reward $\hatMRTimeStep{1}{a}{b}{t}$  is close to the true reward $\MR{1}{a}{b}$. The event $\Event_L$ is formalized as follows:
\[\forall a \in \mathcal{A}, t \le T: |\hatMRTimeStep{1}{a}{b}{t} - \MR{1}{a}{b}| \le \frac{10 \sqrt{\log T}}{\sqrt{\NumPullsOneArg{a}{t}}}.   \]
Informally speaking, the clean event for the follower is the event that the follower satisfies high-probability anytime regret bounds. The event $\Event_F$ is formalized as follows:
\[\forall a\in \mathcal{A}, t \le T:  \sum_{1 \le t' < t \mid a_t = a} (\MR{2}{a}{b^*(a)} - \MR{2}{a_t}{b_t}) \le C' \sqrt{|\mathcal{B}| \NumPullsOneArg{a}{t} \log T} \]

We first prove that the clean event $\Event$ occurs with high probability.
\begin{lemma}
\label{lemma:cleanlipschitzucb}
Assume the setup of Theorem \ref{thrm:dpL} and the notation above. Then the clean event occurs with high probability: $\mathbb{P}[\Event] \ge 1 - T^{-3} (|\mathcal{A}| + 1)$.
\end{lemma}
\begin{proof}
We union bound over $\Event_L$ and $\Event_F$. The analysis for $\Event_F$ follows from the high-probability anytime regret bound assumption. The analysis for $\Event_L$ follows from a Chernoff bound (and using the analogue of one of the canonical bandit models in \cite{L20}) combined with a union bound. 
\end{proof}

The following lemma guarantees, for each arm $a \in \mathcal{A}$, that the empirical mean $\hatMROneArgTimeStep{1}{a}{t}$ is close to the mean reward if the follower were to best-respond $\max_{b \in \mathcal{B}} \MR{1}{a}{b}$. Conceptually speaking, this lemma guarantees that the confidence sets for the leader are always ``correct''. 
\begin{lemma}
\label{lemma:UCBconfidenceboundscorrect}
Assume the setup of Theorem \ref{thrm:dpL} and the notation above. Suppose that the clean event $\Event$ holds. Then for any $t \le T$ and $a \in \mathcal{A}$, it holds that: 
\[\left| \hatMROneArgTimeStep{1}{a}{t} - \MR{1}{a}{b^*(a)} \right| \le 
\frac{10 \sqrt{\abs{\mathcal{B}} \log T}}{\sqrt{\NumPullsOneArg{a}{t}}} + C' \cdot L \cdot \frac{\sqrt{|\mathcal{B}| \log T}}{\sqrt{\NumPullsOneArg{a}{t}}}.\] 
\end{lemma}
\begin{proof}
We observe that: 
\begin{align*}
\left| \hatMROneArgTimeStep{1}{a}{t} - \MR{1}{a}{b^*(a)} \right| &= \left|\left(\frac{1}{\NumPullsOneArg{a}{t}} \sum_{b \in \mathcal{B}} \NumPulls{a}{b}{t} \cdot \hatMR{1}{a}{b} \right) - \MR{1}{a}{b^*(a)}\right| \\
&= \left|\left(\frac{1}{\NumPullsOneArg{a}{t}} \sum_{b \in \mathcal{B}} \NumPulls{a}{b}{t} \cdot \hatMR{1}{a}{b} \right) - \frac{1}{\NumPullsOneArg{a}{t}} \left(\sum_{b \in \mathcal{B}} \NumPulls{a}{b}{t} \cdot \MR{1}{a}{b^*(a)}\right) \right| \\
 &\le \frac{1}{\NumPullsOneArg{a}{t}} \sum_{b \in \mathcal{B}} \NumPulls{a}{b}{t} \cdot \left|  \hatMR{1}{a}{b} - \MR{1}{a}{b^*(a)}\right| \\
 &\le \underbrace{\frac{1}{\NumPullsOneArg{a}{t}} \sum_{b \in \mathcal{B}} \NumPulls{a}{b}{t} \cdot \left|\hatMR{1}{a}{b} - \MR{1}{a}{b}\right|}_{(A)} + \underbrace{\frac{1}{\NumPullsOneArg{a}{t}} \sum_{b \in \mathcal{B}} \NumPulls{a}{b}{t} \cdot \left|\MR{1}{a}{b}  - \MR{1}{a}{b^*(a)}\right|}_{(B)}.
\end{align*}

First, we will bound term (A), which relates the error of the estimate of $\MR{1}{a}{b}$. We see that:
\begin{align*}
\frac{1}{\NumPullsOneArg{a}{t}} \sum_{b \in \mathcal{B}} \NumPulls{a}{b}{t} \cdot \left|\hatMR{1}{a}{b} -  \MR{1}{a}{b}\right| &\le_{(1)} \frac{1}{\NumPullsOneArg{a}{t}} \sum_{b \in \mathcal{B}} \NumPulls{a}{b}{t} \cdot \frac{10\sqrt{\log T}}{\sqrt{\NumPulls{a}{b}{t}}} \\
&= \frac{10 \sqrt{\log T}}{\NumPullsOneArg{a}{t}} \sum_{\sum_{b \in \mathcal{B}}} \sqrt{\NumPulls{a}{b}{t}} \\
&\le_{(2)} \frac{10 \sqrt{|\mathcal{B}| \log T}}{\sqrt{\NumPullsOneArg{a}{t}}}. 
\end{align*}
where (1) uses the clean event $\Event_L$ and (2) uses Jensen's inequality. 

Term (B) represents represents the difference in the leader's utility between the arm chosen by the follower and the follower's best-response. We can bound this as: 
\begin{align*}
\frac{1}{\NumPullsOneArg{a}{t}} \sum_{b \in \mathcal{B}} \NumPulls{a}{b}{t} \cdot \left|\MR{1}{a}{b}  - \MR{1}{a}{b^*(a)}\right| &\le_{(1)} \frac{L^*}{\NumPullsOneArg{a}{t}} \sum_{b \in \mathcal{B}} \NumPulls{a}{b}{t} \cdot \left|\MR{2}{a}{b}  -\MR{2}{a}{b^*(a)} \right| \\
&=_{(2)} \frac{L^*}{\NumPullsOneArg{a}{t}} \sum_{b \in \mathcal{B}} \NumPulls{a}{b}{t} \cdot \left(\MR{2}{a}{b^*(a)}   - \MR{2}{a}{b}\right),     
\end{align*}
where (1) uses the Lipschitz property and (2) uses the fact that $b^*(a)$ is the best arm for the follower, given that the leader pulls arm $a$. Using the clean event $\Event_F$ and that $L \ge L^*$, we see that:
\[\frac{L^*}{\NumPullsOneArg{a}{t}} \sum_{b \in \mathcal{B}} \NumPulls{a}{b}{t} \cdot \left(\MR{2}{a}{b^*(a)}   - \MR{2}{a}{b}\right) = \frac{L^*}{\NumPullsOneArg{a}{t}} \sum_{1 \le t' < t \mid a_t = a} (\MR{2}{a}{b^*(a)} - \MR{2}{a_t}{b_t}) \le C' \cdot L  \frac{\sqrt{|\mathcal{B}|  \log T}}{\sqrt{\NumPullsOneArg{a}{t}}}.
\]
Taken together, these terms give the desired bound.    
\end{proof}

It will also be convenient to bound the following two quantities which surface in our regret analysis. At a conceptual level, $B_1$ captures the sum of the sizes of the confidence sets of the arms pulled by the leader, and the term $B_2$ captures the cumulative suboptimality of the follower relative to the action $a$ that they are provided in each time step.  
\begin{lemma}
\label{lemma:b1b2}
Assume the setup of Theorem \ref{thrm:dpL} and the notation above. Suppose that the clean event $\Event$ holds. Then it holds that:
\begin{align*}
B_1 &:= \sum_{t=1}^T \left(\frac{10 \sqrt{\mathcal{B} \log T}}{\sqrt{\NumPullsOneArg{a_t}{t}}} + C' \cdot L \cdot \frac{\sqrt{|\mathcal{B}| \log T}}{\sqrt{\NumPullsOneArg{a_t}{t}}} \right) \le O\left(L \sqrt{T |\mathcal{A}| |\mathcal{B}| \log T} \right) \\
B_2 &:= \sum_{t=1}^T \left(\MR{2}{a_t}{b^*(a_t)}  - \MR{2}{a_t}{b_t} \right) \le  O\left(\sqrt{T |\mathcal{A}| |\mathcal{B}| \log T} \right) \\
\end{align*}
\end{lemma}
\begin{proof}

To bound $B_2$, we see that:
\begin{align*}
B_2 &= \sum_{t=1}^T \left(\MR{2}{a_t}{b^*(a_t)}  - \MR{2}{a_t}{b_t} \right) \\
&=  \sum_{a \in \mathcal{A}} \sum_{t \in T \mid a_t = a} \left(\MR{2}{a}{b^*(a)}- \MR{2}{a}{b_t} \right) \\
&\le_{(A)} \sum_{a \in \mathcal{A}} C' \cdot \sqrt{|\mathcal{B}| \cd  \NumPullsOneArg{a}{T} \log T} \\
&= C'  \cdot  \sqrt{|\mathcal{B}|  \log T} \cdot \sum_{a \in \mathcal{A}} \sqrt{\NumPullsOneArg{a}{T}} \\
&\le_{(B)} C'  \cdot  \sqrt{T |\mathcal{A}| |\mathcal{B}| \log T},
\end{align*}
where (A) uses the event $\Event_F$ and (B) uses Jensen's inequality. 

To bound $B_1$, we note that we must upper bound this both with a) the gap of the confidence interval, as well as b) the error on the leader's estimates of their value for arm $a$. Taken together, this yields;
\begin{align*}
B_1 &= \sum_{t=1}^T \left(\frac{10 \sqrt{|\mathcal{B}| \log T}}{\sqrt{\NumPullsOneArg{a_t}{t}}} + C' \cdot L \cdot \frac{\sqrt{|\mathcal{B}| \log T}}{\sqrt{\NumPullsOneArg{a_t}{t}}} \right) \\
&= \sum_{t=1}^T \frac{10 \sqrt{|\mathcal{B}| \log T}}{\sqrt{\NumPullsOneArg{a_t}{t}}} + \sum_{t=1}^T C' \cdot L \cdot \frac{\sqrt{|\mathcal{B}| \log T}}{\sqrt{\NumPullsOneArg{a_t}{t}}} \\
&\le (10 \sqrt{|\mathcal{B}| \log T} + C' \cdot L \sqrt{|\mathcal{B}| \log T}) \sum_{t=1}^T \frac{1}{\sqrt{\NumPullsOneArg{a_t}{t}}} \\
&\le_{(A)} (10 \sqrt{|\mathcal{B} \log T} + C' \cdot L \sqrt{|\mathcal{B}| \log T}) \cdot (2 \cdot \sqrt{T |\mathcal{A}|} + |\mathcal{A}|) \\
&= O \left(L \sqrt{T |\mathcal{A}| |\mathcal{B}| \log T} \right). 
\end{align*}
where (A) follows from Lemma \ref{lemma:boundsqrtexpression}

\end{proof}

We now prove Theorem \ref{thrm:dpL}. 
\begin{proof}[Proof of Theorem \ref{thrm:dpL}]

Assume that clean event $\Event$ holds. This occurs with probability at least $1 - (|\mathcal{A} + 1) T^{-3}$ (Lemma \ref{lemma:cleanlipschitzucb}), so the clean event not occurring counts negligibly towards regret. 

 Moreover, let $(a^*, b^*(a^*))$ be the Stackelberg equilibrium. Let $\alpha_t(a) = \frac{10 \sqrt{\mathcal{B} \log T}}{\sqrt{\NumPullsOneArg{a}{t}}} + C \cdot L \cdot \frac{\sqrt{\log T}}{\sqrt{\NumPullsOneArg{a}{t}}}$ be the confidence bound size at time step $t$ and let $\UCBOneArgTimeStep{1}{a}{t} = \hatMROneArgTimeStep{1}{a}{t} + \alpha_t(a)$ denote the UCB estimate in $\LipschitzUCB(L, C)$ computed during time step $t$ prior to reward at time step $t$ being observed.

We can bound the leader's regret as: 
\begin{align*}
 R_1(T) &= \sum_{t=1}^T  (\MR{1}{a^*}{b^*(a^*)} - \MR{1}{a_t}{b_t})\\
  & 
= \sum_{t=1}^T \left(\MR{1}{a^*}{b^*(a^*)} - \MR{1}{a_t}{b^*(a_t)} \right) + \sum_{t=1}^T \left(\MR{1}{a_t}{b^*(a_t)} - \MR{1}{a_t}{b_t} \right)\\
&\le_{(A)} \sum_{t=1}^T \left(\hatMROneArg{1}{a^*} + \alpha_t(a^*) - \hatMROneArg{1}{a_t} + \alpha_t(a_t) \right) +  \sum_{t=1}^T \left|\MR{1}{a_t}{b^*(a_t)} - \MR{1}{a_t}{b_t} \right| \\
&\le \sum_{t=1}^T \left(\UCBOneArgTimeStep{1}{a^*}{t} - \UCBOneArg{1}{a_t} + 2 \cdot \alpha_t(a_t) \right) + L \cdot \sum_{t=1}^T \left|\MR{2}{a_t}{b^*(a_t)} - \MR{2}{a_t}{b_t}  \right| \\
&\le 2 \cdot \sum_{t=1}^T \alpha_t(a_t) + L \cdot \sum_{t=1}^T \left(\MR{2}{a_t}{b^*(a_t)} - \MR{2}{a_t}{b_t} \right) \\
&= 2 \cdot \sum_{t=1}^T \left(\frac{10 \sqrt{\mathcal{B} \log T}}{\sqrt{\NumPullsOneArg{a_t}{t}}} + C' \cdot L \cdot \frac{\sqrt{|\mathcal{B}| \log T}}{\sqrt{\NumPullsOneArg{a_t}{t}}} \right) + L \cdot B_2 \\
&= 2 \cdot B_1 + L \cdot B_2  \\
&\le_{(B)} O\left(L \sqrt{T |\mathcal{A}| |\mathcal{B}| \log T} \right) 
\end{align*}
where (A) uses Lemma \ref{lemma:UCBconfidenceboundscorrect} and (B) uses Lemma \ref{lemma:b1b2}.  

We also bound the follower's regret as: 
\begin{align*}
 R_2(T) &= \sum_{t=1}^T \left( \MR{2}{a^*}{b^*(a^*)}  - \MR{2}{a_t}{b_t} \right) \\
&= \sum_{t=1}^T \left(\MR{2}{a^*}{b^*(a^*)} - \MR{2}{a^*}{b^*(a_t)}  \right) + \sum_{t=1}^T \left( \MR{2}{a^*}{b^*(a_t)} - \MR{2}{a_t}{b_t} \right) \\
&= \sum_{t=1}^T L \cdot \left|\MR{1}{a^*}{b^*(a^*)} - \MR{1}{a_t}{b^*(a_t)}\right| + B_2 \\
&=_{(A)} \sum_{t=1}^T L \cdot \left(\MR{1}{a^*}{b^*(a^*)} - \MR{1}{a_t}{b^*(a_t)}\right) + B_2 \\
&\le_{(B)} \sum_{t=1}^T L \cdot \left(\hatMROneArgTimeStep{1}{a^*}{t} + \alpha_t(a^*) - \hatMROneArgTimeStep{1}{a_t}{t} + \alpha_t(a^*)\right) + B_2 \\
&= \sum_{t=1}^T L \cdot \left(\UCBOneArgTimeStep{1}{a^*}{t} - \UCBOneArgTimeStep{1}{a_t}{t} + 2 \cdot \alpha_t(a_t)\right) + B_2 \\
&\le \sum_{t=1}^T L \cdot \left(2 \cdot \alpha_t(a_t) \right) + B_2 \\
&= 2 L \cdot \sum_{t=1}^T \left(\frac{10 \sqrt{\mathcal{B} \log T}}{\sqrt{\NumPullsOneArg{a_t}{t}}} + C' \cdot L \cdot \frac{\sqrt{|\mathcal{B}| \log T}}{\sqrt{\NumPullsOneArg{a_t}{t}}} \right) + B_2 \\
&= 2 L \cdot B_1 + B_2  \\
&\le_{(C)} O\left(L^2 \sqrt{T |\mathcal{A}| |\mathcal{B}| \log T} \right)
\end{align*}
where (A) uses the fact that $a^*$ is the action chosen by the leader at the Stackelberg equilibrium
where (B) uses Lemma \ref{lemma:UCBconfidenceboundscorrect} and (C) uses Lemma \ref{lemma:b1b2}.

\end{proof}

\subsection{Proof of Theorem \ref{thm:UCBweaker}}\label{appendix:proofUCBweaker}
We prove Theorem \ref{thm:UCBweaker}, restated below. 
\UCBweaker*

This theorem assumes that $\gamma = \Omega\left( T^{-1/4} \sqrt{|\mathcal{A}|  |\mathcal{B}| \cdot \log T} \right)$.

\paragraph{Notation.} Let $\hatMRTimeStep{1}{a}{b}{t}$ denote the empirical mean of the leader's observed rewards 
\[\left\{ \RR{1}{a}{b}{t'}  \mid 1 \le t' < t, a_t = a, b_t = b \right\}\] 
for $(a,b)$ up to time step $t$. (The leader can observe this information in a $\WeakSetting$.)  Let $\UCBMRTimeStep{1}{a}{b}{t}$ denote the UCB estimate in $\PhasedUCB$ during time step $t$. Let $\NumPullsOneArg{a}{t} = |\left\{1 \le t' < t \mid a_t = a \right\}|$ be the number of times that arm $a$ is pulled before time step $t$. Let $\NumPulls{a}{b}{t} = |\left\{1 \le t' < t \mid a_t = a, b_t = b \right\}|$ be the number of times that arms $(a,b)$ are pulled before time step $t$. Let $C$ be a constant such that $\ActiveArmElimination$ has high-probability instantaneous regret $g(t, T, \mathcal{B}) = C \cdot \sqrt{|\mathcal{B}| \log T / t)}$ (such a constant $C$ exists by \Cref{prop:AAE}). Let $\mathcal{B}_t(a)$ be the computation of the active set at line 3 of $\PhasedUCB$ during time step $t$. Let $s_t(a)$ be the value of the variable $s'(a)$ at the end of the $\ComputeQuantities$ algorithm, when it is called at the beginning of time step $t$ in \texttt{PhasedUCB}. Let $(a^*, b^*)$ be the Stackelberg equilibrium.  

\paragraph{Clean event.} We define the clean event $\Event := \Event_L \cap \Event_F \cap \Event_{L, F}$ to be the intersection of a clean event $\Event_L$ for the leader, a clean event $\Event_F$ for the follower, and a clean event $\Event_{L, F}$ for the follower (using the leader's assessment of the follower). Informally speaking, the clean event $\Event_L$ for the leader is the event that the empirical mean $\hatMR{1}{a}{b}$ is always sufficiently close to the true mean reward $\MR{1}{a}{b}$. 
We formalize the clean event $\Event_L$ as follows: 
\[\forall t \in T, a\in \mathcal{A}, b \in \mathcal{B}: |\hatMRTimeStep{1}{a}{b}{t}  - \MR{1}{a}{b}| \le \frac{10 \sqrt{\log T}}{\sqrt{\NumPulls{a}{b}{t}}}.\]
The clean event $\Event_F$ for the follower is the event that the follower satisfies the high-probability instantaneous regret guarantee:  
\[\forall t \le T: \left|\MR{2}{a_t}{b_t} -\max_{b \in \mathcal{B}} \MR{2}{a_t}{b}  \right| \le C \cdot \frac{\sqrt{|\mathcal{B}| \log T}}{\sqrt{\NumPullsOneArg{a}{t}}}. \]
The final clean event $\Event_{L, F}$ is the event that the active arm set $\mathcal{B}_t(a^*)$ for the Stackelberg action always contains the follower's best-response:
\[\forall t \in T, b \in \mathcal{B}: \argmax_{b \in \mathcal{B}} \MR{2}{a^*}{b}  \in \mathcal{B}_t(a^*).\]

\begin{lemma}
\label{lemma:cleaneventucbweaker}
Assume the setup of Theorem \ref{thm:UCBweaker} and notation above. Then the clean event $\Event$ occurs with high probability: $\mathbb{P}[\Event] \ge 1 - (2 \cdot |\mathcal{A}| + 1) \cdot T^{-3}$.   
\end{lemma}
\begin{proof}
We union bound for $\Event_F$, $\Event_L$, and $\Event_{L, F}$. The analysis for $\Event_L$ follows from a Chernoff bound (and using the analogue of one of the canonical bandit models in \cite{L20}) combined with a union bound. The analysis for $\Event_F$ follows from \Cref{prop:AAE}. The analysis for $\Event_{L, F}$ follows from standard properties of $\ActiveArmElimination$ (e.g., see \cite{L20}) combined with a union bound over $\mathcal{A}$. \end{proof}

The first lemma shows that if the follower runs ActiveArmElimination, for every $a \in \mathcal{A}$ and $b \in \mathcal{B}_t(a)$, we can upper and lower bound the number of pulls $\NumPulls{a}{b}{t}$ in terms of the last phase that the follower has completed (as assessed by the leader). 
\begin{lemma}
\label{lemma:boundarmpullsphaseducb}
Assume the setup of Theorem \ref{thm:UCBweaker} and notation above. Then for every time step $t$, and every $a \in \mathcal{A}$ and $b \in \mathcal{B}_t(a)$, it holds that:
\[ \NumPulls{a}{b}{t} \in \left[\sum_{i=1}^{s_t(a)} M_{i}, \sum_{i=1}^{s_t(a) + 1} M_{i} + 1\right]\]
\end{lemma}
\begin{proof}
This follows from the implementation of $\ComputeQuantities$ combined with the specification of $\ActiveArmElimination$, which guarantees that the follower has finished phase $s_t(a)$ by the end of round $t-2$ and is at most one step into phase $s_t(a) + 2$. 
\end{proof}

The next lemma guarantees that at every time step $t$, the chosen pair of actions $(a_t, b_t)$ are in the $\epsilon_t$-best-response sets for each player, where $\epsilon_t$ depends on the number of times $\NumPullsOneArg{a_t}{t} $ that arm $a_t$ has been chosen so far. 
\begin{lemma}
\label{lemma:UCBweakermain}
Assume the setup of Theorem \ref{thm:UCBweaker} and notation above. Suppose that the clean event $\Event$ holds. Then for every time step $t$, it holds that for 
\[\MR{1}{a_t}{b_t} \ge \min_{a \in \mathcal{A}_{\epsilon_t}^1} \min_{b \in \mathcal{B}_{\epsilon_t}(a)} \MR{1}{a}{b} \]
\[\MR{2}{a_t}{b_t} \ge \min_{a \in \mathcal{A}_{\epsilon_t}^1} \min_{b \in \mathcal{B}_{\epsilon_t}(a)} \MR{2}{a}{b}.\]
for $\epsilon_t = \Theta(\sqrt{|\mathcal{B}| \cdot \log T / \NumPullsOneArg{a_t}{t}} )$. 
\end{lemma}
\begin{proof}
It suffices to show that $a_t \in \mathcal{A}_{\epsilon_t}$ and $b_t \in \mathcal{B}_{\epsilon_t}(a_t)$.

By the clean event $\Event_F$, it immediately follows that $b_t \in \mathcal{B}_{\epsilon_t}(a_t)$. 

To show that $a_t \in \mathcal{A}_{\epsilon_t}$, it suffices to show that $\max_{b \in \mathcal{B}_{\epsilon_t}(a_t)} \MR{1}{a_t}{b} \ge \max_{a' \in \mathcal{A}} \min_{b' \in \mathcal{B}_{\epsilon}(a')} \MR{1}{a'}{b'}  - \epsilon_t$, which can be written as $\max_{a' \in \mathcal{A}} \min_{b' \in \mathcal{B}_{\epsilon_t}(a')} \MR{1}{a'}{b'}  \le \max_{b \in \mathcal{B}_{\epsilon_t}(a_t)} \MR{1}{a}{b} + \epsilon_t$. To see this, observe that: 
\begin{align*}
 \max_{a' \in \mathcal{A}} \min_{b' \in \mathcal{B}_{\epsilon_t}(a')} \MR{1}{a'}{b'}  &\le  \MR{1}{a^*}{b^*} \\
  &\le_{(A)} \max_{b \in \mathcal{B}'(a^*)} \UCBMRTimeStep{1}{a^*}{b}{t} \\
 &\le \max_{b \in \mathcal{B}'_t(a_t)} \UCBMRTimeStep{1}{a_t}{b}{t}\\
 &\le_{(B)}  \max_{b \in \mathcal{B}'_t(a_t)} \left(\MR{1}{a_t}{b}  + 20 \cdot \sqrt{\frac{\log T}{\NumPulls{a_t}{b}{t}}}\right)  \\
 &\le_{(C)} \max_{b \in \mathcal{B}'_t(a_t)} \left(\MR{1}{a_t}{b}  + 20 \cdot \sqrt{\frac{\log T}{\sum_{i=1}^{s_t(a)} M_{i}}} \right)  \\
  &\le_{(D)} \max_{b \in \mathcal{B}'_t(a_t)} \left(\MR{1}{a_t}{b}  \right) +  \Theta\left(\sqrt{\frac{|\mathcal{B}| \log T}{\NumPullsOneArg{a_t}{t}}} \right)  \\
  &\le  \max_{b \in \mathcal{B}'_t(a_t)} \MR{1}{a_t}{b}  + \epsilon_t \\
 &\le_{(E)} \max_{b \in \mathcal{B}_{\epsilon_t}(a_t)} \UCBMRTimeStep{1}{a_t}{b}{t} + \epsilon_t. 
 \end{align*}
 where (A) uses the event $\Event_{L, F}$, (B) uses the event $\Event_L$, (C) applies the lower bound in Lemma \ref{lemma:boundarmpullsphaseducb}, (D) uses the upper bound in Lemma \ref{lemma:boundarmpullsphaseducb} to see that:
  \[ \NumPullsOneArg{a_t}{t} \le \sum_{b \in \mathcal{B}} \NumPulls{a_t}{b}{t} \le \sum_{b \in \mathcal{B}} \left(\left(\sum_{i=1}^{s_t(a) + 1} M_{i}\right) + 1\right) \le \Theta\left(|\mathcal{B}| \cdot \sum_{i=1}^{s_t(a)} M_{i}\right) \] since every arm is pulled  
  and (E) uses the clean event $\Event_F$. 
\end{proof}

Now, we prove Theorem \ref{thm:UCBweaker}. 

\begin{proof}[Proof of Theorem \ref{thm:UCBweaker}]

Assume that the clean event $\Event$ occurs. This occurs with probability at least $1 - (2 \cd |\mathcal{A}| + 1) \cd T^{-3}$ (Lemma \ref{lemma:cleaneventucbweaker}), so the clean event not occurring counts negligibly towards regret. 

We apply Lemma \ref{lemma:UCBweakermain} to see that at time step $t$, it holds that for $\epsilon_t = \Theta(\sqrt{|\mathcal{B}| \cdot \log T / \NumPullsOneArg{a_t}{t}})$, it holds that 
\[\MR{1}{a_t}{b_t} \ge \min_{a \in \mathcal{A}_{\epsilon_t}} \min_{b \in \mathcal{B}_{\epsilon_t}(a)} \MR{1}{a}{b} \]
\[\MR{2}{a_t}{b_t} \ge \min_{a \in \mathcal{A}_{\epsilon_t}} \min_{b \in \mathcal{B}_{\epsilon_t}(a)} \MR{2}{a}{b}.\]
For the leader, this implies that: 
\begin{align*}
 R_1(T) &= \benchmarkrelaxedweaker{1}  \cdot T  - \sum_{t=1}^T \MR{1}{a_t}{b_t} \\
 &\le \sum_{t=1}^T \left(\epsilon_t + \min_{a \in \mathcal{A}_{\epsilon_t}} \min_{b \in \mathcal{B}_{\epsilon_t}(a)} \MR{1}{a}{b} -  \sum_{t=1}^T \MR{1}{a_t}{b_t}\right) + \sum_{t=1}^T \mathbbm{1}[\epsilon_t > \gamma] \\
 &\le \left(\sum_{t=1}^T \epsilon_t\right) + \sum_{t=1}^T \mathbbm{1}[\epsilon_t > \gamma]. 
\end{align*}
For the follower, this similarly implies that:
\begin{align*}
 R_2(T) &= \benchmarkrelaxedweaker{2}  \cdot T  - \sum_{t=1}^T \MR{2}{a_t}{b_t} \\
 &\le \sum_{t=1}^T \left(\epsilon_t + \min_{a \in \mathcal{A}_{\epsilon_t}} \min_{b \in \mathcal{B}_{\epsilon_t}(a)} \MR{2}{a}{b} -  \sum_{t=1}^T \MR{2}{a_t}{b_t}\right) \\ 
 &\le \left(\sum_{t=1}^T \epsilon_t\right) + \sum_{t=1}^T \mathbbm{1}[\epsilon_t > \gamma]. 
\end{align*}

To bound $ \sum_{t=1}^T \epsilon_t$, we observe that: 
\begin{align*}
 \sum_{t=1}^T \epsilon_t &= \sum_{t=1}^T \Theta\left(\sqrt{\frac{|\mathcal{B}| \cdot \log T}{\NumPullsOneArg{a_t}{t}}} \right) \\
 &= \Theta\left(\sqrt{|\mathcal{B}| \cdot \log T} \cdot \sum_{t=1}^T \frac{1}{\sqrt{\NumPullsOneArg{a_t}{t}}} \right) \\
  &\le_{(A)} O\left(\sqrt{|\mathcal{B}| \cdot \log T} \cdot \sqrt{|\mathcal{A}| \cdot T}\right) \\
\end{align*}
where (A) follows from Lemma \ref{lemma:boundsqrtexpression}. This gives the desired upper bound. 

To bound $\sum_{t=1}^T \mathbbm{1}[\epsilon_t > \gamma]$, based on the setting of $\epsilon_t$, we observe that $\epsilon_t \le \gamma$ when $n_{a_t} = O\left(\frac{|\mathcal{B}| \cdot (\log T)}{\epsilon_t^2} \right)$. This means that $\mathbbm{1}[\epsilon_t > \gamma]$ occurs in at most $\Theta\left(\frac{|\mathcal{A}| \cdot |\mathcal{B}| \cdot (\log T)}{\gamma^2} \right)$ time steps. As long as $\gamma = \Omega\left( T^{-1/4} \sqrt{|\mathcal{A}|  |\mathcal{B}| \cdot \log T} \right)$, this term contributes $O\left(\sqrt{|\mathcal{B}| \cdot \log T} \cdot \sqrt{|\mathcal{A}| \cdot T}\right)$ to regret. 

\end{proof}

\section{Proofs for Section \ref{appendix:regularizers}}\label{appendix:proofsregularizers}

\subsection{Proof of Theorem \ref{thm:ETCETCregularizer}}

\etcetcregularizer*

The proof follows a similar argument to the proof of Theorem \ref{thm:ETCETC} and borrows 
some lemmas from Appendix \ref{appendix:proofetcetc}

\begin{proof}[Proof of Theorem \ref{thm:ETCETCregularizer}]

Assume that the clean event $\Event$ holds. This occurs with probability at least $1 - (|\mathcal{A}| \cdot |\mathcal{B}| + |\mathcal{A}|) T^{-3}$ (Lemma \ref{lemma:cleanETCETC}), so the clean event not occuring counts negligibly towards regret.

First, we consider the first $E_2 \cdot |\mathcal{B}| \cdot |\mathcal{A}| + E_1 \cdot |\mathcal{A}|$ time steps. Each time step results in $O(1)$ regret for both players. Based on the settings of $E_1$ and $E_2$, these phases contribute a regret of:
\[E_2 \cdot |\mathcal{B}| \cdot |\mathcal{A}| + E_1 \cdot |\mathcal{A}| = O \left(|\mathcal{A}|^{1 - \eta} \cdot |\mathcal{B}|^{1 - \eta} \cdot (\log T)^{1 - \eta} (c \cdot T)^{\eta} \right). \]

We focus on $t > E_2 \cdot |\mathcal{B}| \cdot |\mathcal{A}| + E_1 \cdot |\mathcal{A}|$ for the remainder of the analysis. Our main ingredient is Lemma \ref{lemma:mainlemmaETCETC}.
Note that $\epsilon^* = \Theta\left(\max\left(\frac{\sqrt{\log T}}{\sqrt{E_1}}, \frac{\sqrt{\log T}}{\sqrt{E_2}} \right)\right) = \Theta\left((|\mathcal{A}| \cdot |\mathcal{B}| \cdot (\log T))^{\eta/2} \cd (c \cdot T)^{-\eta/2}\right) $ based on the settings of $E_1$ and $E_2$. The regret of the leader can be bounded as:
\begin{align*}
  &\benchmarkrelaxedstronger{1} \cdot (T - E_2 \cdot |\mathcal{B}| \cdot |\mathcal{A}| - E_1 \cdot |\mathcal{A}|) - \sum_{t > E_2 \cdot |\mathcal{B}| \cdot |\mathcal{A}| + E_1 \cdot |\mathcal{A}|} \MR{1}{a_t}{b_t}  \\
  &\le (T - E_2 \cdot |\mathcal{B}| \cdot |\mathcal{A}| - E_1 \cdot |\mathcal{A}|) \cdot \left(\max_{a \in \mathcal{A}} \min_{b \in \mathcal{B}_{\epsilon^*}(a)} \MR{1}{a}{b} + c \cdot (\epsilon^*)^d \right) - \sum_{t > E_2 \cdot |\mathcal{B}| \cdot |\mathcal{A}| + E_1 \cdot |\mathcal{A}|} \MR{1}{\tilde{a}}{\tilde{b}(\tilde{a})} \\
  &= (T - E_2 \cdot |\mathcal{B}| \cdot |\mathcal{A}| - E_1 \cdot |\mathcal{A}|) \cdot c \cdot (\epsilon^*)^d  + (T - E_2 \cdot |\mathcal{B}| \cdot |\mathcal{A}| - E_1 \cdot |\mathcal{A}|) \left(\max_{a \in \mathcal{A}} \min_{b \in \mathcal{B}_{\epsilon^*}(a)} \MR{1}{a}{b} - \MR{1}{\tilde{a}}{\tilde{b}(\tilde{a})} \right) \\
  &\le_{(A)} T \cdot c \cdot (\epsilon^*)^d  + T \cdot \epsilon^* \\
  &\le_{(B)} T \cdot c \cdot (\epsilon^*)^d   \\
  &\le \Theta\left((c \cdot T)^{1-(\eta \cdot d/2)} \cdot (|\mathcal{A}| \cdot |\mathcal{B}| \cdot (\log T))^{\eta \cdot d /2}\right) \\
   &= \Theta\left((c \cdot T)^{\eta} \cdot (|\mathcal{A}| \cdot |\mathcal{B}| \cdot (\log T))^{1 - \eta}\right).
\end{align*}
where (A) follows from Lemma \ref{lemma:mainlemmaETCETC} and (B) uses the fact that $c \ge 1$ and $d \le 1$. The regret of the follower can similarly be bounded as:
\begin{align*}
  &\benchmarkrelaxedstronger{1} \cdot (T - E_2 \cdot |\mathcal{B}| \cdot |\mathcal{A}| - E_1 \cdot |\mathcal{A}|) - \sum_{t > E_2 \cdot |\mathcal{B}| \cdot |\mathcal{A}| + E_1 \cdot |\mathcal{A}|} \MR{2}{a_t}{b_t} \\
  &\le (T - E_2 \cdot |\mathcal{B}| \cdot |\mathcal{A}| - E_1 \cdot |\mathcal{A}|) \cdot \left(\min_{a \in \mathcal{A}_{\epsilon^*}} \max_{b \in \mathcal{B}} \MR{2}{a}{b} + c \cdot (\epsilon^*)^d  \right) - \sum_{t > E_2 \cdot |\mathcal{B}| \cdot |\mathcal{A}| + E_1 \cdot |\mathcal{A}|} \MR{2}{\tilde{a}}{\tilde{b}(\tilde{a})} \\
  &= (T - E_2 \cdot |\mathcal{B}| \cdot |\mathcal{A}| - E_1 \cdot |\mathcal{A}|) \cdot c \cdot (\epsilon^*)^d  + (T - E_2 \cdot |\mathcal{B}| \cdot |\mathcal{A}| - E_1 \cdot |\mathcal{A}|) \left(\min_{a \in \mathcal{A}_{\epsilon^*}} \max_{b \in \mathcal{B}} \MR{2}{a}{b} - \MR{2}{\tilde{a}}{\tilde{b}(\tilde{a})} \right) \\
  &\le_{(B)} T \cdot c \cdot (\epsilon^*)^d + T \cdot \epsilon^* \\
  &\le O\left(T^{2/3} (\log T)^{1/3} |\mathcal{A}|^{1/3} |\mathcal{B}|^{1/3} \right). 
\end{align*}
where (B) follows from Lemma \ref{lemma:mainlemmaETCETC}.
This proves the desired result. 

\end{proof}

\subsection{Proof of Theorem \ref{thm:explorethenucbregularizer}}

\explorethenucbregularizer*

The proof follows a similar argument to the proof of Theorem \ref{thm:ETCETC} and borrows 
some lemmas from Appendix \ref{appendix:proofexplorethenucb}

\begin{proof}[Proof of Theorem \ref{thm:explorethenucbregularizer}]
 Assume that the clean event $\Event$ holds. This occurs with probability at least $1 - (1 + |\mathcal{A}|) T^{-3}$ (Lemma \ref{lemma:cleaneventexplorethenucb}), so the clean event not occurring counts negligibly towards regret.

The regret in the explore phase is bounded by $O(1)$ in each round, the total regret from that phase is $E \cdot |\mathcal{A}| = O((|\mathcal{A}| \cdot |\mathcal{B}| \cdot (\log T))^{1 - \eta} \cdot (c \cdot T)^{\eta})$ for either player. 

The remainder of the analysis boils down to bounding the regret in the UCB phase. We separately analyze the regret of the leader and the follower. Observe that $\epsilon^* = \max_{t > E} g(t, T, \mathcal{B}) = O\left( (|\mathcal{A}| \cdot |\mathcal{B}| \log T)^{\eta / 2} \cdot (c \cdot T)^{- \eta /2} \right)$ based on the assumption on the follower's algorithm.

\paragraph{Regret for the leader.} We bound the regret as:
\begin{align*}
&\benchmarkrelaxedstronger{1} \cdot (T-E \cd \abs{\mathcal{A}}) - \sum_{t=E \cdot |\mathcal{A}|}^{T} \MR{1}{a_t}{b_t} \\
&\le \sum_{t=E \cdot |\mathcal{A}|+1}^{T} \left( c \cdot (\epsilon^*)^d + \max_{a \in \mathcal{A}} \min_{b \in \mathcal{B}_{\epsilon^*}(a)} \MR{1}{a}{b} - \MR{1}{a_t}{b_t} \right)  \\
&=  \sum_{a \in \mathcal{A}} \sum_{t \in S(a)}  \left(c \cdot (\epsilon^*)^d + \max_{a \in \mathcal{A}} \min_{b \in \mathcal{B}_{\epsilon^*}(a)} \MR{1}{a}{b} -  \MR{1}{a_t}{b_t}  \right) \\
&\le |\mathcal{A}| + \sum_{a \in \mathcal{A}} \sum_{t \in S(a) \setminus \left\{\max(S(a)) \right\}} \left(c \cdot (\epsilon^*)^d + \max_{a \in \mathcal{A}} \min_{b \in \mathcal{B}_{\epsilon^*}(a)} \MR{1}{a}{b} -  \MR{1}{a_t}{b_t} \right)\\
&\le |\mathcal{A}| + \underbrace{c \cdot (\epsilon^*)^d \cdot T}_{(1)}\\
&+  \underbrace{\sum_{a \in \mathcal{A}} (\NumPullsTwoTimesOneArg{a}{E \cdot |\mathcal{A}|}{T+1} - 1) \cdot \left(\max_{a \in \mathcal{A}} \min_{b \in \mathcal{B}_{\epsilon^*}(a)} \MR{1}{a}{b} -  \frac{1}{\NumPullsTwoTimesOneArg{a}{E \cdot |\mathcal{A}|}{T+1} - 1} \sum_{t \in S(a) \setminus \left\{\max(S(a)) \right\}} \left( \MR{1}{a_t}{b_t} \right) \right)}_{(2)}
\end{align*}
The term $|\mathcal{A}|$ computes negligibly and 
term (1) is equal to $O((|\mathcal{A}| \cdot |\mathcal{B}| \cdot (\log T))^{\eta \cdot d / 2} \cdot (c \cdot T)^{1 - \eta \cdot d /2}) = O((|\mathcal{A}| \cdot |\mathcal{B}| \cdot (\log T))^{1 - \eta} \cdot (c \cdot T)^{\eta})$. Term (2) can be bounded by by the same argument as Theorem \ref{thm:explorethenucb}, which we repeat for completeness: 
\begin{align*}
  &\sum_{a \in \mathcal{A}} (\NumPullsTwoTimesOneArg{a}{E \cdot |\mathcal{A}|}{T+1} - 1) \cdot \left(\max_{a \in \mathcal{A}} \min_{b \in \mathcal{B}_{\epsilon^*}(a)} \MR{1}{a}{b} -  \frac{1}{\NumPullsTwoTimesOneArg{a}{E \cdot |\mathcal{A}|}{T+1} - 1} \sum_{t \in S(a) \setminus \left\{\max(S(a)) \right\}} \left( \MR{1}{a_t}{b_t} \right) \right) \\
  &\le  \sum_{a \in \mathcal{A}} (\NumPullsTwoTimesOneArg{a}{E \cdot |\mathcal{A}|}{T+1} - 1) \cdot \frac{20 \sqrt{\log T}} {\sqrt{\NumPullsTwoTimesOneArg{a}{E \cdot |\mathcal{A}|}{T+1}} - 1} \\
   &\le O\left(\sqrt{\log T} \cdot \sum_{a \in \mathcal{A}} \sqrt{\NumPullsTwoTimesOneArg{a}{E \cdot |\mathcal{A}|}{T+1} - 1} \right) \\
   &\le O\left(\sqrt{|\mathcal{A}| T \log T} \right),
\end{align*}
where the first inequality uses Lemma \ref{lemma:UCBbound} and the last inequality uses Jensen's inequality. 

\paragraph{Regret for the follower.} Note that $\cup_{a \in \mathcal{A}_{\epsilon^*}} S(a)$ denotes the set of time steps where an action in $\mathcal{A}_{\epsilon^*}$ is chosen. We bound the regret as:
\begin{align*}
   &\benchmarkrelaxedstronger{2} \cdot (T - E \cd \abs{\mathcal{A}}) - \sum_{t=E \cdot |\mathcal{A}|}^{T} \MR{2}{a_t}{b_t}  \\
   &\le \underbrace{\left(\sum_{t=E \cdot |\mathcal{A}|}^T \mathbbm{1}[t \not\in \cup_{a \in \mathcal{A}_{\epsilon^*}} S(a)]\right)}_{(1)} + \underbrace{\sum_{t \in \cup_{a \in \mathcal{A}_{\epsilon^*}} S(a)} \left(\min_{a \in \mathcal{A}_{\epsilon^*}} \max_{b \in \mathcal{B}} \MR{2}{a}{b} - \MR{2}{a_t}{b_t} \right)}_{(2)} + \underbrace{c \cdot (\epsilon^*)^d \cdot |\cup_{a \in \mathcal{A}_{\epsilon^*}} S(a)|}_{(3)} 
\end{align*}

Term (1) can be bounded by a similar argument to Theorem \ref{thm:explorethenucb}, which we repeat for completeness. This term can be rewritten as $\sum_{t=E \cdot |\mathcal{A}|}^T \mathbbm{1}[t \not\in \cup_{a \in \mathcal{A}_{\epsilon^*}} S(a)] = \sum_{a \not\in \mathcal{A}_{\epsilon^*}} \NumPullsTwoTimesOneArg{a}{E \cdot |\mathcal{A}|}{T}$. This counts the number of times that arms outside of $\mathcal{A}_{\epsilon^*}$ are pulled during the UCB phase. The key intuition is when an arm $a_t \not\in \mathcal{A}_{\epsilon^*}$, it holds that: 
\[ \MR{1}{a_t}{b_t} \le \max_{b \in \mathcal{B}_{\epsilon^*}(a')} \MR{1}{a_t}{b} < \max_{a \in \mathcal{A}} \min_{b \in \mathcal{B}_{\epsilon^*}(a)} \MR{1}{a}{b}  - \epsilon^*,\]
where the first inequality uses the fact that $b_t \in \mathcal{B}_{\epsilon^*}(a_t)$ (which follows from the clean event $\Event_F$) and the second inequality uses the fact that $a_t \not\in \mathcal{A}_{\epsilon^*}$. This implies that for any $a' \not\in \mathcal{A}_{\epsilon^*}$, the average reward across all time steps (except for the last time step) where $a'$ is pulled satisfies: 
\[\frac{1}{\NumPullsTwoTimesOneArg{a'}{E \cdot |\mathcal{A}|}{T} - 1} \sum_{ t \in S(a') \setminus \left\{\max(S(a'))\right\}} \MR{1}{a_{t'}}{b_{t'}} < \max_{a \in \mathcal{A}} \min_{b \in \mathcal{B}_{\epsilon^*}(a)} \MR{1}{a}{b}  - \epsilon^*.  \]
However, by Lemma \ref{lemma:UCBbound}, we can also lower bound the average reward across all time steps (except for the last time step) where $a'$ is pulled in terms of $\NumPullsTwoTimesOneArg{a'}{E \cdot |\mathcal{A}|}{T}$ as follows: 
\[\frac{1}{\NumPullsTwoTimesOneArg{a'}{E \cdot |\mathcal{A}|}{T} - 1} \sum_{ t \in S(a') \setminus \left\{\max(S(a'))\right\}} \MR{1}{a_{t'}}{b_{t'}} \ge \max_{a \in \mathcal{A}} \min_{b \in \mathcal{B}_{\epsilon^*}(a)} \MR{1}{a}{b} - \frac{10 \sqrt{\log T}}{\sqrt{\NumPullsTwoTimesOneArg{a'}{E \cdot |\mathcal{A}|}{T} - 1}}. \]
Putting these two inequalities together, we see that:
\[\frac{10 \sqrt{\log T}}{\sqrt{\NumPullsTwoTimesOneArg{a'}{E \cdot |\mathcal{A}|}{T} - 1}} \ge \epsilon^*, \]
which bounds the number of times that $a'$ is pulled during the UCB phase as follows: 
\[\NumPullsTwoTimesOneArg{a'}{E \cdot |\mathcal{A}|}{T} \le \Theta\left(\frac{\log T}{(\epsilon^*)^2} \right) = \Theta\left((|\mathcal{A}| \cdot |\mathcal{B}|)^{-\eta}\cdot (\log T)^{1 - \eta} \cdot (c \cdot T)^{\eta} \right). \]
This means that:
\[\sum_{t=E \cdot |\mathcal{A}|}^T \mathbbm{1}[t \not\in \cup_{a \in \mathcal{A}_{\epsilon^*}} S(a)] = \sum_{a \not\in \mathcal{A}_{\epsilon}} \NumPullsTwoTimesOneArg{a}{E \cdot |\mathcal{A}|}{T} \le \Theta\left((|\mathcal{A}| \cdot \log T)^{1 - \eta} \cdot (|\mathcal{B}|)^{-\eta} \cdot (c \cdot T)^{\eta} \right) \]

Next, we bound term (2):
\begin{align*}
  \min_{a \in \mathcal{A}_{\epsilon^*}} \max_{b \in \mathcal{B}} \MR{2}{a}{b} - \mathbb{E}[\MR{2}{a_t}{b_t}] &\le \sum_{t \in \cup_{a \in \mathcal{A}_{\epsilon^*}} S(a)} \left(\max_{b \in \mathcal{B}} \MR{2}{a_t}{b} - \mathbb{E}[\MR{2}{a_t}{b_t}] \right) \le  |\cup_{a \in \mathcal{A}_{\epsilon^*}} S(a)| \cdot \epsilon^* \\
  &\le T \cdot \epsilon^* \\
  &\le T \cdot c \cdot (\epsilon^*)^d \\
  &= O((|\mathcal{A}| \cdot |\mathcal{B}| \cdot (\log T))^{\eta \cdot d / 2} \cdot (c \cdot T)^{1 - \eta \cdot d /2}) \\
  &\le O((|\mathcal{A}| \cdot |\mathcal{B}| \cdot (\log T))^{1 - \eta} \cdot (c \cdot T)^{\eta}).
\end{align*}

Finally, we bound term (3) as 
\[\epsilon^* \cdot |S| \le T \cdot \epsilon^* \le T \cdot c \cdot (\epsilon^*)^d \le O((|\mathcal{A}| \cdot |\mathcal{B}| \cdot (\log T))^{\eta \cdot d / 2} \cdot (c \cdot T)^{1 - \eta \cdot d /2}) = O((|\mathcal{A}| \cdot |\mathcal{B}| \cdot (\log T))^{1 - \eta} \cdot (c \cdot T)^{\eta}).\]
   
\end{proof}

\section{Proofs for \Cref{sec:algorithmsassumptions}}

\subsection{Proofs for \Cref{subsec:algosassumptions}}\label{appendix:proofsalgosassumptions}

The follower algorithms $\ALG_2$ that we analyze in this section run a separate instantation of a standard bandit algorithm for every $a \in \mathcal{A}$. We show that if $\ALG$ satisfies a high-probability instantaneous (resp. anytime) regret bound, the same high-probability instantaneous (resp. anytime) regret bound is inherited for $\ALG_2$ (recall that in Section \ref{subsec:assumptionsplayer2} we defined high-probability instantaneous regret and high-probability anytime regret for both single-bandit learners which act in isolation and follower algorithms). 
\begin{lemma}
\label{lemma:conversion}
Suppose that the follower algorithm $\ALG_2$ runs a separate instantation, for every $a \in \mathcal{A}$, of an single-bandit learning algorithm $\ALG$ operating on the arms $\mathcal{B}$. If $\ALG$ satisfies high-probability instantaneous regret $g$, then $\ALG_2$ satisfies high-probability instantaneous regret $g$. Similarly, if $\ALG$ also satisfies high-probability anytime regret $h$, then $\ALG_2$ also satisfies high-probability anytime regret $h$.
\end{lemma}
\begin{proof}
We use the following notation in the proof. Let $\NumPullsOneArg{a}{t}$ be the number of times that arm $a$ has been pulled up prior to time step $t$. Following \Cref{app:hist}, the follower's history can be represented as: 
\[\Hist_{2, t} := \left\{(t', a_{t'}, b_{t'}, \RR{2}{a_{t'}}{b_{t'}}{t'}) \mid 1 \le t' < t, a_{t'} = a\right\},\]
and the follower's history on the arm $a \in \mathcal{A}$ can be represented as: 
\[\Hist_{2, t, a} := \left\{(\NumPullsOneArg{a}{t'+1}, b_{t'}, \RR{2}{a_{t'}}{b_{t'}}{t'}) \mid 1 \le t' < t, a_{t'} = a\right\}.\]

Using this notation and by the definition of $\ALG_2$, we see that $\ALG_2(a_t, \Hist_{2,t}) = \ALG(\Hist_{2,t, a_t})$. We use this relationship to analyze the  high-probability instantaneous regret and high-probability anytime regret of $\ALG_2$. 

\paragraph{High-probability instantaneous regret.} 
Let the time horizon be $T$, and suppose that $\ALG$ satisfies high-probability instantaneous regret $g(t, T, \mathcal{B})$ for every $1 \le t \le T$. Using this combined with the fact that $\ALG_2(a_t, \Hist_{2,t}) = \ALG(\Hist_{2,t, a_t})$, we see that for each $a \in \mathcal{A}$: 
\[\mathbb{P}\left[\forall t \in [T] \mid \MR{2}{a_t}{b_{t}} \ge \max_{b \in \mathcal{B}} \MR{2}{a_t}{b} - g(\NumPullsOneArg{a}{t+1}, T)\right] 
\ge 1 - T^{-3}.\]
Taking a union bound over $a \in \mathcal{A}$ demonstrates that:
\[\mathbb{P}\left[\forall t \in [T], a \in \mathcal{A} \mid \MR{2}{a_t}{b_{t}} \ge \max_{b \in \mathcal{B}} \MR{2}{a_t}{b} - g(\NumPullsOneArg{a}{t+1}, T)\right] \ge 1 - |\mathcal{A}| \cdot T^{-3},\]
so $\ALG_2$ satisfies high-probability instantaneous regret $g$. 

\paragraph{High-probability anytime regret.} Let the time horizon be $T$, and suppose that $\ALG$ satisfies high-probability anytime regret $h(t, T, \mathcal{B})$ for every $1 \le t \le T$. Using this combined with the fact that $\ALG_2(a_t, \Hist_{2,t}) = \ALG(\Hist_{2,t, a_t})$, we see that for each $a \in \mathcal{A}$: 
\begin{align*}
\mathbb{P}\left[\forall t \in [T] \mid \sum_{t' \le t \mid a_{t'} = a} \max_{b \in \mathcal{B}} \MR{2}{a}{b} - \sum_{t' \le t \mid a_{t'} = a} \MR{2}{a}{b_{t'}} \le h(\NumPullsOneArg{a}{t+1}, T)\right] &\ge 1 - T^{-3}.   
\end{align*}
Taking a union bound over $a \in \mathcal{A}$ demonstrates that:
\[\mathbb{P}\left[\forall t \in [T], a \in \mathcal{A} \mid \sum_{t' \le t \mid a_{t'} = a} \max_{b \in \mathcal{B}} \MR{2}{a}{b} - \sum_{t' \le t \mid a_{t'} = a} \MR{2}{a}{b_{t'}} \le h(\NumPullsOneArg{a}{t+1}, T)\right] \ge 1 - |\mathcal{A}| \cdot T^{-3},\]
so $\ALG_2$ satisfies high-probability anytime regret $h$. 
\end{proof}

Using Lemma \ref{lemma:conversion}, it suffices to analyze the high-probability instantaneous regret and high-probability anytime regret of the following standard bandit algorithms as single-bandit learners with arms $\mathcal{B}$, mean rewards $v_{2}(b)$, and stochastic rewards $r_{2,t}(b)$. In the proofs, we let $\NumPullsOneArg{b}{t}$ denote the number of times that arm $b$ has been pulled prior to time step $t$. 

\AAE*
\begin{proof}[Proof of \Cref{prop:AAE}]

We first show the high-probability instantaneous regret bound and then deduce the high-probability anytime regret bound. 

\paragraph{High-probability instantaneous regret bound.} By Lemma \ref{lemma:conversion}, it suffices to show the bound for $\ActiveArmElimination$ (using phase lengths $M_i = \Theta(\log(T) \cd 2^{2i})$) as a single-bandit learner with arms $\mathcal{B}$, mean rewards $v_2(b)$, and stochastic rewards $r_{2,t}(b)$. We let $\NumPullsOneArg{b}{t}$ denote the number of times that arm $b$ has been pulled prior to time step $t$ in the current phase. Let $\hatMROneArgTimeStep{2}{b}{t}$ denote the empirical mean reward for arm $b$ over the rewards observed prior to time step $t$ in the \textit{previous} (last completed) phase. Let $\mathcal{B}'_{t, \text{curr}}$ be the set of arms active in the \textit{current} phase, and let $\mathcal{B}'_{t, \text{prev}}$ be the set of arms active in the \textit{previous} (last completed) phase. For each time step $t$, let $s'_t$ denote the index of the \textit{previous} (last completed) phase at time step $t$. 

Let the clean event $\Event$ denote the event that at every time step $t$, it holds that:
\[\forall t \in [T], b \in \mathcal{B}'_{t, \text{prev}} : \abs{v_2(b) - \hatMROneArgTimeStep{2}{b}{t}} \le \frac{10 \sqrt{\log T}}{\sqrt{M_{s'_t}}}. \]
Applying a Chernoff bound and a union bound, it holds that $P[\Event] \ge 1 - T^{-3}$. 

We condition on the clean event $\Event$ for the remainder of the analysis. Let $b^* = \argmax_{b \in \mathcal{B}} v_2(b)$. Using the elimination rule, we can bound the suboptimality of each arm $b \in \mathcal{B}'_{t, \text{curr}}$: 
\begin{align*}
  &\abs{v_2(b^*) - v_2(b)} \\
  &\leq \abs{\hatMROneArgTimeStep{2}{b^*}{t} -  v_2(b^*)} + \abs{\hatMROneArgTimeStep{2}{b}{t} -  v_2(b)} + \abs{\hatMROneArgTimeStep{2}{b^*}{t} -  \hatMROneArgTimeStep{2}{b}{t}} \\
  &\leq 40 \frac{\sqrt{\log(T)}}{\sqrt{M_{s'_t}}} \\
  &\leq \Theta(2^{-s'_t}). 
\end{align*}

It suffices to lower bound $2^{-2 \cdot s'_t}$. We observe that:
\[t \le |\mathcal{B}| \left( M_{s'_t +1} + \sum_{s=1}^{s'_t} M_{s} \right) \le \Theta(\abs{\mathcal{B}} \cd \log(T) \cdot 2^{2 \cdot s'_t}), \]
where the last expression uses the geometric rate of increase of $M_i = \Theta(\log(T) \cdot 2^{2i})$. 
This implies that 
\[2^{-s'_t} = O(\sqrt{\abs{\mathcal{B}} \cd \log T/t}).\]

Altogether, this implies that:
\[v_2(b_t) \ge \max_{b \in \mathcal{B}} v_2(b) - O(\sqrt{|\mathcal{B}| \cd \log T/t}),\]
as desired. 

\paragraph{High-probability anytime regret bound.}
Using Observation \ref{observation:conversion}, it holds that the high-probability anytime regret can be bounded as: 
$$\sum_{t'=1}^t O \left(\sqrt{\frac{\log(T) \cd \abs{\mathcal{B}}}{t'}}\right) = \sqrt{\log(T) \cd \abs{\mathcal{B}}} \cd O \left(\sum_{t'=1}^t \frac{1}{\sqrt{i}}\right) \leq_{(A)} \Theta(\sqrt{\log(T) \cd t \cd \abs{\mathcal{B}}})$$
where (A) follows from an integral bound and Jensen's inequality. This proves the desired bound. 
\end{proof}

\ETC*
\begin{proof}[Proof of \Cref{prop:ETC}]

By Lemma \ref{lemma:conversion}, it suffices to show the instantaneous regret bound for $\ExploreThenCommit$ as a single-bandit learner with arms $\mathcal{B}$, mean rewards $v_2(b)$, and stochastic rewards $r_{2,t}(b)$. We let $\NumPullsOneArg{b}{t}$ denote the number of times that arm $b$ has been pulled prior to time step $t$. Let $\hatMROneArgTimeStep{2}{b}{t}$ denote the empirical mean reward for arm $b$ over the rewards observed prior to time step $t$. 

Let the clean event $\Event$ capture the event that the empirical mean of every arm is close to the true mean whenever $t > E \cdot |\mathcal{B}|$ time steps, that is: 
$$\forall b \in \mathcal{B}, t > E \cdot |\mathcal{B}|: \abs{\hatMROneArgTimeStep{2}{b}{t} - v_2(b)} \leq 10 \cdot \frac{\sqrt{\log(T)}}{\sqrt{E}}$$
Applying a Chernoff bound (and using the analogue of one of the canonical bandit models in \cite{L20}), it holds that $P[\Event] \ge 1 - T^{-3}$. 

Now, conditioning on the clean event $\Event$, we see that after time step $t > E \cdot |\mathcal{B}|$, it holds that:
$$\abs{\hatMROneArgTimeStep{2}{b}{t} - v_2(b)} \leq 10 \frac{\sqrt{\log(T)}}{\sqrt{E}}.$$
Since the algorithm chooses the arm with highest empirical mean from the first  $E \cdot |\mathcal{B}|$ time steps is selected, this means that: 
$$\max_{b \in \mathcal{B}} v_2(b) - v_2(b)  \leq 20 \cd \frac{\sqrt{\log(T)}}{\sqrt{E}}$$
for any $t >  E \cdot |\mathcal{B}|$. 
\end{proof}

\UCB*
\begin{proof}[Proof of \Cref{prop:UCB}]
By Lemma \ref{lemma:conversion}, it suffices to show the anytime regret bound for UCB as a single-bandit learner with arms $\mathcal{B}$, mean rewards $v_2(b)$, and stochastic rewards $r_{2,t}(b)$. We let $\NumPullsOneArg{b}{t}$ denote the number of times that arm $b$ has been pulled prior to time step $t$. Let $\hatMROneArgTimeStep{2}{b}{t}$ denote the empirical mean reward for arm $b$ over the rewards observed prior to time step $t$. 

We define the \emph{clean event} $\Event$ as the true mean being contained within the upper and lower confidence bounds for each arm $a$, that is: 
$$\forall b \in \mathcal{B}, t \le T: \abs{\hatMROneArgTimeStep{2}{b}{t} - v_2(b)} \leq 10 \cd \sqrt{\frac{\log(T)}{\NumPullsOneArg{b}{t}}}.$$ By a Chernoff bound (and using the analogue of one of the canonical bandit models in \cite{L20}) followed by a union bound, we have that $P[\Event] \ge 1- T^{-3}$. 

We condition on $\Event$ for the remainder of the analysis. Since the arm with highest upper confidence bound is always chosen and since $\Event$ holds, the selected arm $b_t$'s true mean $v_2(b_t)$ falls within the $2 \cdot \sqrt{\frac{\log(T)}{\NumPullsOneArg{b_t}{t}}}$ bound. By Lemma \ref{lemma:boundsqrtexpression}, this means that the regret at any time step $t$ for any arm $a \in \mathcal{A}$ is upper bounded by: 
$$10 \cdot \sum_{t'=1}^t \sqrt{\frac{\log(T)}{\NumPullsOneArg{b_{t'}}{t'}}} \leq 10 \cdot \sqrt{\log(T) \cd \abs{\mathcal{B}} \cd t}$$
as desired. 
\end{proof}

\subsection{Proofs for Section \ref{subsec:generalizationfollower}}\label{appendix:proofsalgosgeneralized}

We prove Theorem \ref{thm:explorethenucbgeneralizedg} in Appendix \ref{appendix:proofexplorethenucbgeneralizedg} and we prove Theorem \ref{thrm:dpLgeneralized} in Appendix \ref{appendix:proofdpLgeneralized}.

\subsubsection{Proof of Theorem \ref{thm:explorethenucbgeneralizedg}}\label{appendix:proofexplorethenucbgeneralizedg}

We prove Theorem \ref{thm:explorethenucbgeneralizedg}, following a similar argument to the proof of Theorem \ref{thm:explorethenucb}.

\explorethenucbgeneralized*

We assume $\gamma = \omega\left(|\mathcal{A}|^{c_1/(1+c_1)} |\mathcal{B}|^{c_2/(1+c_1)} (\log T)^{c_3/(1+c_1)} T^{-1/(1+c_1)}\right)$.

\paragraph{Notation and Clean Event.} We use the same notation as in the proof of Theorem \ref{thm:explorethenucb}. We also define the clean event $\Event := \Event_L \cap \Event_F$ to be the same as in the proof of Theorem \ref{thm:explorethenucb}. 

We prove that the clean event $\Event$ occurs with high probability, generalizing Lemma \ref{lemma:cleaneventexplorethenucb}.
\begin{lemma}
\label{lemma:cleaneventexplorethenucbgeneralized}
Assume the notation above.  Let $\ALG_{2}$ be any algorithm with high-probability instantaneous regret $g$ where $g(t, T, \mathcal{B}) = O(E^{-c_1} |\mathcal{B}|^{c_2} (\log T)^{c_3})$ for $t > E$ and $g(t, T, \mathcal{B}) = 1$ for $t \le E$, and let $\ALG_1 = \ExploreThenUCB(E)$.  Then, the event $\Event$ occurs with high probability: $\mathbb{P}[\Event] \ge 1 - T^{-3} (|\mathcal{A}| + 1)$. 
\end{lemma}
\begin{proof}
We first show that $\mathbb{P}[\Event_F] \ge 1 - |\mathcal{A}| \cdot T^{-3}$. A sufficient condition for this event to hold is that: 
\[\forall t > E \cdot |\mathcal{A}|: \MR{2}{a_t}{b_t} \ge \max_{b \in \mathcal{B}} \MR{2}{a_t}{b} - \max_{t > E} g(t, T, \mathcal{B}).\]
Since the exploration phases pulls every arm $a \in \mathcal{A}$ a total of $E$ times, the high-probability instantaneous regret assumption guarantees that this event holds with probability at least $1 - |\mathcal{A}| \cdot T^{-3}$, as desired.

We next show that $\mathbb{P}[\Event_L] \ge 1 - T^{-3}$. This follows from a a Chernoff bound (and using the analogue of one of the canonical bandit models in \cite{L20}) combined with a union bound. 

The lemma follows from another union bound over $\Event_L$ and $\Event_F$. 
\end{proof}

Now we are ready to prove Theorem \ref{thm:explorethenucbgeneralizedg}. 

\begin{proof}[Proof of Theorem \ref{thm:explorethenucbgeneralizedg}]
Assume that the clean event $\Event$ holds. This occurs with probability at least $1 - (1 + |\mathcal{A}|) T^{-3}$ (Lemma \ref{lemma:cleaneventexplorethenucbgeneralized}), so the clean event not occurring counts negligibly towards regret.

The regret in the explore phase is bounded by $O(1)$ in each round, the total regret from that phase is $O(T^{1/(1+c_1)} |\mathcal{A}|^{c_1/(1+c_1)}  |\mathcal{B}|^{c_2/(1+c_1)} (\log T)^{c_3/(1+c_1)})$ for either player. 

The remainder of the analysis boils down to bounding the regret in the UCB phase. We separately analyze the regret of the leader and the follower. Observe that $\epsilon^* = \max_{t > E} g(t, T, \mathcal{B}) = O\left( |\mathcal{B}|^{c_2} (\log T)^{c_3} E^{-c_1} \right) = O\left(|\mathcal{A}|^{c_1/(1+c_1)} |\mathcal{B}|^{c_2/(1+c_1)} (\log T)^{c_3/(1+c_1)} T^{-c_1/(1+c_1)} \right)$ is based on the assumption on the follower's algorithm.

\paragraph{Regret for the leader.} We use a similar analysis as in the proof of Theorem \ref{thm:explorethenucb}, repeating the full analysis for completeness.
\begin{align*}
&\benchmarkrelaxedstronger{1} \cdot (T-E \cd \abs{\mathcal{A}}) - \sum_{t=E \cdot |\mathcal{A}| + 1}^{T} \MR{1}{a_t}{b_t} \\
&\le \sum_{t=E \cdot |\mathcal{A}| + 1}^{T} \left(\epsilon^* + \max_{a \in \mathcal{A}} \min_{b \in \mathcal{B}_{\epsilon^*}(a)} \MR{1}{a}{b} - \MR{1}{a_t}{b_t} \right)  \\
&=  \sum_{a \in \mathcal{A}} \sum_{t \in S(a)}  \left(\epsilon^* + \max_{a \in \mathcal{A}} \min_{b \in \mathcal{B}_{\epsilon^*}(a)} \MR{1}{a}{b} -  \MR{1}{a_t}{b_t}  \right) \\
&\le |\mathcal{A}| + \sum_{a \in \mathcal{A}} \sum_{t \in S(a) \setminus \left\{\max(S(a)) \right\}} \left(\epsilon^* + \max_{a \in \mathcal{A}} \min_{b \in \mathcal{B}_{\epsilon^*}(a)} \MR{1}{a}{b} -  \MR{1}{a_t}{b_t} \right)\\
&\le |\mathcal{A}| + \underbrace{\epsilon^* \cdot T}_{(1)}\\
&+  \underbrace{\sum_{a \in \mathcal{A}} (\NumPullsTwoTimesOneArg{a}{E \cdot |\mathcal{A}|}{T+1} - 1) \cdot \left(\max_{a \in \mathcal{A}} \min_{b \in \mathcal{B}_{\epsilon^*}(a)} \MR{1}{a}{b} -  \frac{1}{\NumPullsTwoTimesOneArg{a}{E \cdot |\mathcal{A}|}{T+1} - 1} \sum_{t \in S(a) \setminus \left\{\max(S(a)) \right\}} \left( \MR{1}{a_t}{b_t} \right) \right)}_{(2)}
\end{align*}
The term $|\mathcal{A}|$ computes negligibly, 
term (1) is equal to $\Theta(\mathcal{A}|^{c_1/(1+c_1)} |\mathcal{B}|^{c_2/(1+c_1)} (\log T)^{c_3/(1+c_1)} T^{1/(1+c_1)})$, and term (2) can be bounded by:
\begin{align*}
  &\sum_{a \in \mathcal{A}} (\NumPullsTwoTimesOneArg{a}{E \cdot |\mathcal{A}|}{T} - 1) \cdot \left(\max_{a \in \mathcal{A}} \min_{b \in \mathcal{B}_{\epsilon^*}(a)} \MR{1}{a}{b} -  \frac{1}{\NumPullsTwoTimesOneArg{a}{E \cdot |\mathcal{A}|}{T+1} - 1} \sum_{t \in S(a) \setminus \left\{\max(S(a)) \right\}} \left( \MR{1}{a_t}{b_t} \right) \right) \\
  &\le  \sum_{a \in \mathcal{A}} (\NumPullsTwoTimesOneArg{a}{E \cdot |\mathcal{A}|}{T+1} - 1) \cdot \frac{20 \sqrt{\log T}} {\sqrt{\NumPullsTwoTimesOneArg{a}{E \cdot |\mathcal{A}|}{T+1}} - 1} \\
   &\le O\left(\sqrt{\log T} \cdot \sum_{a \in \mathcal{A}} \sqrt{\NumPullsTwoTimesOneArg{a}{E \cdot |\mathcal{A}|}{T+1} - 1} \right) \\
   &\le O\left(\sqrt{|\mathcal{A}| T \log T} \right),
\end{align*}
where the first inequality uses Lemma \ref{lemma:UCBbound} and the last inequality uses Jensen's inequality. 

\paragraph{Regret for the follower.} We  modify the analysis from the proof of Theorem \ref{thm:explorethenucb}. We bound the regret as:
\begin{align*}
   &\benchmarkrelaxedstronger{2} \cdot (T - E \cd \abs{\mathcal{A}}) - \sum_{t=E \cdot |\mathcal{A}| + 1}^{T} \MR{2}{a_t}{b_t}  \\
   &\le \underbrace{\sum_{t = E |\mathcal{A}|+1}^T \left(\min_{a \in \mathcal{A}_{\epsilon_t}} \max_{b \in \mathcal{B}} \MR{2}{a}{b} - \MR{2}{a_t}{b_t} \right)}_{(1)} +  \underbrace{\sum_{t = E |\mathcal{A}|+1}^T \epsilon_t}_{(2)} 
\end{align*}
where
\[
\epsilon_t =
\begin{cases}
    1 & \text{ if } \NumPullsTwoTimesOneArg{a_t}{E \cdot |\mathcal{A}|}{t} = 1 \\
\max\left(\epsilon^*, 20 \frac{\sqrt{\log T}}{\sqrt{\NumPullsTwoTimesOneArg{a_t}{E \cdot |\mathcal{A}|}{t}}} \right) & \text{ else }.
\end{cases}
\]

We bound term (1). We first show that $a_t \in \mathcal{A}_{\epsilon_t}$: 
\begin{align*}
\max_{b \in \mathcal{B}_{\epsilon_t}(a_t)} \MR{1}{a_t}{b} &\ge \max_{b \in \mathcal{B}_{\epsilon^*}(a_t)} \MR{1}{a_t}{b} \\
&\ge_{(A)} \frac{1}{\NumPullsTwoTimesOneArg{a_t}{E \cdot |\mathcal{A}|}{t}} \sum_{E \cdot |\mathcal{A}| <t' < t \mid a_{t'} = a_t} \MR{1}{a_{t'}}{b_{t'}} \\
&\ge_{(B)} \hatMROneArgTimeStep{1}{a_t}{t} - \frac{10 \sqrt{\log T}}{\sqrt{\NumPullsTwoTimesOneArg{a_t}{E \cdot |\mathcal{A}|}{t}}} \\
&= \UCBOneArgTimeStep{1}{a_t}{t} - \frac{20 \sqrt{\log T}}{\sqrt{\NumPullsTwoTimesOneArg{a_t}{E \cdot |\mathcal{A}|}{t}}} \\
&= \max_{a \in \mathcal{A}} \left(\UCBOneArgTimeStep{1}{a}{t}\right) - \frac{20 \sqrt{\log T}}{\sqrt{\NumPullsTwoTimesOneArg{a_t}{E \cdot |\mathcal{A}|}{t}}}  \\
&\ge_{(C)}   \max_{a \in \mathcal{A}} \left(\frac{1}{\NumPullsTwoTimesOneArg{a_t}{E \cdot |\mathcal{A}|}{t}} \sum_{E \cdot |\mathcal{A}| <t' < t \mid a_{t'} = a} \MR{1}{a_{t'}}{b_{t'}} \right) - \frac{20 \sqrt{\log T}}{\sqrt{\NumPullsTwoTimesOneArg{a_t}{E \cdot |\mathcal{A}|}{t}}}  \\
&\ge_{(D)}   \max_{a \in \mathcal{A}} \min_{b \in \mathcal{B}_{\epsilon^*}(a)}  \MR{1}{a}{b} - \frac{20 \sqrt{\log T}}{\sqrt{\NumPullsTwoTimesOneArg{a_t}{E \cdot |\mathcal{A}|}{t}}}  \\
&\ge   \max_{a \in \mathcal{A}} \min_{b \in \mathcal{B}_{\epsilon_t}(a)}  \MR{1}{a}{b} - \epsilon_t.
\end{align*}
where (A) and (D) uses the event $\Event_F$, and (B) and (C) use the event $\Event_L$. Applying $\Event_F$ again, this implies that: 
\begin{align*}
 \min_{a \in \mathcal{A}_{\epsilon_t}} \max_{b \in \mathcal{B}} \MR{2}{a}{b} - \MR{2}{a_t}{b_t} &\le \MR{2}{a_t}{b} - \MR{2}{a_t}{b_t} \\  
 &\le \epsilon^*.
\end{align*}
Putting this all together, term (1) is bounded by 
\[\sum_{t = E |\mathcal{A}|+1}^T \left(\min_{a \in \mathcal{A}_{\epsilon_t}} \max_{b \in \mathcal{B}} \MR{2}{a}{b} - \MR{2}{a_t}{b_t} \right) \le \epsilon^* \cdot T = \Theta(\mathcal{A}|^{c_1/(1+c_1)} |\mathcal{B}|^{c_2/(1+c_1)} (\log T)^{c_3/(1+c_1)} T^{1/(1+c_1)}),\] as desired.

We next bound term (2) as follows: 
\begin{align*}
  \sum_{t = E |\mathcal{A}|+1}^T \epsilon_t 
  &= |\mathcal{A}| +\sum_{a \in \mathcal{A}} \sum_{t \in S(a) \setminus \min(S(a))} \max\left(\epsilon^*, 20 \frac{\sqrt{\log T}}{\sqrt{\NumPullsTwoTimesOneArg{a}{E \cdot |\mathcal{A}|}{t}}} \right)  \\
   &\le |\mathcal{A}| + \epsilon^* \cdot T + 20 \sqrt{\log T} \cdot \sum_{a \in \mathcal{A}} \sum_{t \in S(a) \setminus \min(S(a))} \frac{1}{\sqrt{\NumPullsTwoTimesOneArg{a}{E \cdot |\mathcal{A}|}{t}}}  \\
   &\le_{(A)} |\mathcal{A}| + \epsilon^* \cdot T + O \left(\sqrt{T |\mathcal{A}| \log T } \right)\\
   &\le  \Theta(\mathcal{A}|^{c_1/(1+c_1)} |\mathcal{B}|^{c_2/(1+c_1)} (\log T)^{c_3/(1+c_1)} T^{1/(1+c_1)}) + O \left(\sqrt{T |\mathcal{A}| \log T } \right),
\end{align*}
where (A) uses Lemma \ref{lemma:boundsqrtexpression}.

Putting this all together yields the desired bound.
\end{proof}

\subsubsection{Proof of Theorem \ref{thrm:dpLgeneralized}}\label{appendix:proofdpLgeneralized}

We prove Theorem \ref{thrm:dpLgeneralized}, following a similar approach to the proof of Theorem \ref{thrm:dpL}.

\dpLgeneralized*

\paragraph{Notation.} We use the same notation as in the proof of Theorem \ref{thrm:dpL}.

\paragraph{Clean event.} We again define the clean event $\Event = \Event_L \cap \Event_F$ to be the intersection of a clean event $\Event_L$ for the leader and a clean event $\Event_F$ for the follower. The event $\Event_L$ is the same as in the proof of Theorem \ref{thrm:dpL}. The event $\Event_F$ is formalized as follows:
\[\forall a\in \mathcal{A}, t \le T:  \sum_{1 \le t' < t \mid a_t = a} (\MR{2}{a}{b^*(a)} - \MR{2}{a_t}{b_t}) \le C' (\NumPullsOneArg{a}{t})^{c_1} |\mathcal{B}|^{c_2} (\log T)^{c_3} \]

We first generalize Lemma \ref{lemma:cleanlipschitzucb}. 
\begin{lemma}
\label{lemma:cleanlipschitzucbgeneralized}
Assume the setup of Theorem \ref{thrm:dpL} and the notation above. Then the clean event occurs with high probability: $\mathbb{P}[\Event] \ge 1 - T^{-3} (|\mathcal{A}| + 1)$.
\end{lemma}
\begin{proof}
We union bound over $\Event_L$ and $\Event_F$. The analysis for $\Event_F$ follows from the high-probability anytime regret bound assumption. The analysis for $\Event_L$ follows from a Chernoff bound (and using the analogue of one of the canonical bandit models in \cite{L20}) combined with a union bound. 
\end{proof}

The following lemma generalizes Lemma \ref{lemma:UCBconfidenceboundscorrect}. 
\begin{lemma}
\label{lemma:UCBconfidenceboundscorrectgeneralized}
Assume the setup of Theorem \ref{thrm:dpLgeneralized} and the notation above. Suppose that the clean event $\Event$ holds. Then for any $t \le T$ and $a \in \mathcal{A}$, it holds that: 
\[\left| \hatMROneArgTimeStep{1}{a}{t} - \MR{1}{a}{b^*(a)} \right| \le 
\frac{10 \sqrt{\abs{\mathcal{B}} \log T}}{\sqrt{\NumPullsOneArg{a}{t}}} + C' \cdot L \cdot (\NumPullsOneArg{a}{t})^{c_1-1} \cdot |\mathcal{B}|^{c_2} (\log T)^{c_3} .\] 
\end{lemma}
\begin{proof}
The proof follows similarly to the proof of Lemma \ref{lemma:UCBconfidenceboundscorrect}. 
We observe that: 
\begin{align*}
\left| \hatMROneArgTimeStep{1}{a}{t} - \MR{1}{a}{b^*(a)} \right| &= \left|\left(\frac{1}{\NumPullsOneArg{a}{t}} \sum_{b \in \mathcal{B}} \NumPulls{a}{b}{t} \cdot \hatMR{1}{a}{b} \right) - \MR{1}{a}{b^*(a)}\right| \\
&= \left|\left(\frac{1}{\NumPullsOneArg{a}{t}} \sum_{b \in \mathcal{B}} \NumPulls{a}{b}{t} \cdot \hatMR{1}{a}{b} \right) - \frac{1}{\NumPullsOneArg{a}{t}} \left(\sum_{b \in \mathcal{B}} \NumPulls{a}{b}{t} \cdot \MR{1}{a}{b^*(a)}\right) \right| \\
 &\le \frac{1}{\NumPullsOneArg{a}{t}} \sum_{b \in \mathcal{B}} \NumPulls{a}{b}{t} \cdot \left|  \hatMR{1}{a}{b} - \MR{1}{a}{b^*(a)}\right| \\
 &\le \underbrace{\frac{1}{\NumPullsOneArg{a}{t}} \sum_{b \in \mathcal{B}} \NumPulls{a}{b}{t} \cdot \left|\hatMR{1}{a}{b} - \MR{1}{a}{b}\right|}_{(A)} + \underbrace{\frac{1}{\NumPullsOneArg{a}{t}} \sum_{b \in \mathcal{B}} \NumPulls{a}{b}{t} \cdot \left|\MR{1}{a}{b}  - \MR{1}{a}{b^*(a)}\right|}_{(B)}.
\end{align*}

The bound of term (A) proceeds the same as before, and repeat the proof for completeness: 
\begin{align*}
\frac{1}{\NumPullsOneArg{a}{t}} \sum_{b \in \mathcal{B}} \NumPulls{a}{b}{t} \cdot \left|\hatMR{1}{a}{b} -  \MR{1}{a}{b}\right| &\le_{(1)} \frac{1}{\NumPullsOneArg{a}{t}} \sum_{b \in \mathcal{B}} \NumPulls{a}{b}{t} \cdot \frac{10\sqrt{\log T}}{\sqrt{\NumPulls{a}{b}{t}}} \\
&= \frac{10 \sqrt{\log T}}{\NumPullsOneArg{a}{t}} \sum_{\sum_{b \in \mathcal{B}}} \sqrt{\NumPulls{a}{b}{t}} \\
&\le_{(2)} \frac{10 \sqrt{|\mathcal{B}| \log T}}{\sqrt{\NumPullsOneArg{a}{t}}}. 
\end{align*}
where (1) uses the clean event $\Event_L$ and (2) uses Jensen's inequality. 

The bound of term (B) proceeds similarly, with some minor modifications:
\begin{align*}
\frac{1}{\NumPullsOneArg{a}{t}} \sum_{b \in \mathcal{B}} \NumPulls{a}{b}{t} \cdot \left|\MR{1}{a}{b}  - \MR{1}{a}{b^*(a)}\right| &\le_{(1)} \frac{L^*}{\NumPullsOneArg{a}{t}} \sum_{b \in \mathcal{B}} \NumPulls{a}{b}{t} \cdot \left|\MR{2}{a}{b}  -\MR{2}{a}{b^*(a)} \right| \\
&=_{(2)} \frac{L^*}{\NumPullsOneArg{a}{t}} \sum_{b \in \mathcal{B}} \NumPulls{a}{b}{t} \cdot \left(\MR{2}{a}{b^*(a)}   - \MR{2}{a}{b}\right),     
\end{align*}
where (1) uses the Lipschitz property and (2) uses the fact that $b^*(a)$ is the best arm for the follower, given that the leader pulls arm $a$. Using the clean event $\Event_F$ and that $L \ge L^*$, we see that:
\begin{align*}
\frac{L^*}{\NumPullsOneArg{a}{t}} \sum_{b \in \mathcal{B}} \NumPulls{a}{b}{t} \cdot \left(\MR{2}{a}{b^*(a)}   - \MR{2}{a}{b}\right) &= \frac{L^*}{\NumPullsOneArg{a}{t}} \sum_{1 \le t' < t \mid a_t = a} (\MR{2}{a}{b^*(a)} - \MR{2}{a_t}{b_t}) \\
&\le C' \cdot L\cdot  (\NumPullsOneArg{a}{t})^{c_1-1} |\mathcal{B}|^{c_2} (\log T)^{c_3}.     
\end{align*}
Taken together, these terms give the desired bound.    
\end{proof}

We next generalize Lemma \ref{lemma:b1b2}.  
\begin{lemma}
\label{lemma:b1b2generalized}
Assume the setup of Theorem \ref{thrm:dpLgeneralized} and the notation above. Suppose that the clean event $\Event$ holds. Then it holds that:
\begin{align*}
B_1 &:= \sum_{t=1}^T \left(\frac{10 \sqrt{\mathcal{B} \log T}}{\sqrt{\NumPullsOneArg{a_t}{t}}} + C' \cdot L \cdot |\mathcal{B}|\frac{\sqrt{|\mathcal{B}| \log T}}{\sqrt{\NumPullsOneArg{a_t}{t}}} \right) \le O \left(\sqrt{|\mathcal{A}| |\mathcal{B}| T \log T} + L \cdot |\mathcal{A}|^{1-c_1} |\mathcal{B}|^{c_2} (\log T)^{c_3} T^{c_1} \right) \\
B_2 &:= \sum_{t=1}^T \left(\MR{2}{a_t}{b^*(a_t)}  - \MR{2}{a_t}{b_t} \right) \le  O\left(|\mathcal{A}|^{1-c_1} |\mathcal{B}|^{c_2} \cdot (\log T)^{c_3}  T^{c_1} \right) \\
\end{align*}
\end{lemma}
\begin{proof}

To bound $B_2$, we see that:
\begin{align*}
B_2 &= \sum_{t=1}^T \left(\MR{2}{a_t}{b^*(a_t)}  - \MR{2}{a_t}{b_t} \right) \\
&=  \sum_{a \in \mathcal{A}} \sum_{t \in T \mid a_t = a} \left(\MR{2}{a}{b^*(a)}- \MR{2}{a}{b_t} \right) \\
&\le_{(A)} \sum_{a \in \mathcal{A}} C' (\NumPullsOneArg{a}{T})^{c_1} |\mathcal{B}|^{c_2} \cdot (\log T)^{c_3} \\
&= C' |\mathcal{B}|^{c_2} \cdot (\log T)^{c_3}  \cdot \sum_{a \in \mathcal{A}}  (\NumPullsOneArg{a}{T+1})^{c_1}  \\
&\le_{(B)}  O\left(|\mathcal{A}|^{1-c_1} |\mathcal{B}|^{c_2} \cdot (\log T)^{c_3}  T^{c_1}\right)
\end{align*}
where (A) uses the event $\Event_F$ and (B) uses Jensen's inequality. 

To bound $B_1$:
\begin{align*}
B_1 &= \sum_{t=1}^T \left(\frac{10 \sqrt{|\mathcal{B}| \log T}}{\sqrt{\NumPullsOneArg{a_t}{t}}} +C' \cdot L \cdot |\NumPullsOneArg{a_t}{t}|^{c_1-1} |\mathcal{B}|^{c_2} (\log T)^{c_3} \right) \\
&=_{(A)} O \left(\sqrt{|\mathcal{A}| |\mathcal{B}| T \log T} \right) +C' \cdot L \cdot |\mathcal{B}|^{c_2} (\log T)^{c_3} \cdot \sum_{t=1}^T |\NumPullsOneArg{a_t}{t}|^{c_1-1}  \\
&=_{(A)} O \left(\sqrt{|\mathcal{A}| |\mathcal{B}| T \log T} \right) +C' \cdot L \cdot |\mathcal{B}|^{c_2} (\log T)^{c_3} \cdot \sum_{a\in \mathcal{A}} \sum_{t \mid a_t = a} |\NumPullsOneArg{a_t}{t}|^{c_1-1}  \\
&\le_{(B)} O \left(\sqrt{|\mathcal{A}| |\mathcal{B}| T \log T} + L \cdot |\mathcal{B}|^{c_2} (\log T)^{c_3} \cdot \sum_{t \mid a_t = a} |\NumPullsOneArg{a}{T+1}|^{c_1} \right)  \\
&\le_{(C)} O \left(\sqrt{|\mathcal{A}| |\mathcal{B}| T \log T} + L \cdot |\mathcal{A}|^{1-c_1} |\mathcal{B}|^{c_2} (\log T)^{c_3} T^{c_1} \right).
\end{align*}
where (A) follows from Lemma \ref{lemma:boundsqrtexpression}, (B) follows from an integral bound, and (C) follows from Jensen's inequality.

\end{proof}

We now prove Theorem \ref{thrm:dpL}. 
\begin{proof}[Proof of Theorem \ref{thrm:dpLgeneralized}]

Assume that clean event $\Event$ holds. This occurs with probability at least $1 - (|\mathcal{A} + 1) T^{-3}$ (Lemma \ref{lemma:cleanlipschitzucbgeneralized}), so the clean event not occurring counts negligibly towards regret. 

 Moreover, let $(a^*, b^*(a^*))$ be the Stackelberg equilibrium. Let 
 \[\alpha_t(a) = \frac{10 \sqrt{\mathcal{B} \log T}}{\sqrt{\NumPullsOneArg{a}{t}}} + C' \cdot L \cdot |\NumPullsOneArg{a}{t}|^{c_1-1} |\mathcal{B}|^{c_2} (\log T)^{c_3}\] be the confidence bound size at time step $t$ and let $\UCBOneArgTimeStep{1}{a}{t} = \hatMROneArgTimeStep{1}{a}{t} + \alpha_t(a)$ denote the UCB estimate in $\LipschitzUCBGeneralized(L, C)$ computed during time step $t$ prior to reward at time step $t$ being observed.

We can bound the leader's regret as: 
\begin{align*}
 R_1(T) &= \sum_{t=1}^T  (\MR{1}{a^*}{b^*(a^*)} - \MR{1}{a_t}{b_t})\\
  & 
= \sum_{t=1}^T \left(\MR{1}{a^*}{b^*(a^*)} - \MR{1}{a_t}{b^*(a_t)} \right) + \sum_{t=1}^T \left(\MR{1}{a_t}{b^*(a_t)} - \MR{1}{a_t}{b_t} \right)\\
&\le_{(A)} \sum_{t=1}^T \left(\hatMROneArg{1}{a^*} + \alpha_t(a^*) - \hatMROneArg{1}{a_t} + \alpha_t(a_t) \right) +  \sum_{t=1}^T \left|\MR{1}{a_t}{b^*(a_t)} - \MR{1}{a_t}{b_t} \right| \\
&\le \sum_{t=1}^T \left(\UCBOneArgTimeStep{1}{a^*}{t} - \UCBOneArg{1}{a_t} + 2 \cdot \alpha_t(a_t) \right) + L \cdot \sum_{t=1}^T \left|\MR{2}{a_t}{b^*(a_t)} - \MR{2}{a_t}{b_t}  \right| \\
&\le 2 \cdot \sum_{t=1}^T \alpha_t(a_t) + L \cdot \sum_{t=1}^T \left(\MR{2}{a_t}{b^*(a_t)} - \MR{2}{a_t}{b_t} \right) \\
&= 2 \cdot \sum_{t=1}^T \left(\frac{10 \sqrt{\mathcal{B} \log T}}{\sqrt{\NumPullsOneArg{a_t}{t}}} + C' \cdot L \cdot |\NumPullsOneArg{a_t}{t}|^{c_1-1} |\mathcal{B}|^{c_2} (\log T)^{c_3} \right) + L \cdot B_2 \\
&= 2 \cdot B_1 + L \cdot B_2  \\
&\le_{(B)} O \left(\sqrt{|\mathcal{A}| |\mathcal{B}| T \log T} + L \cdot |\mathcal{A}|^{1-c_1} |\mathcal{B}|^{c_2} (\log T)^{c_3} T^{c_1} \right)
\end{align*}
where (A) uses Lemma \ref{lemma:UCBconfidenceboundscorrectgeneralized} and (B) uses Lemma \ref{lemma:b1b2generalized}.  

We also bound the follower's regret as: 
\begin{align*}
 R_2(T) &= \sum_{t=1}^T \left( \MR{2}{a^*}{b^*(a^*)}  - \MR{2}{a_t}{b_t} \right) \\
&= \sum_{t=1}^T \left(\MR{2}{a^*}{b^*(a^*)} - \MR{2}{a^*}{b^*(a_t)}  \right) + \sum_{t=1}^T \left( \MR{2}{a^*}{b^*(a_t)} - \MR{2}{a_t}{b_t} \right) \\
&= \sum_{t=1}^T L \cdot \left|\MR{1}{a^*}{b^*(a^*)} - \MR{1}{a_t}{b^*(a_t)}\right| + B_2 \\
&=_{(A)} \sum_{t=1}^T L \cdot \left(\MR{1}{a^*}{b^*(a^*)} - \MR{1}{a_t}{b^*(a_t)}\right) + B_2 \\
&\le_{(B)} \sum_{t=1}^T L \cdot \left(\hatMROneArgTimeStep{1}{a^*}{t} + \alpha_t(a^*) - \hatMROneArgTimeStep{1}{a_t}{t} + \alpha_t(a^*)\right) + B_2 \\
&= \sum_{t=1}^T L \cdot \left(\UCBOneArgTimeStep{1}{a^*}{t} - \UCBOneArgTimeStep{1}{a_t}{t} + 2 \cdot \alpha_t(a_t)\right) + B_2 \\
&\le \sum_{t=1}^T L \cdot \left(2 \cdot \alpha_t(a_t) \right) + B_2 \\
&= 2 L \cdot \sum_{t=1}^T \left(\frac{10 \sqrt{\mathcal{B} \log T}}{\sqrt{\NumPullsOneArg{a_t}{t}}} + C' \cdot L \cdot |\NumPullsOneArg{a_t}{t}|^{c_1-1} |\mathcal{B}|^{c_2} (\log T)^{c_3} \right) + B_2 \\
&= 2 L \cdot B_1 + B_2  \\
&\le_{(C)} O \left(L \sqrt{|\mathcal{A}| |\mathcal{B}| T \log T} + L^2 \cdot |\mathcal{A}|^{1-c_1} |\mathcal{B}|^{c_2} (\log T)^{c_3} T^{c_1} \right)
\end{align*}
where (A) uses the fact that $a^*$ is the action chosen by the leader at the Stackelberg equilibrium
where (B) uses Lemma \ref{lemma:UCBconfidenceboundscorrectgeneralized} and (C) uses Lemma \ref{lemma:b1b2generalized}.

\end{proof}

\end{document}